\documentclass[twoside,11pt]{article}  %
\usepackage{blindtext}

\usepackage[preprint]{jmlr2e} %

\usepackage{booktabs}
\usepackage[utf8]{inputenc}
\usepackage[margin=1in]{geometry}
\usepackage{amsmath, amssymb}%
\usepackage{mathtools}
\usepackage{mathrsfs}
\usepackage[mathcal]{euscript}
\usepackage{tikz}
\usepackage{tikz-3dplot}
\usetikzlibrary{shapes,arrows, automata}
\usetikzlibrary{calc}
\usetikzlibrary{patterns}
\usetikzlibrary{intersections}
\usepackage{hyperref}
\usepackage{changes}
\usepackage[caption=false,font=footnotesize]{subfig}
\newcommand{\suma}{\Large$+$}
\newcommand{\con}{\Large $\pplus$}

\tikzset{%
  bullet/.style={ yshift = -.1mm
    , scale = .8},
  extfill/.style={fill = blue!20}, %
  frozenfill/.style={fill = gray, fill opacity = .4}, %
  ext/.style={trapezium, trapezium angle=67.5, draw,
  inner ysep=5pt, outer sep=0pt, extfill,
  minimum height=1.2cm, minimum width=0pt},
  extl/.style={isosceles triangle,
    isosceles triangle apex angle=60, minimum height=1.2cm,
    draw, extfill, minimum width=0pt},
  bias/.style={minimum height=0pt,
    draw, extfill, minimum width=2.4em},
  block/.style    = {draw, thick, rectangle, minimum height = 3em, minimum width = 5em},
  sum/.style      = {draw, circle, node distance = 2cm}, %
  con/.style      = {draw, circle, node distance = 2cm}, %
  input/.style    = {coordinate}, %
  int/.style    = {coordinate}, %
  missing/.style={
    draw=none,
    fill=none,
    yshift = 0.1cm,
    scale = 1.5,
    execute at begin node={\color{black}{$\vdots$}}
  },
  output/.style   = {coordinate}, %
  exts/.style={ext, 
  minimum height=1cm},
  pini/.style = {pin edge={to-,thick,black}}, %
  pino/.style = {pin edge={-to,thick,black}}, %
  blk/.style = {draw, thick, fill=blue!20, minimum size=2.8em, minimum width = 4.5em, text centered, text width = 4.5em},
  theblk/.style = {blk, text width = 5.2em, minimum height = 4em, fill = none, text centered},
}

\newcommand{\RightAngle}[4][6pt]{%
        \draw ($#3!#1!#2$)
        --($ #3!2!($($#3!#1!#2$)!.5!($#3!#1!#4$)$) $)
        --($#3!#1!#4$) ;
        }

\newcommand{\DrawMarkov}[1]
{
    \begin{tikzpicture}[->, >=stealth', auto, semithick, node distance=2.5cm]
    \tikzstyle{every state}=[fill=gray, fill opacity = .2,draw=black,thick,text=black,text opacity = 1, scale=1]
    \tikzstyle{prob}=[text = blue]
    \node[state]    (A)                     {$0$};
    \node[state]    (B)[right of=A]   {$1$};
    \path
    (A) edge[loop left]         node[prob, below, yshift = -.2em]{$1 - #1$}       (A)
    edge[bend left,above]       node[prob]{$#1$}     (B)
    (B) edge[bend left,below]   node[prob]{$#1$}     (A)
    edge[loop right]            node[prob, above, yshift = .2em]{$1 - #1$}      (B);
  \end{tikzpicture}
}

\allowdisplaybreaks

\definecolor{mgray}{HTML}{8A8B8C}
\usepackage{titlecaps} %
\usepackage{graphicx}
\usepackage{calc}
\usepackage{xcolor}
\colorlet{colori}{cyan}
\colorlet{colorb1}{red}
\colorlet{colorb2}{blue}
\colorlet{colorb}{colorb1!75!colorb2}
\colorlet{coloro}{colori!50!colorb}
\colorlet{colorp}{mgray!40}
\colorlet{colorv}{colori}
\colorlet{colorh}{colorb}

\newtheorem{property}{Property} %
\newtheorem{fact}{Fact} %

\newcommand{\qed}{$\blacksquare$}
\newcommand{\ds}{\displaystyle}
\newcommand{\secref}[1]{Section~\ref{#1}}
\newcommand{\appref}[1]{Appendix~\ref{#1}} %
\newcommand{\exref}[1]{Example~\ref{#1}}
\newcommand{\figref}[1]{Figure~\ref{#1}} %
\newcommand{\tabref}[1]{Table~\ref{#1}} %
\newcommand{\lemref}[1]{Lemma~\ref{#1}}
\newcommand{\thmref}[1]{Theorem~\ref{#1}}
\newcommand{\propref}[1]{Proposition~\ref{#1}}
\newcommand{\proptyref}[1]{Property~\ref{#1}}
\newcommand{\corolref}[1]{Corollary~\ref{#1}}
\newcommand{\factref}[1]{Fact~\ref{#1}}
\newcommand{\defref}[1]{Definition~\ref{#1}}

\newcommand{\ev}{{\mathbb{E}}}
\newcommand{\pr}{\mathbb{P}}
\newcommand{\prob}[1]{\pr\left\{#1\right\}}

\newcommand{\E}[1]{\ev\left[#1\right]}

\newcommand{\Ed}[2]{\ev_{#1}\left[#2\right]}
\newcommand{\bEd}[2]{\ev_{#1}\bigl[#2\bigr]}

\newcommand{\cqb}{\cQ_{\mathsf{B}}}

\newcommand{\cC}{\mathcal{C}}
\newcommand{\cD}{\mathcal{D}}

\newcommand{\cF}{\mathcal{F}}
\newcommand{\cG}{\mathcal{G}}

\newcommand{\cN}{\mathcal{N}}
\newcommand{\cU}{\mathcal{U}}
\newcommand{\cP}{\mathcal{P}}
\newcommand{\cQ}{\mathcal{Q}}

\newcommand{\cS}{\mathcal{S}}

\newcommand{\cX}{\mathcal{X}}
\newcommand{\cY}{\mathcal{Y}}
\newcommand{\cZ}{\mathcal{Z}}

\newcommand{\cDh}{\hat{\cD}}

\newcommand{\defeq}{\triangleq}
\newcommand{\eps}{\epsilon}
\newcommand{\T}{\mathrm{T}}

\newcommand{\Bt}{\tilde{B}}
\newcommand{\La}{\Lambda}

\newcommand{\norm}[1]{\|#1\|}
\newcommand{\bnorm}[1]{\bigl\|#1\bigr\|}
\newcommand{\bbnorm}[1]{\left\|#1\right\|}
\newcommand{\ip}[2]{\langle#1,#2\rangle}
\newcommand{\bip}[2]{\bigl\langle#1,#2\bigr\rangle}
\newcommand{\fip}[2]{\left\langle#1,#2\right\rangle}

\newcommand{\kld}[2]{D(#1\|#2)}
\newcommand{\bkld}[2]{D\bigl(#1\big\|#2\bigr)}

\newcommand{\fh}{\hat{f}}
\newcommand{\gh}{\hat{g}}

\newcommand{\xh}{\hat{x}}

\newcommand{\gammah}{\hat{\gamma}}

\newcommand{\pif}[1]{\pi_{#1}}
\newcommand{\pim}{\pi_{\mathsf{M}}}
\newcommand{\pic}{\pi_{\mathsf{C}}}

\newcommand{\pii}{\pi_{\mathsf{I}}}
\newcommand{\pib}{\pi_{\mathsf{B}}}
\newcommand{\pimn}[1]{\pi_{\mathsf{M}_{#1}}}
\newcommand{\picn}[1]{\pi_{\mathsf{C}_{#1}}}
\newcommand{\pibn}[1]{\pi_{\mathsf{B}_{#1}}}

\newcommand{\Pm}{P^{\mathsf{M}}}
\newcommand{\Ph}{\hat{P}}

\newcommand{\Xh}{\hat{X}}
\newcommand{\Yh}{{\hat{Y}}}
\newcommand{\Zh}{\hat{Z}}
\newcommand{\Sh}{\hat{S}}

\newcommand{\ft}{\tilde{f}}
\newcommand{\gt}{\tilde{g}}

\newcommand{\fb}{\bar{f}}
\newcommand{\gb}{\bar{g}}
\newcommand{\hb}{\bar{h}}

\newcommand{\gammab}{\bar{\gamma}}

\newcommand{\kb}{{\bar{k}}}
\newcommand{\Kb}{{\bar{K}}}
\newcommand{\sigmab}{\bar{\sigma}}
\newcommand{\sigmah}{\hat{\sigma}}

\DeclareMathOperator{\corr}{cov}
\DeclareMathOperator{\softmax}{softmax}
\DeclareMathOperator{\sign}{sgn}

\DeclareMathOperator{\relint}{relint}

\DeclareMathOperator{\tr}{tr}
\DeclareMathOperator*{\argmax}{arg\,max}
\DeclareMathOperator*{\argmin}{arg\,min}
\DeclareMathOperator*{\maximize}{maximize}
\DeclareMathOperator*{\minimize}{minimize}

\DeclareMathOperator{\spn}{span}
\DeclareMathAlphabet{\mathbbb}{U}{bbold}{m}{n}  %
\newcommand{\kron}{\mathbbb{1}}
\newcommand{\cst}[1]{\mathtt{1}_{#1}} %

\newcommand{\spnn}[1]{\spn^{#1}}

\newcommand{\cV}{\mathcal{V}}
\newcommand{\cW}{\mathcal{W}}
\newcommand{\orthp}{\mathbin{\boxplus}} %
\newcommand{\orthm}{\mathbin{\boxminus}} %

\newcommand{\proj}[2]{\Pi\left(#1;{#2}\right)}   %

\newcommand{\funcs}{\cF}
\newcommand{\sfuncs}{\cG}
\newcommand{\spc}[1]{\funcs_{#1}}   %
\newcommand{\spct}[1]{\tilde{\funcs}_{#1}}   %
\newcommand{\spctn}[2]{\spct{#1}^{\,#2}}   %

\newcommand{\spcn}[2]{\spc{#1}^{\,#2}}

\newcommand{\sspc}[1]{\sfuncs_{#1}} %
\newcommand{\sspcn}[2]{\sspc{#1}^{\,#2}} %
\newcommand{\sspcs}[2]{\sspc{#1}^{(#2)}} %
\newcommand{\spcs}[2]{\funcs_{#1|#2}}
\newcommand{\spcsn}[3]{\funcs_{#1|#2}^{\,#3}}

\newcommand{\pplus}{\mathbin{+\mkern-10mu+}}

\newcommand{\Unif}{\mathrm{Unif}}

\newcommand{\bs}{\mathsf{b}}
\newcommand{\cb}[2]{\bs_{#1}^{(#2)}} %

\newcommand{\Ds}{\mathscr{D}}
\newcommand{\Hs}{\mathscr{H}} %
\newcommand{\cside}{\cC_{\mathsf{MC}}}

\newcommand{\Pt}{\tilde{P}} %

\newcommand{\MAP}{{\mathrm{MAP}}}
\newcommand{\MLE}{{\mathrm{MLE}}}
\newcommand{\yh}{{\hat{y}}}

\newcommand{\lsym}{\ell} %
\newcommand{\llrrt}[2]{\tilde{\lsym}_{#1;#2}} %
\newcommand{\llrt}[1]{\tilde{\lsym}_{#1}} %

\newcommand{\lpmi}{\mathfrak{i}} %
\newcommand{\lpmih}{\hat{\lpmi}} %

\newcommand{\mdn}[1]{\md_{#1}} %
\newcommand{\mdnl}[1]{\md_{\leq #1}} %
\DeclareMathOperator{\md}{\zeta} %
\DeclareMathOperator{\rank}{rank} %
\DeclareMathOperator{\dom}{dom} %

\newcommand{\sfx}{{\sf x}}
\newcommand{\sfy}{{\sf y}}

\newcommand{\wx}{W_\sfx}
\newcommand{\wy}{W_\sfy}
\newcommand{\dx}{d_\sfx}
\newcommand{\dy}{d_\sfy}

\renewcommand{\vec}[1]{\underline{#1}}

\newcommand{\Hsm}{\Hs_{\rm{m}}}
\newcommand{\cbi}{\cC_{\sf BI}}

\newcommand{\Pest}{P^{(\mathrm{est})}}
\newcommand{\Pml}{P^{(\mathrm{ML})}}
\newcommand{\Lb}{L_{\sf B}}
\newcommand{\lml}{L^{(\mathrm{ML})}}
\newcommand{\Qb}{\bar{Q}}

\newcommand{\pint}{P^{\sf I}}
\newcommand{\pbi}{P^{\sf B}}
\newcommand{\ent}{{\mathsf{ent}}}
\newcommand{\pment}{P^{\ent}}
\newcommand{\opxn}[1]{\tau_{{#1}}}

\newcommand{\bms}{\mathrm{BMS}}

\newcommand{\rsymb}{\star} %
\newcommand{\refine}[1]{#1^\rsymb}
\newcommand{\dl}[1]{k_{#1}}
\newcommand{\ones}[1]{(1)^{#1}}

\newcommand{\cnone}{\{k;\, \spc{\cX} ;\, \spc{\cY}\}} %
\newcommand{\cnest}{\{\ones{k};\, \spc{\cX} ;\, \spc{\cY}\}} %
\newcommand{\cpi}{\cC_{\pi}} %
\newcommand{\cpir}{\refine{\cpi}} %

\newcommand{\csider}{\refine{\cside}} %
\newcommand{\cbir}{\refine{\cbi}} %

\newcommand{\llogs}{\llog_S} %
\newcommand{\llogsv}[1]{\llog_{S}^{(#1)}} %

\newcommand{\He}{\mathtt{He}}

\newcommand{\Lac}{\Sigma} %

\newcommand{\mytitle}{Neural Feature Learning in Function Space}

\usepackage{lastpage}
\jmlrheading{25}{2024}{1-\pageref{LastPage}}{9/23; Revised
1/24}{5/24}{23-1202}{Xiangxiang Xu and Lizhong Zheng}
\ShortHeadings{\mytitle}{Xu and Zheng}
\firstpageno{1}

\title{\mytitle\thanks{This work was presented in part at 2022 58th Annual Allerton Conference on Communication, Control, and Computing (Allerton), Monticello, IL, USA, Sep. 2022
 \citep{xu2022multivariate}.}}  
\author{\name Xiangxiang Xu \email xuxx@mit.edu 
       \AND
       \name Lizhong Zheng \email lizhong@mit.edu \\[1em]
       \addr Department of Electrical Engineering and Computer Science\\
       Massachusetts Institute of Technology\\
       Cambridge, MA 02139-4307, USA}

\begin{document}

\editor{Kilian Weinberger}

\maketitle
\begin{abstract}%
  We present a novel framework for learning system design with neural feature extractors. First, we introduce the feature geometry, which unifies statistical dependence and feature representations in a function space equipped with inner products. This connection defines function-space concepts on statistical dependence, such as norms, orthogonal projection, and spectral decomposition, exhibiting clear operational meanings. In particular, we associate each learning setting with a dependence component and formulate learning tasks as finding corresponding feature approximations. We propose a nesting technique, which provides systematic algorithm designs for learning the optimal features from data samples with off-the-shelf network architectures and optimizers. We further demonstrate multivariate learning applications, including conditional inference and multimodal learning, where we present the optimal features and reveal their connections to classical approaches.
\end{abstract}


\section{Introduction}
Learning useful feature representations from data observations is a fundamental task in machine learning. Early developments of such algorithms focused on learning optimal linear features, e.g., linear regression, PCA (Principal Component Analysis) \citep{pearson1901liii}, CCA (Canonical Correlation Analysis) \citep{hotelling1936relations},  and LDA (Linear Discriminant Analysis) \citep{fisher1936use}. The resulting algorithms admit straightforward implementations, with well-established connections between learned features and statistical behaviors of data samples. However, practical learning applications often involve data with complex structures which linear features fail to capture. To address such problems, practitioners employ more complicated feature designs and build inference models based on these features, e.g., 
kernel methods  \citep{cortes1995support, hofmann2008kernel} and deep neural networks \citep{lecun2015deep}. The feature representations serve as the information carrier, capturing useful information from data for subsequent processing.
An illustration of such feature-centric learning systems is shown in \figref{fig:blk}, which consists of two parts:
\begin{enumerate}
\item A \emph{learning} module which generates a collection of features from the data. Data can take different forms, for example, input-output pairs\footnote{In literature, the input variables are sometimes referred to as independent/predictor variables, and the output variables are also referred to as dependent/response/target variables.} or some tuples.  The features can be either specified implicitly, e.g., by a kernel function in kernel methods, or explicitly parameterized as feature extractors, e.g., artificial neural networks. The features are learned via a training process, e.g., optimizing a training objective defined on the features.
\item An \emph{assembling} module which uses learned features to build an inference model or a collection of inference models. The inference models are used to provide information about data. For example, when the data take the form of input-output pairs, a typical inference model provides prediction or estimation of output variables based on the input variables.
 The assembling module determines the relation between features and resulting models, which can also be specified implicitly, e.g., in kernel methods. %
\end{enumerate}
\begin{figure}[!ht]
  \centering
  \resizebox{.7\textwidth}{!}{%
\begin{tikzpicture}[node distance=3.9cm,auto,>=latex']
    \node [theblk, pin={[pini]left:{\sf Data}}] (x) {{\sf Learning }};    %
    \node [theblk] (c) [right of=x, pin={[pino]right:{\sf Model}%
    }] {{ \sf Assembling }};    
    \node [coordinate] (end) [right of=c, node distance=1.2cm]{};
    \path[->, thick] (x) edge node [pos=.5, above] {\sf Feature}  (c); %
\end{tikzpicture}

}
  \caption{Schematic diagram of a general feature-centric learning system } 
  \label{fig:blk}  
\end{figure}
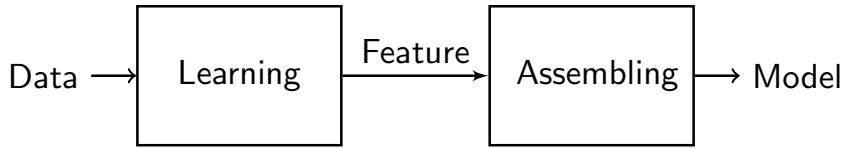

Learning systems are commonly designed with a predetermined assembling module, which allows the whole system to be learned in an end-to-end manner. One representative example of such designs is deep neural networks.
On one hand, this end-to-end characteristic makes it possible to employ large-scale and often over-parametrized neural feature extractors \citep{lecun2015deep, krizhevsky2017imagenet, he2016deep, vaswani2017attention}, which can effectively capture hidden structures in data. On the other hand, the choices of assembling modules and learning objectives are often empirically designed, with design heuristics varying a lot across different tasks. Such heuristic choices make the learned feature extractors hard to interpret and adapt, often viewed as black boxes. In addition, the empirical designs are typically inefficient, especially for multivariate learning problems, e.g., multimodal learning \citep{ngiam2011multimodal}, where there can exist many potentially useful assembling modules and learning objectives to consider.
To address this issue and obtain more principled designs, recent developments have adopted statistical and information-theoretical tools in designing training objectives. The design goal is to guarantee that learned features are informative, i.e., carry useful information for the inference tasks. To this end, a common practice is to incorporate information measures in learning objectives, such as mutual information  \citep{tishby2000information, tishby2015deep} and rate distortion function \citep{chan2022redunet}. However, information measures might not effectively
 characterize the usefulness of features for learning tasks, due to the essentially different information processing natures.
{For example, originally introduced in characterizing the optimal rate of a communication system \citep{shannon1948mathematical}, the mutual information is invariant to bijective transformations on variables---as the amount of information to communicate does not change after such transformations. 
As a consequence, when we consider the mutual information between feature representations and the input/output variables \citep{tishby2015deep}, features up to such bijective transformations, e.g., sigmoid or exponential mappings, are all equivalent and give the same mutual information. This is in contrast to their often highly distinct performances on learning tasks.
Due to this intrinsic discrepancy, it is often impractical to learn features based solely on the information measures. As a compromise, the information measures are often applied to analyze features learned from other objectives \citep{tishby2015deep}, or used as regularization terms in the training objective to facilitate learning \citep{belghazi2018mutual}, where the learned features generally depend on case-by-case design choices. 
}
In this work, we aim to establish a framework for learning feature representations that capture the statistical nature of data, and can be assembled to build different inference models without retraining. In particular, the features are learned to approximate and represent the statistical dependence of data, instead of solving specific inference tasks. To this end, we introduce a
 geometry on function spaces, coined \emph{feature geometry}, which we apply to connect statistical dependence with features. 
This connection allows us to represent statistical dependence by corresponding operations in feature spaces. Specifically, the features are learned by approximating the statistical dependence, and the approximation error measures the amount of information carried by features.  The resulting features capture the statistical dependence and thus are useful for general inference tasks.%

Our main contributions of this work are as follows.
\begin{itemize}
\item We establish a framework for designing learning systems that separate feature learning and feature usages, where the learned features can be assembled to build different inference models without retraining. 
In particular, we introduce the feature geometry, which converts feature learning problems to corresponding function-space operations on statistical dependence. The resulting optimal features capture the statistical dependence of data and can be adapted to different inference tasks.
\item  We propose a nesting technique for decomposing statistical dependence and learning the associated feature representations. The nesting technique provides a systematic approach to construct training objectives and develop learning algorithms, where we can learn optimal features by
employing the state-of-the-art deep learning practices, e.g., network architecture and optimizer designs. %
\item We present the applications of this unified framework in multivariate learning problems, where we design feature learning algorithms to decompose and represent the corresponding multivariate dependence.
As case studies, we consider two learning scenarios: learning for conditional inference and multimodal learning with missing modalities. We investigate the optimal features and demonstrate their relations to classical learning problems, such as maximum likelihood estimation and the maximum entropy principle  \citep{jaynes1957informationa, jaynes1957informationb}.
\end{itemize}

The rest of the paper is organized as follows. In \secref{sec:notation-pre}, we introduce the feature geometry, including operations on feature spaces and  function representations of statistical dependence. In \secref{sec:dep:approx}, we present the learning system design in a bivariate learning setting, where we demonstrate the feature learning algorithm and the design of assembling modules. In \secref{sec:nest}, we introduce the nesting technique, as a systematic approach to learning algorithm design in multivariate learning problems. We then demonstrate the applications of the nesting technique, where we study a conditional inference problem in 
 \secref{sec:side}, and a multimodal learning problem
in \secref{sec:mm}. Then, we present our experimental verification for the proposed learning algorithms in \secref{sec:exp}, where we compare the learned features with the theoretical solutions. Finally, we summarize related works in \secref{sec:relate} and provide some concluding remarks in
\secref{sec:conclusion}.

\section{Notations and Preliminaries} %
\label{sec:notation-pre}
For a random variable $Z$, we use
$\cZ$ to denote the corresponding alphabet, and use $z$ to denote a specific value in $\cZ$. We use $P_Z$ to denote the probability distribution of $Z$. 

For the sake of simplicity, we present our development with finite alphabets and associated discrete random variables. {The corresponding results can be readily extended to general alphabets under certain regularity conditions, which we discuss in \appref{app:notation:con} for completeness.}

\subsection{Feature Geometry}%
\subsubsection{Vector Space}
Given an inner product space with inner product $\ip{\cdot}{\cdot}$ and its induced norm $\norm{\cdot}$, we can define the projection and orthogonal complement as follows. 
\begin{definition}
  \label{def:proj}
  Give a subspace $\cW$ of $\cV$,
we denote the projection of a vector $v \in \cV$ onto $\cW$ by
 \begin{gather}
  \proj{v}{\cW} \defeq \argmin_{w \in \cW} \bbnorm{v - w}^2.
 \end{gather} %
In addition, we use $\cV \orthm \cW$ to denote the orthogonal complement of $\cW$ in $\cV$, viz.,
 $ \cV \orthm \cW \defeq \{v \in \cV \colon \ip{v}{w} = 0\text{ for all }w \in \cW\}$.
\end{definition}
We use ``$\orthp$'' to denote the direct sum of orthogonal subspaces, i.e., $\cV =  \cV_1 \orthp \cV_2$ indicates that $\cV = \cV_1 + \cV_2$ and $\cV_1 \perp \cV_2$. Then we have the following facts.
\begin{fact}
  \label{fact:orthm}
  If $\cV = \cV_1 \orthp \cV_2$, then $\cV_2 = \cV \orthm \cV_1$. In addition, if $\cW$ is a subspace of $\cV$, then $\cV = \cW \orthp (\cV \orthm \cW) $.
\end{fact}

\begin{fact}[Orthogonality Principle] Given $v \in \cV$ and a subspace $\cW$ of $\cV$, then
  $w = \proj{v}{\cW}$ if and only if $w \in \cW$ and $v - w \in \cV \orthm \cW$.  In addition, we have %
  $v = \proj{v}{\cW} + \proj{v}{\cV \orthm \cW}$.
  \label{fact:proj}
\end{fact}

\subsubsection{Feature Space}
\label{sec:feature-space}
Given an alphabet $\cZ$, we use $\cP^\cZ$ to denote the collection of probability distributions supported on $\cZ$, and use $\relint(\cP^\cZ)$ to denote the relative interior of $\cP^\cZ$, i.e., the collection of distributions { $P \in \cP^\cZ$ with $P(z) > 0$ for all $z \in \cZ$.}

We use $\spc{\cZ} \defeq \{ \cZ \to \mathbb{R}\}$ to denote the collection of features (functions) on given $\cZ$. {Specifically, we use $\cst{\cZ} \in \spc{\cZ}$ to denote the constant feature that takes value 1, i.e., $\cst{\cZ}(z) \equiv 1$ for all $z \in \cZ$.}
  We define the inner product on $\spc{\cZ}$ as\footnote{Throughout our development, we use $\Ed{Q}{f(Z)}$ to denote the expectation of function $f$ with respect to $Q \in \cP^{\cZ}$, i.e., $\bEd{\Zh\sim Q}{f(\Zh)}$. Specifically, we will also use $\E{f(Z)}$ to represent $\Ed{P_Z}{f(Z)}$, i.e., the expectation with respect to the underlying distribution; similar conventions apply to conditional expectations.
}
 $ \ip{f_1}{f_2} \defeq \Ed{R}{f_1(Z)f_2(Z)}$, where $R \in \relint(\cP^{\cZ})$ is referred to as the \textbf{metric distribution}. {This defines operations, e.g., norm and projection, on $\spc{\cZ}$.} Specifically, we have the induced norm $\norm{f} \defeq \sqrt{\ip{f}{f}}$, and the projection of $f \in \spc{\cZ}$ onto a subspace $\sspc{\cZ}$ of $\spc{\cZ}$, i.e., $\proj{f}{\sspc{\cZ}}$, is defined according to \defref{def:proj}. 
{We also use $\spct{\cZ} \defeq \{f \in \spc{\cZ} \colon \Ed{R}{f(Z)} = 0 \}$ to denote the collection of zero-mean features on $\cZ$, which corresponds to the orthogonal complement of constant features, i.e.,  $\spct{\cZ} = \spc{\cZ} \orthm \spn\{\cst{\cZ}\}$}.

For each $k \geq 1$, we use $\spcn{\cZ}{k} \defeq \left(\spc{\cZ}\right)^k = \{ \cZ \to \mathbb{R}^k\}$ to denote the collection  of $k$-dimensional features.
{For $f \in \spcn{\cZ}{k}$, we use 
 $\spn\{f\} \subset \spc{\cZ}$ to denote the subspace of $\spc{\cZ}$ spanned by all dimensions of $f$}, i.e., $\spn\{f\} = \spn\{f_1, \dots, f_k\}$, 
and use $\proj{f}{\sspc{\cZ}}$ to denote the corresponding projection on each dimension, i.e.,  
$\proj{f}{\sspc{\cZ}} = [\proj{f_1}{\sspc{\cZ}}, \dots, \proj{f_k}{\sspc{\cZ}}]^\T \in \sspcn{\cZ}{k} \defeq (\sspc{\cZ})^{k}$.
For $f_1 \in \spcn{\cZ}{k_1}, f_2 \in \spcn{\cZ}{k_2}$, we denote $\La_{f_1, f_2} \defeq \Ed{R}{f_1(Z)f_2^\T(Z)}\in \mathbb{R}^{k_1\times k_2}$. Specifically, we define $\La_f \defeq \La_{f, f} = \Ed{R}{f(Z)f^\T(Z)}$ for feature $f \in \spcn{\cZ}{k}$.  We also use $[k]$ to denote the set $\{i \in \mathbb{Z}\colon 1\leq i \leq k \}$ with ascending order, and define $ f_{[k]} \in \spcn{\cZ}{k}$ as $f_{[k]}  \colon z \mapsto (f_1(z), \dots, f_{{k}}(z))^\T$
 for given $f_1, \dots, f_k \in \spc{\cZ}$. 

\subsubsection{ {Neural Feature Extractors} }
{Our development will focus on neural feature extractors, i.e., features represented by neural networks. To begin, we consider a neural network of a given architecture with $m$ trainable parameters and $k$-dimensional outputs. Let $\cZ$ denote the domain of input data $Z$, and denote the trainable parameters by}\footnote{We use underlined lowercase letters to denote Euclidean vectors, e.g., $\vec{\theta}$, $\vec{v}$, $\vec{b}$, to distinguish them from scalars.} {$\vec{\theta} \in \mathbb{R}^m$. Then, for each $\vec{\theta} \in \mathbb{R}^m$, we can denote the associated neural feature as $f^{(\vec{\theta})} \in \spcn{\cZ}{k}$. }

{For the sake of simplicity, our discussions focus on the ideal case where the neural network has sufficient expressive power}\footnote{Such idealized network can be constructed by using the universal approximation theorem; see, e.g.,\cite{cybenko1989approximation}, for detailed discussions.}, {i.e., each feature in $\spcn{\cZ}{k}$ can be well-approximated by a large $m$ and some $\vec{\theta} \in \mathbb{R}^m$.  Under this idealized assumption, optimizing over network parameters $\vec{\theta}$ is equivalent to finding the optimal $f \in \spcn{\cZ}{k}$. Therefore, we will omit the network parameter $\vec{\theta}$ and use $f$ to denote such neural feature extractor, which we represented as a trapezoid block}\footnote{The literature sometimes use the longer (shorter) bases to  indicate the larger (smaller) dimensions of input and output. However, our schematic representation does not make this distinction.}, as shown in \figref{fig:nfe:f}.

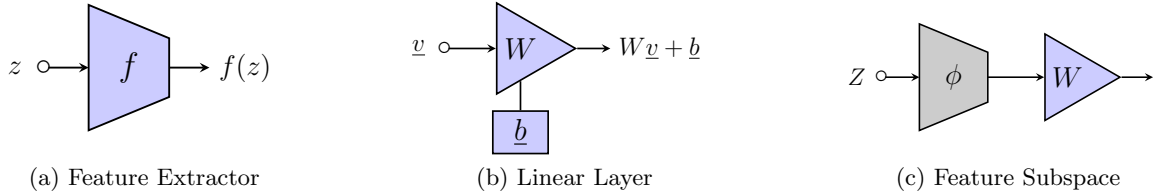
\begin{figure}[!ht]
  \centering    
  \subfloat[Feature Extractor]{\raisebox{.8em}{
      \resizebox{.24\textwidth}{!}{
\begin{tikzpicture}[auto, thick, node distance=2cm, >=stealth]

  \draw 
  node [ext, rotate = -90] (f) {}
  node [input, left of = f, node distance=1.2cm] (z)  {}
  node [xshift = -0.5cm] at (z) {\large $z$}
  node [xshift = -0.7mm] at (z) {\Large \textopenbullet}
  node at (f) {\Large$f$};
  \draw [->] (z) -- (f);
  \draw[->] (f) -- ++ (1.2, 0) node [right] {\large $f(z)$};

\end{tikzpicture}

       }
    }\label{fig:nfe:f}
  }%
  \hspace{.07\textwidth}
  \subfloat[Linear Layer]{\resizebox{.25\textwidth}{!}{%
\begin{tikzpicture}[auto, thick, node distance=2cm, >=stealth]

  \draw %
  node [extl, node distance=2.5cm] (W) {\Large $W$}
  node [bias, below of = W, node distance=1.4cm] (b) {\Large $\vec{b}$}
  node [input, left of = W, node distance=1.2cm] (v)  {}
  node [xshift = -0.5cm] at (v) {\large $\vec{v}$}
  node [xshift = -0.7mm] at (v) {\Large \textopenbullet};
  \draw [->] (v) -- (W);
  \draw [-] (W) -- (b);
  \draw[->] (W) -- ++ (1.5, 0) node [right] {\large $W\vec{v}+\vec{b}$}; %

\end{tikzpicture}

}
    \label{fig:nfe:wb}
}
  \hspace{.07\textwidth}
  \subfloat[Feature Subspace]{
    \raisebox{.8em}{      \resizebox{.27\textwidth}{!}{%
\begin{tikzpicture}[auto, thick, node distance=2cm, >=stealth]

  \draw 
  node [extl, node distance=2.5cm] (W) {\Large $W$}
  node [ext, frozenfill, left of = W, rotate = -90] (f) {}
  node [input, left of = f, node distance=1.2cm] (z)  {}
  node [xshift = -0.5cm] at (z) {\large $Z$}
  node [xshift = -0.7mm] at (z) {\Large \textopenbullet}
  node at (f) {\Large $\phi$};
  \draw [->] (z) -- (f);
  \draw[->] (f) -- (W);
  \draw[->] (W) -- ++ (1.5, 0) node [right] {};%

\end{tikzpicture}

}
    }
    \label{fig:nfe:sspc}
  }
  \label{fig:nfe}
  \caption{Schematic representations
 of neural feature extractors. 
\protect\subref{fig:nfe:f}: a general feature extractor $f \in \spcn{\cZ}{k}$; \protect\subref{fig:nfe:wb}: a linear layer with weight matrix $W$ and bias $\vec{b}$; \protect\subref{fig:nfe:sspc}:
the composition of feature extractor blocks, where each dimension of the output lies in the feature subspace $\spn\{\phi\}$.
}
\end{figure}

{ The linear layer is an important building block in neural feature extractors, which performs an affine transformation on the input, as shown in \figref{fig:nfe:wb}. In particular, a linear layer with input dimension $d$ and output dimension $k$ is specified by a weight matrix $W \in \mathbb{R}^{k \times d}$ and bias vector $\vec{b} \in \mathbb{R}^k$, which outputs $W\vec{v} + \vec{b}$ for an input $\vec{v} \in \mathbb{R}^d$. We can use linear layers to construct feature subspaces, as shown in \figref{fig:nfe:sspc}, where we have considered a given $d$-dimensional feature $\phi = (\phi_1, \dots, \phi_d)^\T \in \spcn{\cX}{d}$ and a linear layer without the bias term. Then, the collection of output features by varying weight matrix $W$ is}\footnote{If we also consider the bias term $\vec{b}$, the feature subspace will become $\spnn{k}\{\cst{\cZ}, \phi_1, \dots, \phi_d\}$, where $\cst{\cZ}$ represents the constant function $\cZ \ni z \mapsto 1$.}
{$\{W \phi\colon W \in \mathbb{R}^{k\times d} \} = \spnn{k}\{\phi\}$. 
Note that since the generated features are restricted to a subspace, the cascaded feature extractor has restricted expressive power, which can generally affect the feature learning. We will discuss such effects in our later developments (cf. \secref{sec:h:const}).
}

\subsubsection{Joint Functions}
Given alphabets $\cX, \cY$ and a metric distribution $R_{X, Y} \in \relint(\cP^{\cX \times \cY})$,  $\spc{\cX \times \cY}$ is composed of all joint functions of $x$ and $y$. In particular, 
for given $f \in \spc{\cX}$, $g \in \spc{\cY}$, we use $f \otimes g$ to denote their product $\left((x, y) \mapsto f(x) \cdot g(y)\right) \in \spc{\cX \times \cY}$, and refer to such functions
as \emph{product functions}. For each %
product function $\gamma \in \spc{\cX \times \cY}$, we can {find
 $\sigma = \norm{\gamma} \geq 0$, and $f \in \spc{\cX}, g \in \spc{\cY}$ with $\norm{f} = \norm{g} = 1$, such that
}
\begin{align}
 \gamma = \sigma \cdot (f \otimes g).\label{eq:form:rank1}  
\end{align}
We refer to \eqref{eq:form:rank1} as the \textbf{standard form} of product functions.

In addition, for given $f = (f_1, \dots, f_k)^\T \in \spcn{\cX}{k}$ and $g = (g_1, \dots, g_k)^\T \in \spcn{\cY}{k}$, we denote $ f \otimes g \defeq \sum_{i = 1}^k f_i \otimes g_i$. 
For two subspaces $\sspc{\cX}$ and $\sspc{\cY}$ of $\spc{\cX}$ and $\spc{\cY}$, respectively, we denote the tensor product of $\sspc{\cX}$ and $\sspc{\cY}$ as   $\sspc{\cX} \otimes \sspc{\cY} \defeq \spn\{f \otimes g\colon f \in \sspc{\cX}, g \in \sspc{\cY}\}$.

Note that by extending each $f = (x \mapsto f(x)) \in \spc{\cX}$ to $\left((x, y) \mapsto f(x)\right) \in \spc{\cX \times \cY}$, %
 $\spc{\cX}$ becomes %
 a subspace of $\spc{\cX \times \cY}$, with the metric distribution being the marginal distribution $R_X$ of $R_{X, Y}$. We then denote the orthogonal complement of $\spc{\cX}$ in  $\spc{\cX \times \cY}$  as
 \begin{align}
   \spcs{\cY}{\cX} \defeq \spc{\cX \times \cY} \orthm \spc{\cX}.\label{eq:spcs:def}   
 \end{align}
 We establish a correspondence between the distribution space $\cP^\cZ$ and the feature space $\spc{\cZ}$ by the density ratio function.
\begin{definition}
  Given a metric distribution $R \in \relint(\cP^\cZ)$, for each
  $P \in \cP^\cZ$, we define the (centered) density ratio function $\llrrt{P}{R} \in \spc{\cZ}$ as
  \begin{align*}
    \llrrt{P}{R}(z) \defeq \frac{P(z) - R(z)}{R(z)}, \quad \text{for all }z \in \cZ.
  \end{align*}
\end{definition}%
It is easy to verify that $\llrrt{P}{R}$ has mean zero
, i.e., $\llrrt{P}{R} \in \spct{\cZ}$. We will simply refer to $\llrrt{P}{R}$ as the density ratio or likelihood ratio and use $\llrt{P}$ to denote $\llrrt{P}{R}$ when there is no ambiguity about the metric distribution $R$.

In particular,  given random variables $X$ and $Y$ with the joint distribution $P_{X, Y}$, we denote the density ratio $ \llrrt{P_{X, Y}}{P_{X}P_Y}$ by $\lpmi_{X; Y}$, i.e.,
\begin{align}
  \lpmi_{X; Y}(x, y) = \frac{P_{X, Y}(x, y) - P_X(x)P_Y(y)}{P_X(x)P_Y(y)}, \quad \text{for all $x\in \cX, y \in \cY$.}
  \label{eq:cdk:def}
\end{align}
We refer to $\lpmi_{X; Y}$ as the \textbf{canonical dependence kernel} (CDK) function, which connects the $(X; Y)$ dependence with $\spc{\cX \times \cY}$. 

With the feature geometry, we can associate function-space concepts with corresponding operations on features, which we summarize as follows. %

\begin{property}
\label{propty:exp}
Consider the feature geometry on $\spc{\cX \times \cY}$ defined by the metric distribution $R_{X, Y} = P_{X}P_{Y}$. Then,  we have
 $\ip{f_1 \otimes g_1}{f_2 \otimes g_2}
    = \Ed{P_X}{f_1(X)f_2(X)} \cdot \Ed{P_Y}{g_1(Y)g_2(Y)}$ for given $f_1, f_2 \in \spc{\cX}, g_1, g_2 \in \spc{\cY}$. In addition, 
  For any $k \geq 1$ and $f \in \spcn{\cX}{k}, g \in \spcn{\cY}{k}$, we have
    \begin{gather}
       \proj{f}{\spn\{\cst{\cX}\}} = \Ed{P_X}{f(X)}, \quad \proj{f}{\spct{\cX}} = f - \Ed{P_X}{f(X)}, \\
      \ip{\lpmi_{X; Y}}{f \otimes g} = \Ed{P_{X, Y}}{f^{\T}(X)g(Y)}  - \left(\Ed{P_X}{f(X)}\right)^{\T}\Ed{P_Y}{g(Y)},\label{eq:exp:cdk}
\\
      \norm{f\otimes g}^2 = \tr(\La_f  \La_g),
      \label{eq:exp}      
    \end{gather}
  where $\La_f = \Ed{P_X}{f(X)f^\T(X)}, \La_g = \Ed{P_Y}{g(Y)g^\T(Y)}$.
\end{property}

\subsubsection{Feature Geometry on Data Samples}
\label{sec:geo:data}
In practice, the variables of interest typically have unknown and complicated
probability distributions, with only data samples available for learning. 
We can similarly define the feature geometry on data samples by considering the corresponding empirical distributions.

To begin, given a dataset consisting of $n$ samples of $Z$, 
represented as $\{z_i\}_{i = 1}^n$, we denote the corresponding empirical distribution $\Ph_Z \in \cP^\cZ$ by
\begin{align}
  \Ph_Z(z) \defeq \frac1n\sum_{i = 1}^n\kron_{\{z_i = z\}}, \quad \text{for all }z \in \cZ,
  \label{eq:emp:z}
\end{align}
where $\kron_{\{\cdot\}}$ denotes the indicator function. Then, for any function $f$ of $Z$, we have
$   \Ed{\Ph_Z}{f(Z)} = \sum_{z \in \cZ} \Ph_Z(z) \cdot f(z) = \frac{1}{n}\sum_{i = 1}^n f(z_i)$,
which is the empirical average of $f$ over the dataset.

Specifically, given a dataset $\Ds = \{(x_i, y_i)\}_{i=1}^n$ with $n$ sample pairs  of $(X, Y)$, the corresponding joint empirical distribution $\Ph_{X, Y} \in \cP^{\cX \times \cY}$ is given by
\begin{align}
  \Ph_{X, Y}(x, y)=\frac1n\sum_{i = 1}^n\kron_{\{x_i = x, y_i = y\}}, \quad \text{for all }x \in \cX, y \in \cY.
  \label{eq:emp:dist:xy}
\end{align}
It is easily verified that the marginal distributions of $\Ph_{X, Y}$ are the empirical distributions $\Ph_X$ of $\{x_i\}_{i =1}^n$ and $\Ph_Y$ of  $\{y_i\}_{i =1}^n$. Therefore, we can express the CDK function associated with the empirical distribution $\Ph_{X, Y}$ as [cf. \eqref{eq:cdk:def}]
 \begin{align}
   \lpmih_{X; Y}(x, y) = \frac{\Ph_{X, Y}(x, y) - \Ph_X(x)\Ph_Y(y)}{\Ph_X(x)\Ph_Y(y)}, \quad \text{for all $x\in \cX, y \in \cY$}.
   \label{eq:lpmih:def}
\end{align}

As a result, given the dataset $\Ds$, we can define the associated feature geometry on $\spc{\cX \times \cY}$ by using the metric distribution $R_{X, Y} = \Ph_{X}\Ph_{Y}$. From \proptyref{propty:exp}, we can evaluate the induced geometric quantities over data samples in $\Ds$, by replacing the distributions by the corresponding empirical distributions, and applying the empirical CDK function $\lpmih_{X; Y}$ in \eqref{eq:exp:cdk}. 

Such characteristic allows us to discuss theoretical properties and algorithmic implementations
in a unified notation. In our later developments, we will state both theoretical results and algorithms designs using the same notation of %
distribution, say, $P_{X, Y}$. This $P_{X, Y}$ corresponds to the underlying distribution
in theoretical analyses, and represents the corresponding empirical distribution in algorithm implementations. %

\begin{remark}
  Note that for finite $\cZ$, $\spc{\cZ}$ is a space with a finite dimension $|\cZ|$. Therefore, we can choose a basis of $\spc{\cZ}$ and represent each feature as corresponding coefficient vectors in Euclidean space $\mathbb{R}^{|\cZ|}$. Similarly, we can represent operators on function spaces as matrices. Such conventions have been introduced and adopted in previous works, e.g., \citep{HuangMWZ2024, xu2022information}, which we summarize in \appref{app:notation:finite} for completeness and comparisons.
\end{remark}

\subsection{Modal Decomposition} %
{We then investigate a representation of joint functions in $\spc{\cX \times \cY}$ by using features in $\spc{\cX}$ and $\spc{\cY}$. This representation will be our basic tool to connect statistical dependence with feature spaces. Throughout this section, we set the metric distribution of $\spc{\cX \times \cY}$ to the product form}\footnote{Under this choice, the norm on $\spc{\cX \times \cY}$ corresponds to the Frobenius (Hilbert--Schmidt) norm.}, {i.e., $R_{X, Y} = R_{X}R_{Y}$. This induced metric distributions on $\spc{\cX}$ and $\spc{\cY}$ are $R_X$ and $R_Y$, respectively.
}

\subsubsection{Definitions and Properties}%

We use the operator $\md$ on $\spc{\cX \times \cY}$ to denote the optimal rank-1 approximation, i.e., 
\begin{align}
  \md(\gamma) \defeq \argmin_{\substack{\gamma'\colon \gamma' = f\otimes g\\
  f \in \spc{\cX}, g \in \spc{\cY }}}
  \norm{\gamma - \gamma'},\quad  \text{for all $\gamma \in \spc{\cX \times \cY}$}.
  \label{eq:def:md}
\end{align}
In addition, for all $k \geq 1$, we define the operator $\mdn{k}$ as  
$\mdn{1} \defeq \md$ and
 \begin{align}
  \qquad \mdn{k}(\gamma) \defeq \md\left(\gamma - \sum_{i = 1}^{k- 1}\mdn{i}(\gamma)\right),
 \end{align}
which we refer to as the $k$-th mode of $\gamma$. Then, we use
\begin{align}
\mdnl{k}(\gamma) \defeq \sum_{i = 1}^k \mdn{i}(\gamma) \qquad \text{and}  \qquad r_k(\gamma) \defeq \gamma -  \mdnl{k}(\gamma)\label{eq:mdnl:r:def}  
\end{align}
 to denote the superposition of the top $k$ modes and the corresponding remainder, respectively.

  \begin{remark}
    If the minimization problem \eqref{eq:def:md} has multiple solutions, the corresponding $\md(\gamma)$ (and $\mdn{k}(\gamma)$) might not be unique. In this case,  $\mdn{1}(\gamma), \dots \mdn{k}(\gamma)$ represent one of such solutions obtained from the sequential rank-1 approximations.
  \end{remark}

  For each $\gamma \in \spc{\cX \times \cY}$, we define the rank of $\gamma$ as
  $\rank(\gamma) \defeq \inf \{k \geq 0\colon \norm{r_k(\gamma)} = 0\}$.
{Suppose $K \defeq \rank(\gamma)$, then from the definition \eqref{eq:mdnl:r:def}, we have $\gamma = \mdnl{K}(\gamma) + r_K(\gamma) = \mdnl{K}(\gamma)$, i.e., 
  \begin{align}
    \gamma = \sum_{i = 1}^K \mdn{i}(\gamma).
    \label{eq:modal:gamma}
  \end{align}
For each $i \in [K]$, let us denote the standard form [cf. \eqref{eq:form:rank1}] of $\mdn{i}(\gamma)$ by
 $\mdn{i}(\gamma) = \sigma_i \left(f_i^* \otimes g_i^*\right)$.}
 Therefore, from \eqref{eq:modal:gamma} we obtain
\begin{align}
    \gamma(x, y)
    = \sum_{i = 1}^{K} \sigma_i \cdot f_i^*(x) \cdot g_i^*(y), \quad \text{for all $x \in \cX$, $y \in \cY$},
    \label{eq:modal:dcmp}
  \end{align}
  where $ \norm{f_i^*} = \norm{g_i^*} = 1$ and $\sigma_i = \norm{\mdn{i}(\gamma)}$. We refer to \eqref{eq:modal:dcmp} as the \emph{modal decomposition} of $\gamma$, which is a special case of Schmidt decomposition \citep{schmidt1907theorie, ekert1995entangled}, or
 singular value decomposition (SVD) in function space. %
 We list several useful characterizations as follows. %

\begin{fact}
  \label{fact:md:ortho}
  Let $K \defeq \rank(\gamma)$, then $\sigma_1 \geq \sigma_2 \geq \dots \geq \sigma_K > 0$. In addition, for all $i, j = 1, \dots, K$, we have\footnote{
    We adopt the Kronecker delta notation
    \begin{align*}
      { \delta _{ij}={
      \begin{cases}
        0&{\text{if }}i\neq j,\\
        1&{\text{if }}i=j.
      \end{cases}}}
    \end{align*}
  } $\ip{f^*_i}{f^*_j} = \ip{g^*_i}{g^*_j} = \delta_{ij}$, and
  \begin{align*}
    \ip{\mdn{i}(\gamma)}{\mdn{j}(\gamma)} = 0, \qquad \text{if }i < j,\\
    \ip{\mdn{i}(\gamma)}{r_j(\gamma)} = 0, \qquad \text{if }i \leq j,
  \end{align*}
  {where $r_j(\cdot)$ is as defined in \eqref{eq:mdnl:r:def}.  }
\end{fact}

   \begin{fact}
     \label{fact:spec}
     For all $i \in [K]$, we have{
       $\ds(f_i^*, g_i^*) = \argmax_{(f, g) \in \cD_i}\,
       \ip{\gamma}{f \otimes g}$
     where we have recursively defined each $\cD_i$ as
$\cD_i = \{(f, g) \in  \spc{\cX} \times \spc{\cY}\colon \norm{f} = \norm{g} = 1 \text{ and } \ip{f}{f_j^*} = \ip{g}{g_j^*} = 0\text{ for all }j \in [i - 1]\}$.}
   \end{fact}

\begin{fact}[Eckart--Young--Mirsky theorem,  \citealt{eckart1936approximation}]%
  \label{fact:low:rank}
  For all $\gamma \in \spc{\cX \times \cY}$ and $k \geq 1$, we have 
  \begin{align*}
    \mdnl{k}(\gamma) = \argmin_{\gamma'\colon \rank(\gamma') \leq k} \norm{\gamma - \gamma'} = \argmin_{\substack{\gamma'\colon \gamma' = f \otimes g,\\
    f \in \spcn{\cX}{k},\, g \in \spcn{\cY}{k}}} \norm{\gamma - \gamma'}.
  \end{align*}
\end{fact}

Therefore, we refer to $\mdnl{k}(\gamma)$ as the rank-$k$ approximation of $\gamma$, and the remainder $r_k(\gamma)$ represents the approximation error. %

\subsubsection{Constrained Modal Decomposition}\label{sec:notation-pre:const-md}
{We then introduce the constrained modal decomposition, which provides an effective implementation of projection operators.}
For given subspace $\sspc{\cX}$ of $\spc{\cX}$ and subspace $\sspc{\cY}$ of $\spc{\cY}$, we define
\begin{gather}
  \md(\gamma|\sspc{\cX}, \sspc{\cY}) \defeq \argmin_{\substack{\gamma'\colon \gamma' = f\otimes g\\
  f \in \sspc{\cX}, g \in \sspc{\cY }}}
  \norm{\gamma - \gamma'},\\
  \mdn{k}(\gamma|\sspc{\cX}, \sspc{\cY}) \defeq \md\left(\gamma - \sum_{i = 1}^{k- 1}\mdn{i}(\gamma|\sspc{\cX}, \sspc{\cY}) \middle|\sspc{\cX}, \sspc{\cY}\right),
  \quad \text{for all $ k \geq 1$},
  \label{eq:mdn:constrain:def}
\end{gather}
where $\mdn{1}(\gamma|\sspc{\cX}, \sspc{\cY}) \defeq \md(\gamma|\sspc{\cX}, \sspc{\cY})$.
Similarly, we denote  $\mdnl{k}(\gamma|\sspc{\cX}, \sspc{\cY}) \defeq \sum_{i = 1}^k \mdn{i}(\gamma|\sspc{\cX}, \sspc{\cY})$.

We can extend the properties of modal decomposition to the constrained case. In particular, we have the following extension of  \factref{fact:low:rank}, of which a proof is provided in \appref{app:prop:md:constraint}.
\begin{proposition}
  \label{prop:md:constraint}
  Suppose $\sspc{\cX}$ and   $\sspc{\cY}$
are subspace of  $\spc{\cX}$ and $\spc{\cY}$, respectively.
Then, for all $\gamma \in \spc{\cX \times \cY}$ and $k \geq 1$, %
 we have $\mdn{k}(\gamma|\sspc{\cX}, \sspc{\cY}) = \mdn{k}(\proj{\gamma}{\sspc{\cX} \otimes \sspc{\cY}})$, and
    \begin{align}
      \mdnl{ k}(\gamma|\sspc{\cX}, \sspc{\cY}) = \mdnl{k}(\proj{\gamma}{\sspc{\cX} \otimes \sspc{\cY}}) = \argmin_{\substack{\gamma'\colon \gamma' = f \otimes g,\\
    f \in \sspcn{\cX}{k},\, g \in \sspcn{\cY}{k}}} \norm{ \gamma - \gamma'},
      \label{eq:lora:constraint}
    \end{align}
    where we have defined $\sspcn{\cX}{k} \defeq \left(\sspc{\cX}\right)^k$ and $\sspcn{\cY}{k} \defeq \left(\sspc{\cY}\right)^k$.%
\end{proposition}

Therefore, we can implement projection operators by solving an equivalent constrained low-rank approximation (or modal decomposition) problem.

\subsection{Statistical Dependence and Induced Features}
Given $(X, Y) \sim P_{X, Y}$, we consider the space $\spc{\cX \times \cY}$ with the metric distribution $R_{X, Y} = P_XP_Y$. Then, we can characterize the statistical dependence between $X$ and $Y$ by the CDK function $\lpmi_{X; Y}$, as defined in \eqref{eq:cdk:def}. {We can characterize the energy of $\lpmi_{X; Y}$
as $\norm{\lpmi_{X; Y}}^2$, also known as the \emph{mean square contingency} \citep{pearson1904mathematical}. Specifically, it is easy to verify that $\norm{\lpmi_{X; Y}} = 0$ if and only if $X$ and $Y$ are independent. }

Suppose $\rank(\lpmi_{X; Y}) = K$ and let the modal decomposition be [cf. \eqref{eq:modal:dcmp}] %
\begin{align}
  \lpmi_{X; Y} = \sum_{i = 1}^K \mdn{i}(\lpmi_{X; Y}) = \sum_{i = 1}^K  \sigma_i \cdot ( f_i^* \otimes g_i^*),
  \label{eq:dcmp:xy}
\end{align}                                                       where for each $i \in [K]$, $\mdn{i}(\lpmi_{X;Y})= \sigma_i \cdot ( f_i^* \otimes g_i^*)$ is the standard form of $i$-th rank-one dependence mode, with strength
characterized by   $\sigma_i = \norm{\mdn{i}(\lpmi_{X; Y})}$. Note that since different modes are orthogonal (cf. \factref{fact:md:ortho}), we have $\bnorm{\lpmi_{X; Y}}^2 = \sum_{i = 1}^K \sigma_i^2$.
From $\sigma_1 \geq \dots \geq \sigma_K$, these modes are ordered by their contributions to the joint dependence.

In particular, the features $f_i^*$'s, $g_i^*$'s are the maximally correlated features in $\spc{\cX}$, $\spc{\cY}$, known as Hirschfeld--Gebelein--R\'{e}nyi (HGR) maximal correlation functions \citep{hirschfeld1935connection, gebelein1941statistische, renyi1959measures}. To see this, let us denote the covariance for given $f \in \spc{\cX}, g \in \spc{\cY}$ as 
  \begin{align}
    \corr(f, g) \defeq \Ed{P_{X, Y}}{f(X)g(Y)} - \Ed{P_{X}P_{Y}}{f(X)g(Y)}.
    \label{eq:corr:def}
  \end{align}
 From  \factref{fact:spec} and the fact $\corr(f, g) = \ip{\lpmi_{X; Y}}{f \otimes g}$,  we obtain the following corollary.
\begin{corollary}[HGR Maximal Correlation Functions]
  \label{cor:hgr}
  For each $i = 1, \dots, K$, we have $\sigma_i = \corr(f_i^*, g_i^*) = \Ed{P_{X, Y}}{f_i^*(X)g_i^*(Y)}$ and %
{  %
   $\ds (f_i^*, g_i^*) = \argmax_{(f, g) \in \cD_i}\, \corr(f_i, g_i)$,
     where we have recursively defined each $\cD_i$ as
$\cD_i = \{(f, g) \in  \spc{\cX} \times \spc{\cY}\colon \norm{f} = \norm{g} = 1 \text{ and } \ip{f}{f_j^*} = \ip{g}{g_j^*} = 0\text{ for all }j \in [i - 1]\}$.
}
\end{corollary}

We can further extend the results to the constrained modal decomposition (cf. \secref{sec:notation-pre:const-md}) of $\lpmi_{X; Y}$. Specifically, given subspaces $\sspc{\cX}$ and $\sspc{\cY}$ of $\spc{\cX}$ and $\spc{\cY}$, respectively, let 
\begin{align}
  \mdn{i}(\lpmi_{X; Y}|\sspc{\cX}, \sspc{\cY}) = \sigmah_i \cdot (\fh_i^* \otimes \gh_i^*), \quad i \geq 1, %
\end{align} %
be corresponding standard form representations.
Then, we can interpret $\sigmah_i, \fh_i^*, \gh_i^*$ as the solution to a constrained maximal correlation problem, formalized as the following extension of \corolref{cor:hgr}. A proof is provided in \appref{app:prop:max:corr:const}.

  \begin{proposition}
    \label{prop:max:corr:const}
    Given subspaces $\sspc{\cX}$ and $\sspc{\cY}$ of $\spc{\cX}$ and $\spc{\cY}$, respectively, for each $i \geq 1$, we have
$\sigmah_i = \corr(\fh_i^*, \gh_i^*) = \bEd{P_{X, Y}}{\fh_i^*(X)\gh_i^*(Y)}$,
    $\ds  (\fh_i^*, \gh_i^*) = \argmax_{(f, g) \in \cDh_i}\, \corr(f, g)$,
    where $\corr$ denotes the covariance [cf. \eqref{eq:corr:def}], and 
{ where we have recursively defined each $\cDh_i$ as
$\cDh_i = \{(f, g) \in  \sspc{\cX} \times \sspc{\cY}\colon \norm{f} = \norm{g} = 1 \text{ and } \ip{f}{\fh_j^*} = \ip{g}{\gh_j^*} = 0\text{ for all }j \in [i - 1]\}$.}
\end{proposition}

In particular, we can interpret CCA (Canonical Correlation Analysis) as the modal decomposition constrained to linear functions.

\begin{example}[Canonical Correlation Analysis]
  \label{eg:cca}
  Suppose $\cX$ and $\cY$ are vector spaces, and $\sspc{\cX}, \sspc{\cY}$ are the space of all linear functions defined on $\cX$, $\cY$, respectively. Then, \propref{prop:max:corr:const} gives solutions to CCA (Canonical Correlation Analysis) \citep{hotelling1936relations}, where
$\sigmah_i$'s are canonical correlations.
\end{example}

\paragraph{Weak Dependence and Local Geometric Analyses}
In the particular case where the statistical dependence between $X$ and $Y$ is weak, we can establish further connections between feature geometry and conventional information measures. Such analyses have been extensively studied in \cite{HuangMWZ2024}, referred to as the local geometric analysis, formalized as follows.

\begin{definition}[$\eps$-Dependence]
 Given $(X, Y) \sim P_{X, Y}$, $X$ and $Y$ are $\eps$-dependent if $\norm{\lpmi_{X; Y}} = O(\eps)$.
\end{definition}

For such weakly dependent variables, we can characterize their mutual information as follows.

\begin{lemma}[{\citealt[Lemma 4.11]{HuangMWZ2024}}]
  \label{lem:mutual}
  If  $X$ and $Y$ are $\eps$-dependent, then we have the mutual information
  $\ds I(X; Y) = \frac12 \cdot \bnorm{\lpmi_{X; Y}}^2 + o(\eps^2)$.
\end{lemma}
Therefore, from \eqref{eq:dcmp:xy} we can also decompose the mutual information into different modes: $ I(X; Y) = \frac12\cdot\bnorm{\lpmi_{X; Y}}^2 + o(\eps^2) = \frac12\sum_{i = 1}^K \sigma_i^2 + o(\eps^2)$.

\section{Dependence Approximation and Feature Learning} %
\label{sec:dep:approx} %
In this section, we demonstrate the learning system design with feature geometry in a learning setting. In particular, we consider optimal feature representations of the statistical dependence, and present learning such features from data and assembling them to build inference models.
To begin, let $X$ and $Y$ denote the random variables of interest, with the joint distribution $P_{X, Y}$. We characterize the statistical dependence between $X$ and $Y$ as the CDK function [cf. \eqref{eq:cdk:def}] $\lpmi_{X; Y} \in \spc{\cX \times \cY}$. In our development, we consider the feature geometry  on $\spc{\cX \times \cY}$ with respect to the metric distribution $R_{X, Y} = P_XP_Y$. We also assume $\lpmi_{X; Y}$ has the modal decomposition \eqref{eq:dcmp:xy}.

\subsection{Low Rank Approximation of Statistical Dependence}%

In learning applications, the joint distribution $P_{X, Y}$ is typically unknown with enormous complexity, making direct computation or estimation of $\lpmi_{X; Y}$ infeasible. To tackle this problem, we consider the representation of $\lpmi_{X; Y}$ using features of $X$ and $Y$, and develop the feature learning algorithms that can be effectively implemented on $(X, Y)$ samples.

Specifically, for given $k \geq 1$ and $k$-dimensional features $f \in \spcn{\cX}{k}$ and $g \in \spcn{\cY}{k}$, we consider the approximation of $\lpmi_{X; Y}$ by the rank-$k$ joint function $f \otimes g = \sum_{i = 1}^k f_i \otimes g_i$. With this formulation, we can convert the computation of the rank-$k$ approximation $\mdnl{k}(\lpmi_{X; Y})$ to an optimization problem, where the objective is the approximation error $\norm{\lpmi_{X: Y} - f \otimes g}$, and where the optimization variables are $k$-dimensional features $f$ and $g$. Then,  $\mdnl{k}(\lpmi_{X; Y})$ can be represented 
 by the resulting optimal features in a factorized form.

However, we cannot directly compute the error $\norm{\lpmi_{X: Y} - f \otimes g}$ for given $f$ and $g$, due to the unknown $\lpmi_{X: Y}$. %
To address this issue, we introduce the H-score, proposed in  \citep{xu2020maximal, xu2022information}.%

\begin{definition}
  \label{def:H}
Given %
$k \geq 1$ and $f \in \spcn{\cX}{k}$, $g \in \spcn{\cY}{k}$, the H-score $\Hs(f, g)$ is defined as
\begin{align}
  \Hs(f, g)
  &\defeq \frac12\left(\bnorm{\lpmi_{X; Y}}^2 - \bbnorm{\lpmi_{X; Y} - f \otimes g}^2\right)  \label{eq:H}
\\
  &= \E{f^{\T}(X) g(Y)}
    - \left(\E{f(X)}\right)^{\T}\E{g(Y)}%
     - \frac12  \cdot \tr\left(\La_{f}\La_{g}\right),
     \label{eq:H:def}
\end{align}
where $\La_f = \E{f(X)f^\T(X)}$,  $\La_g = \E{g(Y)g^\T(Y)}$.
\end{definition}

The H-score measures the goodness of the approximation, with a larger H-score value  indicating a smaller approximation error. In particular, for $k$-dimensional feature inputs, the maximum value of H-score gives the total energy of top-$k$ dependence modes, achieved by the optimal rank-$k$ approximation. Formally, we have the following property
 from \factref{fact:low:rank}.

. %
\begin{property}
  \label{propty:hscore:max}
Given %
$k \geq 1$, let  $\sigma_i = \norm{\!\mdn{i}(\lpmi_{X; Y})}$  for $i \in [k]$. Then, for all $f \in \spcn{\cX}{k}$ and $g \in \spcn{\cY}{k}$, 
\begin{align}
\Hs(f, g) \leq \frac12 \norm{\!\mdnl{k}(\lpmi_{X; Y})}^2
 = \frac12 \sum_{i = 1}^k \sigma_i^2,\label{eq:Hs:leq}  
\end{align}
 where the inequality holds with equality if and only if
 $f \otimes g = \mdnl{k}(\lpmi_{X; Y})$.
\end{property}

In practice, for given features $f$ and $g$, 
we can efficiently compute the H-score $\Hs(f, g)$  from data samples,  by evaluating corresponding empirical averages in \eqref{eq:H:def}. Since $\Hs(f, g)$ is differentiable with respect to $f$ and $g$, we can use it as the training objective for learning the low-rank approximation of $\lpmi_{X; Y}$, where we use neural networks to parameterize $f$ and $g$ and optimize their parameters by batch (minibatch) gradient descent. Suppose the networks have sufficient expressive power, then the optimal solution gives the desired low-rank approximation $\mdnl{k}(\lpmi_{X; Y})$.

{
\begin{remark}
It is worth mentioning that the optimal features learned from finite data samples correspond to the modal decomposition of the associated empirical distribution   (cf. \secref{sec:geo:data}), which is generally different from the underlying distribution. As a result, the learned features will deviate from the theoretical values of features; see, e.g., \cite{huang2020sample, makur2020estimation}
for detailed discussions on the sample complexity of learning such features.
\end{remark}
}

Note that in this particular bivariate setting, the roles of $X$ and $Y$ (and the learned features $f$ and $g$) are symmetric. Moreover, we learn the features by directly factorizing the statistical dependence between $X$ and $Y$, instead of solving a specific inference task, e.g., predicting $Y$ based on $X$, or vice versa. Nevertheless, we can readily solve these inference tasks by simply assembling the learned features, as we will demonstrate next.

\subsection{Feature Assembling and Inference Models} %

With features $f \in \spcn{\cX}{k}, g \in \spcn{\cY}{k}$ learned from maximizing the H-score $\Hs(f, g)$, we then discuss the construction of different inference models by assembling these features. We first consider the case %
where $k \geq \rank(\lpmi_{X; Y})$ and we have learned
$f \otimes g = \lpmi_{X; Y}$ (cf. \proptyref{propty:hscore:max}). Then, we have the following proposition. A proof is provided in \appref{app:prop:pred:est}.
\begin{proposition}
  \label{prop:pred:est}
  Suppose $f \otimes g = \lpmi_{X; Y}$. Then, we have $    \norm{\lpmi_{X; Y}}^2 = \tr(\La_f  \La_g)$ and
  \begin{align}
    P_{Y|X}(y|x) =  P_Y(y) \left(1 + f^\T(x) g(y)\right).
    \label{eq:p:ygx}
  \end{align}
  In addition, for any $d$-dimensional function $\psi \in \spcn{\cY}{d}$, we have
  \begin{align}
    \E{\psi(Y)|X = x} = \E{\psi(Y)} +  \La_{\psi, g}  f(x),
  \label{eq:estimation}
  \end{align}
  {where $\La_{\psi, g} = \E{\psi(Y)g^\T(Y)}$.}
\end{proposition}

Therefore, we can compute the strength of $(X;Y)$ dependence, i.e., $\norm{\lpmi_{X; Y}}$ from the features $f$ and $g$. In addition, the posterior distribution \eqref{eq:p:ygx} and conditional expectation \eqref{eq:estimation} are useful for  supervised learning tasks. Specifically, we consider the case where 
 $X$ and $Y$ are the input variable and target variable, respectively. The $Y$ represents the categorical label in classification tasks, or the target to estimate in regression tasks.

In classification tasks, we can compute the posterior distribution $P_{Y|X}$ of the label $Y$ from \eqref{eq:p:ygx}. The corresponding corresponding MAP (maximum a posteriori) estimation is 
\begin{align}
  \yh_{\MAP}(x) = \argmax_{y \in \cY}P_{Y|X}(y|x) = \argmax_{y \in \cY} P_Y(y) \left(1 + f^\T(x) g(y)\right),
  \label{eq:map}
\end{align}
where $P_{Y}$ can be obtained from training set. This approach is also referred to as the maximal correlation regression (MCR) \citep{xu2020maximal}. {Similarly, the maximum likelihood estimation (MLE)
is given by
\begin{align}
  \yh_{\MLE}(x) = \argmax_{y \in \cY}P_{X|Y}(x|y) = \argmax_{y \in \cY} f^\T(x) g(y).
  \label{eq:ml}
\end{align}
}

If the target variable $Y$ is continuous, it is often of interest to estimate $Y$, or more generally, some transformation $\psi$ of $Y$. Then, the MMSE (minimum mean square error) estimation of $\psi(Y)$ based on $X = x$ is the conditional expectation $\E{\psi(Y)|X = x}$.
From \eqref{eq:estimation}, we can efficiently compute the conditional expectation, where
 $\E{\psi(Y)}$ and $\La_{\psi, g} = \E{\psi(Y)g^\T(Y)}$ can be evaluated from the training dataset by taking the corresponding empirical averages. Therefore, we obtain the model for estimating $\psi(Y)$ for any given $\psi$, by simply assembling the learned features without retraining.

 In practice, it can happen that feature dimension $k <  \rank(\lpmi_{X; Y})$, due to a potentially large $\rank(\lpmi_{X; Y})$. In such case, the best approximation of $\lpmi_{X; Y}$ would be 
the rank-$k$ approximation $\mdnl{k}(\lpmi_{X; Y})$, and we can establish a similar result as follows.  A proof is provided in \appref{app:prop:est:k}.

\begin{proposition}
  \label{prop:est:k}
  Suppose $f \otimes g = \mdnl{k}(\lpmi_{X; Y})$ for $k \geq 1$. Then, for all $d$-dimensional function $\psi \in \spnn{d}\{\cst{\cY}, g_{1}^*, \dots, g_k^*\}$, we have
    $\E{\psi(Y)|X = x} = \E{\psi(Y)} +  \La_{\psi, g}  f(x)$,
where $\cst{\cY}$ is the constant function $(y \mapsto 1) \in \spc{\cY}$,
 and where for each $i \in [k]$, $g_i^*$ is obtained from the standard form of $\mdn{i}(\lpmi_{X; Y})$: %
 $\sigma_i(f_i^* \otimes g_i^*) = \mdn{i}(\lpmi_{X; Y})$.
\end{proposition}

\subsection{Constrained Dependence Approximation }
\label{sec:h:const}

We can readily extend the above analysis to the constrained low-rank approximation problem. Specifically,  we consider the constrained rank-$k$ approximation [cf. \eqref{eq:lora:constraint}] $\mdnl{k}(\lpmi_{X; Y}|\sspc{\cX}, \sspc{\cY})$ for $k \geq 1$, where $\sspc{\cX}$ and $\sspc{\cY}$ are subspaces of $\spc{\cX}$ and $\spc{\cY}$, respectively.
Analogous to \proptyref{propty:hscore:max}, when we restrict $f \in  \sspcn{\cX}{k}$ and $g \in \sspcn{\cY}{k}$, the H-score $\Hs(f, g)$ is maximized if and only if $f \otimes g = \mdnl{k}(\lpmi_{X; Y}|\sspc{\cX}, \sspc{\cY})$.

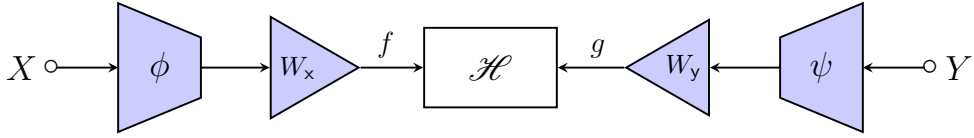
\begin{figure}[!ht]
  \centering
  \resizebox{.8\textwidth}{!}{%
\begin{tikzpicture}[auto, thick, node distance=2cm, >=stealth]
  \draw node [block] (H) at (0, 0) {\Large $\Hs$}
  node [extl, left of = H, node distance=2.8cm] (wx) {$\wx$}
  node [ext, left of = wx, rotate = -90] (phi) {}
  node [extl, right of = H, node distance=2.8cm, rotate = 180] (wy) {}
  node [ext, right of = wy, rotate = 90] (psi) {}
  node [input, left of = phi, node distance=1.5cm] (x)  {}
  node [xshift = -0.5cm] at (x) {\Large $X$}
  node [xshift = -0.7mm] at (x) {\Large \textopenbullet}
  node [input, right of = psi, node distance=1.5cm] (y)  {}
  node [xshift = 0.5cm] at (y){\Large $Y$} 
  node [xshift = 0.7mm] at (y) {\Large \textopenbullet} 
  node at (phi) {\Large$\phi$}
  node at (psi) {\Large$\psi$}
  node at (wy) {$\wy$};
  \draw [->] (x) -- (phi);
  \draw[->] (phi) -- (wx);
  \draw[->] (wx) -- (H) node [pos = .4] {\large $f$};
  \draw[->] (y) -- (psi);
  \draw[->] (psi) -- (wy);
  \draw [->] (wy) -- (H) node [pos = .4, above] {\large $g$};  
\end{tikzpicture}

}
  \caption{Features $f$, $g$ as the output of linear layers. The linear layers are represented as triangle modules, with inputs $\phi, \psi$, and weight matrices $\wx$, $\wy$, respectively.} %
  \label{fig:wx:wy}
\end{figure}

As an application, we can characterize the effects of the restricted expressive power on feature learning. To begin, we consider the maximization of H-score $\Hs(f, g)$, where features $f$ and $g$ are $k$-dimensional outputs of neural networks. In particular, we assume the last layers of the networks are linear layers, which is a common network architecture design in practice. The
 overall network architecture is shown in \figref{fig:wx:wy}, where
we express $f$ as the composition of feature extractor
 $\phi \in \spcn{\cX}{\dx}$ and the last linear layer 
with weight matrix $\wx \in \mathbb{R}^{k\times \dx}$. Similarly, we represent $g$ as the composition of $\psi \in \spcn{\cY}{\dy}$ and 
the linear layer with weight $\wy \in \mathbb{R}^{k\times \dy}$.

Suppose we have trained the weights $\wx, \wy$ and the parameters in $\phi, \psi$ to maximize the H-score $\Hs(f, g)$, and the weights $\wx$ and $\wy$ have converged to their optimal values with respect to $\phi$ and $\psi$. Note that for any given $\phi, \psi$, $f = \wx \phi$ takes values from the set $\{\wx \phi\colon \wx \in \mathbb{R}^{k\times \dx} \} = \spnn{k}\{\phi\}$, and, similarly, $g  = \wy \psi$ takes values from $\spnn{k}\{\psi\}$. Therefore, the optimal $(f, g)$ corresponds to the solution of a constrained low-rank approximation problem, %
and we have
$f \otimes g = \mdnl{k}(\lpmi_{X; Y}|\spn\{\phi\}, \spn\{\psi\})$.

In addition, from \propref{prop:md:constraint} and the orthogonality principle, we can express the approximation error as 
  \begin{align}
    \bnorm{\lpmi_{X; Y} - f \otimes g}^2 
    &=  \bnorm{\lpmi_{X; Y} - \lpmi'_{X; Y} + \lpmi'_{X; Y} - \mdnl{k}(\lpmi'_{X; Y})}^2 \notag\\
    &= \norm{\lpmi_{X; Y} - \lpmi'_{X; Y}}^2 + \norm{r_k(\lpmi'_{X; Y})}^2,
    \label{eq:fg:err:w}
  \end{align}
where $\lpmi'_{X; Y} \defeq \proj{\lpmi_{X; Y}}{\spn\{\phi\} \otimes \spn\{\psi\}}$. Note that \eqref{eq:fg:err:w} decomposes the overall approximation into two terms, where the first term characterizes the effects of insufficient expressive power of $\phi, \psi$, and the second term characterizes the impacts of feature dimension $k$.

\subsection{Relationship to Classification DNNs }
\label{sec:h:dnn}
We conclude this section by discussing a relation between the dependence approximation framework and deep neural networks, studied in \cite{xu2022information}. We consider a classification task where $X$ and $Y$ denote the input data and the target label to predict, respectively. Then, we can interpret the log-likelihood function of DNN as an approximation of the H-score, and thus DNN also learns strongest modes of  $(X; Y)$ dependence.

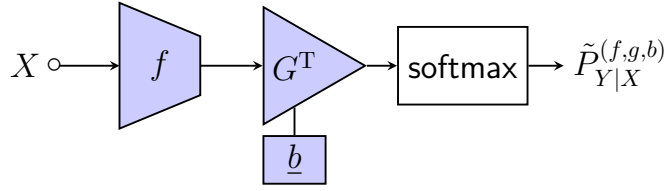
\begin{figure}[!ht]
  \centering  \resizebox{.55\textwidth}{!}{%
\begin{tikzpicture}[auto, thick, node distance=2cm, >=stealth]

  \draw node [block] (sl) at (0, 0) {\Large \sf softmax}
  node [extl, left of = sl, node distance=2.5cm] (G) {\Large $G^\T$}
  node [bias, below of = G, node distance=1.4cm] (b) {\Large $\vec{b}$}
  node [ext, left of = G, rotate = -90] (f) {}
  node [input, left of = f, node distance=1.5cm] (x)  {}
  node [xshift = -0.5cm] at (x) {\Large $X$}
  node [xshift = -0.7mm] at (x) {\Large \textopenbullet}
  node at (f) {\Large$f$};
  \draw [->] (x) -- (f);
  \draw [-] (G) -- (b);
  \draw[->] (f) -- (G);
  \draw[->] (G) -- (sl);
  \draw[->] (sl) -- ++ (1.5, 0) node [right] {\Large $\Pt^{(f, g, b)}_{Y|X}$};

\end{tikzpicture}

}
  \caption{A classification DNN for predicting label $Y$ based on the input $X$. All layers before the classification layer are represented as feature extractor $f$.  The weight and bias associated with each class $Y = y$ are denoted by $g(y)$ and $b(y)$, respectively, which give weight matrix $G^\T$ and bias vector $\vec{b}$ with $G = [g(1), \dots, g(|\cY|)], \vec{b} = [b(1), \dots, g(|\cY)]^\T$. The softmax module outputs a posterior probability,     parameterized by  $f$, $g$ and $b$.%
}
  \label{fig:dnn}
\end{figure}

To begin, let $\{(x_i, y_i)\}_{i=1}^n$ denote the training data, with empirical distribution $P_{X, Y}$ as defined in \eqref{eq:emp:dist:xy}. We depict the architecture of typical classification DNN
 in \figref{fig:dnn}, where we abstract all layers before classification layer as a $k$-dimensional feature extractor $f \in \spcn{\cX}{k}$. %
 The feature $f$ is then processed by 
a classification layer with weight matrix  $G^\T$ and the bias vector $\vec{b}$, and activated by the softmax function\footnote{The softmax function is defined such that, for all $k > 1$ and each $\vec{v} = (v_1, \dots, v_k)^\T \in \mathbb{R}^k$, we have $\softmax(\vec{v}) \in \mathbb{R}^k$, with each $i$-th entry being $\ds[\softmax(\vec{v})]_i \defeq \frac{\exp(v_i)}{\sum_{j = 1}^k \exp(v_j)}$, $i \in [k]$. }. Without loss of generality, we assume $\cY = \{1, \dots, |\cY|\}$, then we can represent
  $G$ and $\vec{b}$ as
\begin{align}
  G \defeq [g(1), \dots, g(|\cY|)] \in \mathbb{R}^{k \times |\cY|},\quad
  \vec{b} \defeq [b(1), \dots, b(|\cY|)]^\T \in \mathbb{R}^{|\cY|},
  \label{eq:def:G:b}
\end{align}
where $g(y) \in \mathbb{R}^k$ and $b(y)\in \mathbb{R}$ denote the weight and bias associated with each class $Y = y$, respectively.  
Then, the softmax output of $(G^\T f(x) + \vec{b})$ gives a parameterized posterior %
\begin{align}
  \Pt^{(f, g, b)}_{Y|X}(y|x)
  \defeq \frac{\exp(f(x) \cdot g(y) + b(y))}{\sum_{y' \in \cY}\exp(f(x) \cdot g(y') + b(y'))}.
  \label{eq:Pt}
\end{align}

\newcommand{\llog}{\mathcal{L}} %

The network parameters are trained to maximize\footnote{{This is equivalent to minimizing $-\llog(f, g, b)$, i.e., the log loss (cross entropy loss).}}
 the resulting log-likelihood\footnote{Throughout our development, all
  logarithms are base $e$, i.e., natural.}
 function
\begin{align}
 \llog(f, g, b) \defeq
 \frac{1}{n} \sum_{i = 1}^n \log\Pt^{(f, g, b)}_{Y|X}(y_i|x_i)  
  =\Ed{(\Xh, \Yh)\sim P_{X, Y}}{\log\Pt^{(f, g, b)}_{Y|X}(\Yh|\Xh)}.
  \label{eq:likelihood:def}
\end{align}
We further define $\ds \llog(f, g) \defeq \max_{b \in \spc{\cY}}\llog(f, g, b)$, by setting the bias $b$ to its optimal value with respect to given $f$ and $g$. It can be verified that $\llog(f, g)$ depends only on the centered versions of $f$ and $g$, formalized as follows. A proof is provided in \appref{app:propty:dnn:center}.
\begin{property}
\label{propty:dnn:center}
  For all $k \geq 1$ and $f \in \spcn{\cX}{k}$, $g \in \spcn{\cY}{k}$, we have $\llog(f, g) = \llog(\ft, \gt)$, where we have defined  $\ft \in \spcn{\cX}{k}, \gt \in \spcn{\cY}{k} $ as $\ft \defeq \proj{f}{\spct{\cX}}, \gt \defeq \proj{g}{\spct{\cY}}$, i.e.,
 $\ft(x) = f(x) - \E{f(X)}$,  $\gt(y) = g(y) - \E{g(Y)}$, for all $x \in \cX, y \in \cY$.
\end{property}

Therefore, it is without loss of generalities to restrict our discussions to zero-mean $f$ and $g$. Specifically, we can verify that for the trivial choice of feature $f = 0$, the resulting likelihood function is $\llog(0, g) = \llog(0, 0) = -H(Y)$, achieved when the posterior distribution satisfies
$\Pt_{Y|X}^{(0, g, b)} = P_Y$, where $H(\cdot)$ denotes the Shannon entropy. In general, we have the following characterization of $\llog(f, g)$, which extends \citep[Theorem 4]{xu2022information}. A proof is provided in \appref{app:prop:dnn}.

\begin{proposition}
  \label{prop:dnn}
  Suppose $X$ and $Y$ are $\eps$-dependent. For all $k\geq 1$, and $f \in \spctn{\cX}{k}$, $g \in \spctn{\cY}{k}$,  if $\llog(f, g) \geq \llog(0, 0) = -H(Y)$, then we have $\norm{f \otimes g} = O(\eps)$, and
  \begin{align}
    \llog(f, g) 
    & = \llog(0, 0) + \underbrace{\frac{1}{2} \cdot \left(\norm{\lpmi_{X; Y}}^2 - \bnorm{\lpmi_{X; Y} - f \otimes g}^2
 \right)}_{=\Hs(f, g)} + o(\eps^2)    \label{eq:l:fg}%
  \end{align}
which is maximized if and only if %
 $f \otimes g = \mdnl{k}(\lpmi_{X, Y}) + o(\eps)$.
\end{proposition}

From \propref{prop:dnn}, the H-score $\Hs(f, g)$ coincides with likelihood function $\llog(f, g)$ in the local regime. For a fully expressive feature extractor $f$ of dimension $k$, the optimal feature $f$ and weight matrix $G^\T$ are approximating the  rank-$k$ approximation of $(X; Y)$ dependence. In this sense, the weight matrix $G^\T$ in classification DNN
essentially characterizes a feature of the label $Y$, 
with a role symmetric to feature extractor $f$. 
However, unlike the H-score implementation, the classification DNN is restricted to categorical $Y$ to make the softmax function \eqref{eq:Pt} computable.

{
\begin{remark}
For general $X, Y$,
 the optimal features $f, g$ that optimize the likelihood function \eqref{eq:likelihood:def} can deviate from the low-rank approximation characterization in \propref{prop:dnn} and have different behaviors. See \cite{xu2018geometric} for detailed discussions.
\end{remark}
}

\section{Nesting Technique for Dependence Decomposition}
\label{sec:nest}
In multivariate learning applications, it is often difficult to summarize the statistical dependence as some bivariate dependence. Instead, the statistical dependence of interest is typically only a component decomposed from the original dependence. In this section, we introduce a nesting technique, 
{which allows us to implement such dependence decomposition operations by training corresponding neural feature extractors.} 
For the ease of presentation, we adopt the bivariate setting introduced previously and consider the feature geometry on $\spc{\cX \times \cY}$ with metric distribution $P_{X}P_Y$. We will discuss the multivariate extensions in later sections.

\subsection{Nesting Configuration and Nested H-score}

The nesting technique is a systematic approach to learn features representing projected dependence components or their modal decomposition. In particular, for a given dependence component of interest, we can construct corresponding training objective for learning the dependence component. 
The resulting training objective is an aggregation of different H-scores, where the inputs to these H-scores are features forming a nested structure. We refer to such functions as the \emph{nested H-scores}. To specify a nested H-score, we introduce its configuration, referred to as the \emph{nesting configuration}, defined as follows.

\begin{definition}
\label{def:nest}
Given $\cX, \cY$ and $k \geq l \geq 1$, we define an $l$-level nesting configuration for $k$-dimensional features as the tuple $\left\{(d_1, \dots, d_l);\,
  \bigl(\sspcs{\cX}{1}, \dots,   \sspcs{\cX}{l}\bigr);\,
  \sspc{\cY}
  \right\}$, 
where  %
\begin{itemize}
\item $(d_1,  \cdots, d_l)$ is a sequence with $d_i > 0$ and $\sum_{i = 1}^l d_i = k$;
\item $\bigl(\sspcs{\cX}{1}, \dots,   \sspcs{\cX}{l}\bigr)$ is an increasing sequence of $l$ subspaces of $\spc{\cX}$:
 $\sspcs{\cX}{1} \subset \dots \subset \sspcs{\cX}{l}$;
\item $\sspc{\cY}$ is a subspace of $\spc{\cY}$.
\end{itemize}
\end{definition}

\paragraph{Nested H-score}

Given a nesting configuration $\cC = \left\{(d_1, \dots, d_l);\,
  \bigl(\sspcs{\cX}{1}, \dots,   \sspcs{\cX}{l}\bigr);\,
  \sspc{\cY}
  \right\}$ for $k$-dimensional features, the associated nested H-score is a function of $k$-dimensional feature pair $f$ and $g$, which we denote by $\Hs(f, g; \cC)$, specified as follows. To begin, let us define $\dl{i} \defeq \sum_{j = 1}^i d_j$ for each $ 0 \leq i \leq l$, representing the total dimension up to $i$-th level. Then, we define the domain of $\Hs(f, g; \cC)$, denoted by $\dom(\cC)$,  as
 \begin{align}
 \dom(\cC) \defeq \left\{(f, g) \colon f \in \spcn{\cX}{k}, g \in \sspcn{\cY}{k},
   f_j \in \sspcs{\cX}{i}, \text{ for all $\dl{i-1} < j \leq \dl{i}$}
   \right\}.
   \label{eq:dom:c}
 \end{align}%
Then, for $(f, g) \in \dom(\cC)$ and each $i \in [l]$, we obtain the H-score $\Hs(f_{[\dl{i}]}, g_{[\dl{i}]})$ by taking the first $\dl{i}$ dimensions of $f, g$. We define the nested H-score $\Hs(f, g; \cC)$ by 
taking the sum of these $l$ H-scores, %
\begin{align}
  \Hs\left(f, g; \cC \right) \defeq
  \sum_{i = 1}^l \Hs(f_{[\dl{i}]}, g_{[\dl{i}]}),\qquad (f, g) \in \dom(\cC).
  \label{eq:hnest:def}
\end{align}

From \eqref{eq:hnest:def}, the nested H-score aggregates different H-scores with nested input features. The nested structure of features is specified by the increasing sequence of dimension indices:
 $[\dl{1}] \subset \dots \subset [\dl{l}] = [k]$, determined by the sequence $(d_1, \dots, d_l)$. The domain of features is specified by
subspaces in the configuration. When $\sspcs{\cX}{i} = \sspc{\cX}$ for all $i \in [l]$, we can simply write the configuration as $\cC = \left\{(d_1, \dots, d_l);\,
  \sspc{\cX};\,
  \sspc{\cY}
  \right\}$ without ambiguity. In particular, we can represent the original H-score for $k$-dimensional input features
as a nested H-score configured by $\{k;\, \spc{\cX} ; \,\spc{\cY}\}$ .

\begin{remark}
{Note that the nested H-score \eqref{eq:hnest:def} is obtained by using a sum function to aggregate different H-score terms $\Hs(f_{[\dl{i}]}, g_{[\dl{i}]})$, $ i = 1, \dots, l$. We shall comment
that the choice of such aggregation functions is not unique.  Generally, for an $l$-level nesting configuration, we can apply any differentiable  $\Gamma\colon \mathbb{R}^l \to \mathbb{R}$ as an aggregation function if $\Gamma$ is strictly increasing in each argument. The aggregated result 
$\Gamma(\Hs(f_{[\dl{1}]}, g_{[\dl{1}]}), \dots, \Hs(f_{[\dl{l}]}, g_{[\dl{l}]}))$ defines a nested H-score that satisfies the same collection of properties.
}
 For the ease of presentation, we adopt the sum form  \eqref{eq:hnest:def} throughout our development, but also provide general discussions in \appref{app:opt} for completeness.
\end{remark}

\begin{remark}
By symmetry, 
we can also define the configuration
$\left\{(d_1, \dots, d_l);\,
  \sspc{\cX};\,
  \bigl(\sspcs{\cY}{1}, \dots,   \sspcs{\cY}{l}\bigr)
  \right\} $ and the associated nested H-score,
for subspaces $\sspc{\cX}$ of $\spc{\cX}$ and $\sspcs{\cY}{1} \subset \dots \subset \sspcs{\cY}{l}$ of $\spc{\cY}$.  %
\end{remark}

\paragraph{Refinements of Nesting Configuration}

Given a nesting configuration for $k$-dimensional features $\cC = \left\{(d_1, \dots, d_l);~
  \bigl(\sspcs{\cX}{1}, \dots,   \sspcs{\cX}{l}\bigr);~
  \sspc{\cY}
  \right\}$, the sequence $(d_1,\dots, d_l)$ defines a partition that separates the $k$ dimensions into $l$ different groups. By refining such partition, we can construct new configurations with higher levels, which we refer to as refined configurations.
In particular, the finest refinement corresponds to the partition where each group has only one dimension. We use $\refine{\cC}$ to denote the finest refinement of $\cC$, given by
 \begin{align}
  \refine{\cC} \defeq   \left\{\ones{k};\,
  \Bigl(\bigl(\sspcs{\cX}{1}\bigr)^{d_1}, \dots,   \bigl(\sspcs{\cX}{l}\bigr)^{d_l}\Bigr);\,
  \sspc{\cY}
  \right\}, 
   \label{eq:refine}
\end{align}
where we have used $\ones{k}$ to denote the all-one sequence of length $k$, and where $\Bigl(\bigl(\sspcs{\cX}{1}\bigr)^{d_1}, \dots,   \bigl(\sspcs{\cX}{l}\bigr)^{d_l}\Bigr)$ represents the length-$k$ sequence
starting with $d_1$ terms of $\sspcs{\cX}{1}$, followed by $d_2$ terms of $\sspcs{\cX}{2}$, up to  $d_l$ terms of $\sspcs{\cX}{l}$.
From \eqref{eq:dom:c}, such refinements do not change the domain, and we have $\dom(\refine{\cC}) = \dom(\cC)$. The corresponding nested H-score is 
\begin{align}
  \Hs\left(f, g; \refine{\cC} \right) =
  \sum_{i = 1}^k \Hs(f_{[i]}, g_{[i]}), \quad (f, g) \in \dom(\cC).
  \label{eq:hnest:r:def}
\end{align}

\subsection{Nesting Technique for  Modal Decomposition }
We then demonstrate the application of nesting technique in learning modal decomposition. %
Given $k$-dimensional features $f, g$, we consider the nesting configuration
$\cnest$, which can also be obtained from the original H-score by the refinement \eqref{eq:refine}: $\cnest = \refine{\cnone}$.
The corresponding nested H-score is the sum of $k$ H-scores:
\begin{align}
  \Hs(f, g; \cnest) = \sum_{i = 1}^k \Hs(f_{[i]}, g_{[i]}).
  \label{eq:H:cnest}
\end{align}
Note that from \proptyref{propty:hscore:max}, for each $i \in [k]$, the H-score $\Hs(f_{[i]}, g_{[i]})$ is maximized if and only if $f_{[i]} \otimes g_{[i]} = \mdnl{i}(\lpmi_{X; Y})$. Therefore, all $k$ terms of H-scores are maximized simultaneously, if and only if we have
$f_{[i]} \otimes g_{[i]} = \mdnl{i}(\lpmi_{X; Y})$ for all $i \in [k]$. By definition, this is also equivalent to
\begin{align}
  f_i \otimes g_i = \mdn{i}(\lpmi_{X; Y}), \quad i \in [k].
  \label{eq:md:fg}
\end{align}

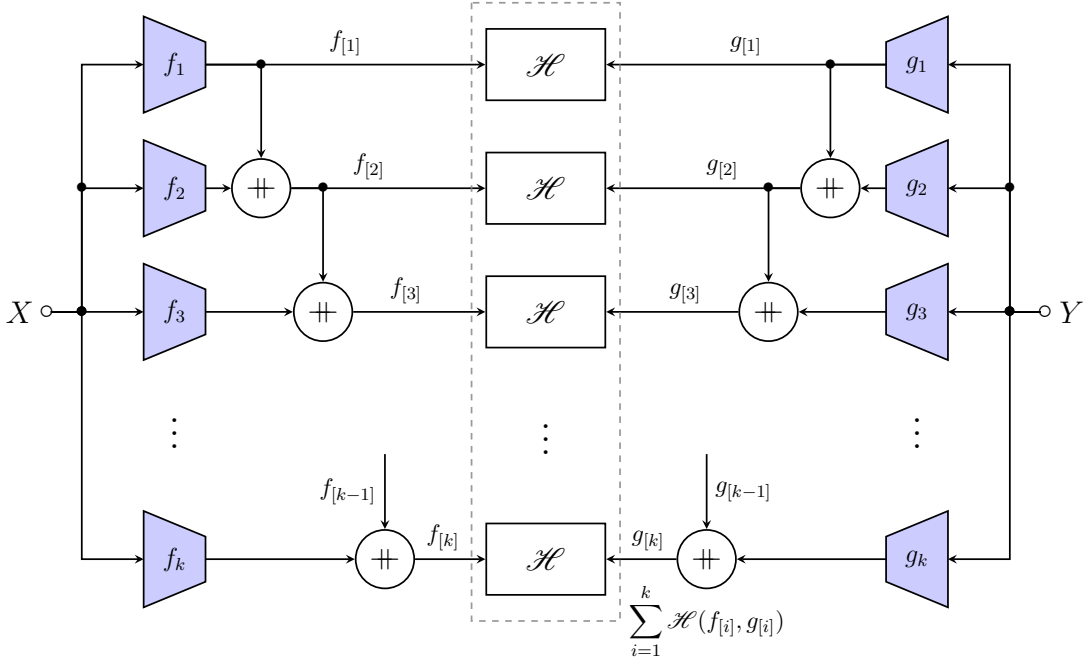
\begin{figure}[!t]
  \centering
  \resizebox{.9\textwidth}{!}{%
\def\dy{2}
\def\dx{.5}
\begin{tikzpicture}[auto, thick, node distance=2cm, >=stealth%
  ]
  \foreach \i/\text [count = \y] in {1/{$\Hs(f_{[1]},g_{[1]})$}, 2/{$\Hs(f_{[2]},g_{[2]})$}, 3/{$\Hs(f_{[3]},g_{[3]})$}, missing/, 4/{$\Hs(f_{[k]},g_{[k]})$}}
  {
    \draw node [block/.try, \i/.try] (H-\i) at (8, -\dy * \y cm) {};
}

  \foreach \i/\sub in {1/1, 2/2, 3/3, 4/k}
  {
    \draw node at (H-\i) [] {\Large $\Hs$};
    \draw node [exts, left of = H-\i, node distance=6cm, rotate = -90] (f-\i)  {} node at (f-\i) {\large $f_\sub$};
    \draw node [exts, right of = H-\i, node distance=6cm, rotate = 90] (g-\i)  {} node at (g-\i) {\large$g_\sub$};   
}
  \draw node [missing] at (f-1|-H-missing) {}
       node [missing] at (g-1 |- H-missing) {};

       \draw [->] (f-1) -- node {$f_{[1]}$}(H-1);
       \draw [->] (g-1) -- node [above] {$g_{[1]}$} (H-1);

  \foreach \i/\sup [count = \c] in {2/2, 3/3, 4/k}
  {
    \draw node [con, right of = f-\i, node distance=.4cm, xshift= \c cm] (conl-\c)  {\con};
    \draw node [con, left of = g-\i, node distance=.4cm, xshift = -\c cm] (conr-\c)  {\con};
    \draw [->] (f-\i) -- (conl-\c);
    \draw [->] (conl-\c) -- node [pos = .4] {$f_{[\sup]}$} (H-\i);
    \draw [->] (g-\i) -- (conr-\c);
    \draw [->] (conr-\c) -- node [above, pos = .4] {$g_{[\sup]}$} (H-\i);
}

\draw [->] (f-1) -| (conl-1);
\draw [->] (g-1) -| (conr-1);
\draw [->] (conl-1) -| (conl-2);
\draw [->] (conr-1) -| (conr-2);
\draw [->] ($(conl-3) + (0, 1.7)$) -- node [left] {$f_{[k-1]}$} (conl-3);
\draw [->] ($(conr-3) + (0, 1.7)$) -- node {$g_{[k-1]}$} (conr-3);

    \foreach \i in {1, 2}
    {
      \draw node at (conl-\i |- f-\i) {\textbullet} ;
      \draw node at (conr-\i |- g-\i) {\textbullet} ;
  }
\draw node [input] (x) at (0,-6) {}
     node [xshift = -0.5cm] at (x) {\Large $X$}
     node [xshift = -0.7mm] at (x) {\Large \textopenbullet};
        
     \draw node [input] (y) at (16,-6) {}
     node [xshift = 0.5cm] at (y){\Large $Y$} 
     node [xshift = 0.7mm] at (y) {\Large \textopenbullet} ;

  \foreach \i [count = \y] in {1, ..., 4}
  {
    \draw [->] (x) -- + (\dx, 0) node {\textbullet} |- (f-\i);
    \draw [->] (y) -- + (-\dx, 0) node {\textbullet} |- (g-\i);
  }

  \draw node at ($(x) + (\dx, \dy)$ ) {\textbullet}
  node at ($(y) + (-\dx, \dy)$ ) {\textbullet};

  \draw [color=gray,opacity = .8, dashed, thick](6.8,-1) rectangle (9.2,-11);
  \draw node at (9.2,-11) [above=-3mm, right=0mm] {$\ds  \sum_{i = 1}^k  \Hs(f_{[i]}, g_{[i]}) 
$%
};%
     
\end{tikzpicture}

}
  \caption{Nesting technique for modal decomposition: the nested H-score is computed with a nested architecture, where ``$\pplus$'' denotes the concatenation operation of two features. }
  \label{fig:nest}
\end{figure}

Hence, the nested H-score $\Hs(f, g; \cnest)$ is maximized if and only if we have \eqref{eq:md:fg}, which gives the top $k$ modes of $(X; Y)$ dependence. In practice, we can compute the nested H-score by using a nested architecture as shown in \figref{fig:nest}, where we have used the ``$\pplus$'' symbol to indicate the concatenation of two vectors, i.e., $\vec{u} \pplus \vec{v} \defeq
\begin{bsmallmatrix}
  \vec{u}\\
  \vec{v}
\end{bsmallmatrix}
$ for two column vectors $\vec{u}, \vec{v}$. By maximizing the nested H-score, we can use \eqref{eq:md:fg} to retrieve each $i$-th dependence mode from corresponding feature pair $(f_i, g_i)$, for $i \in [k]$.

Compared with the features learned in \secref{sec:dep:approx}, the nesting technique provides several new applications.
First, from \factref{fact:md:ortho}, the learned features $f$ and $g$ have orthogonal dimensions, i.e., different dimensions are uncorrelated. In addition, from \eqref{eq:md:fg}, we can compute the energy contained in each $i$-th dependence mode, via $\norm{\mdn{i}(\lpmi_{X; Y})}^2 = \norm{f_i \otimes g_i}^2 = \E{f_i^2(X)}\cdot \E{g_i^2(Y)}$, for $i \in [k]$. This provides a spectrum of $(X; Y)$ dependence and characterizes the usefulness or contribution of each dimension. 
Similarly, we can retrieve top $k$ maximal correlation functions $f_i^*, g_i^*$ and coefficients $\sigma_i$, by using the relations [cf. \eqref{eq:dcmp:xy} and \corolref{cor:hgr}]
\begin{subequations}
  \label{eq:f:g:simga:i}
  \begin{gather}
    f_i^* = \frac{f_i}{\sqrt{\E{f_i^2(X)}}},\quad
    g_i^* = \frac{g_i}{\sqrt{\E{g_i^2(Y)}}},\quad%
    \sigma_i = \sqrt{\E{f_i^2(X)}\cdot \E{g_i^2(Y)}},\qquad i \in [k].
  \end{gather}%
\end{subequations}

\paragraph{Nested Optimality}
From \eqref{eq:md:fg}, for any $d \leq k$, we can represent the optimal rank-$d$ approximation of $\lpmi_{X; Y}$ as $\mdnl{d}(\lpmi_{X; Y}) = f_{[d]} \otimes g_{[d]}$, which corresponds to top $d$-dimensions of learned features. We refer to this property as nested optimality: the learned features give a collection of optimal solutions for different dimensions, with a nested structure. 
This nested optimality provides a convenient and principled feature selection \emph{on the fly}: we can obtain the optimal selection of $d$ feature pairs by simply taking the top $d$ dimensions of learned features.
It is worth noting that, this optimal \emph{selection} indeed gives the optimal $d$-dimensional feature \emph{extraction}, corresponding to the most significant $d$ modes of the $(X; Y)$ dependence. This equivalence is hardly guaranteed in typical feature selection approaches. 
In practice, we can choose the dimension $d$ based on the dependence spectrum, such that selected features capture sufficient amount of dependence information, and then take $f_{[d]}, g_{[d]}$ for further processing.

We can readily extend the discussion to constrained modal decomposition problems. Let $\sspc{\cX}$ and $\sspc{\cY}$ be subspaces of
 $\spc{\cX}$ and $\spc{\cY}$, respectively. Then, the nested H-score $\Hs(f, g; \{\ones{k};\, \sspc{\cX} ; \, \sspc{\cY}\})$ defined for $k$-dimensional features  $f \in \sspcn{\cX}{k}$, $g \in \sspcn{\cY}{k}$, is maximized if and only if
 \begin{align}
   f_i \otimes g_i = \mdn{i}(\lpmi_{X; Y}|\sspc{\cX}, \sspc{\cY}), \quad \text{for all $i \in [k]$}.
   \label{eq:md:fg:sspc}
 \end{align}
 From \eqref{eq:md:fg:sspc}, we can establish a similar nested optimality in the constrained case. In particular, 
when $\sspc{\cX}$ and $\sspc{\cY}$ correspond to the collection of features that can be expressed by neural feature extractors, 
the result also characterizes the effects of restricted expressive power of neural networks (cf. \secref{sec:h:const}).
 Specifically, from \eqref{eq:md:fg:sspc}, when we use feature extractors with restricted expressive power, we can still guarantee the learned features have uncorrelated dimensions.%

\subsection{Nesting Technique for Projection}
\label{sec:nest:proj}%
With the nesting technique, we can also operate projections of statistical dependence in feature spaces. Such operations are the basis of multivariate dependence decomposition, which we will detail in the following sections.

To begin, let $\sspc{\cX}$ denote a subspace of $\spc{\cX}$. Then, from $\spc{\cX} = \sspc{\cX} \orthp (\spc{\cX} \orthm\sspc{\cX})$, we obtain an orthogonal decomposition of function space
\begin{align}
  \spc{\cX \times \cY} =   \spc{\cX} \otimes \spc{\cY} = 
(\sspc{\cX} \otimes \spc{\cY}) \orthp
  ( (\spc{\cX} \orthm\sspc{\cX}) \otimes \spc{\cY}).
  \label{eq:xy:dcmp}
\end{align}
Therefore, by projecting the statistical dependence $\lpmi_{X; Y}$ to these function spaces, we obtain its orthogonal decomposition [cf. \factref{fact:proj}]
 \begin{align}
   \lpmi_{X; Y}
   = \proj{\lpmi_{X; Y}}{\sspc{\cX} \otimes \spc{\cY}} + \proj{\lpmi_{X; Y}}{(\spc{\cX} \orthm \sspc{\cX}) \otimes \spc{\cY}}.
   \label{eq:sspc:dcmp}
 \end{align}
In particular, the first term $\proj{\lpmi_{X; Y}}{\sspc{\cX} \otimes \spc{\cY}}$ characterizes the dependence component aligned with the subspace $\sspc{\cX}$, and the second term represents the component orthogonal to $\sspc{\cX}$. For convenience, we denote
these two dependence components by $\pif{}(\lpmi_{X; Y})$ and $\pif{\perp}(\lpmi_{X; Y})$, respectively, and demonstrate the geometry of the decomposition in \figref{fig:proj}.

\def\olabel{$\lpmi_{X; Y}$}
\def\hlabel{$\pif{}(\lpmi_{X; Y})$} %
\def\vlabel{$\pif{\perp}(\lpmi_{X; Y})$}
\def\planelabel{$\sspc{\cX} \otimes \spc{\cY}$}
\def\olabelxshift{-0.3em}
\def\hlabelxshift{0.3em}%
\begin{figure}[!ht]
  \centering
  \resizebox{.45\textwidth}{!}{%
\tdplotsetmaincoords{75}{20} %
\begin{tikzpicture}[tdplot_main_coords,  >=latex', scale = .5]
  \coordinate (O) at (0,0,0);
  \def\x{5}
  \def\tscale{.7}
  \filldraw[
        draw=none,%
        fill=colorp,%
        fill opacity = .6,%
        ](-1, -1.5, 0)
        -- (\x, -1.5, 0)
        -- (\x, \x, 0)
        -- (-1, \x, 0)
        -- cycle;
        \tdplotsetcoord{Q}{1.1 * \x}{50}{40};
        \draw[->, draw = coloro, thick] (O) -- node[above, xshift = \olabelxshift, scale = \tscale] {\olabel} (Q);
        \draw[->, colorv, thick] (Qxy) -- node[right, scale = \tscale] {\vlabel}(Q);
        \draw[->, colorh, thick] (O) -- node[below, scale = \tscale, xshift = \hlabelxshift] {\hlabel}(Qxy); %

        \node [below, scale = \tscale] at (\x, -1, 0) {\planelabel};
        \draw  node [scale = 1] at (O) {.};
        \RightAngle{(Q)}{(Qxy)}{(O)}; 
\end{tikzpicture}

}
  \caption{Orthogonal decomposition of the CDK function $\lpmi_{X; Y}$: $\pif{}(\lpmi_{X; Y})$ denotes the projection onto the plane $\sspc{\cX} \otimes \spc{\cY}$, and $\pif{\perp}(\lpmi_{X; Y})$ denotes the residual, orthogonal to the plane.} %
  \label{fig:proj}
\end{figure}

In general, the information carried %
by decomposed dependence components depends on the choices of subspace $\sspc{\cX}$, which varies in different learning settings. In spite of such differences, we can learn the decomposition with a unified procedure, which we demonstrate as follows.

To begin, we consider the feature representations of the dependence components. For example, by applying the rank-$k$ approximation on the orthogonal component $\pif{\perp}(\lpmi_{X; Y})$, we obtain
\begin{align*}
  \mdnl{k}(\pif{\perp}(\lpmi_{X; Y}))  = \mdnl{k}(\proj{\lpmi_{X; Y}}{(\spc{\cX} \orthm \sspc{\cX}) \otimes \spc{\cY}})
  = \mdnl{k}(\lpmi_{X; Y}|\spc{\cX} \orthm \sspc{\cX}, \spc{\cY}),
\end{align*}
which can be represented as a pair of $k$-dimensional features.
To learn such feature representations, we introduce the two-level nesting configuration
\begin{align}
  \cpi \defeq  \{(\kb, k);~ (\sspc{\cX}, \spc{\cX}) ;~ \spc{\cY}\}
  \label{eq:cpi:def}
\end{align}
for some feature dimensions $\kb, k \geq 1$. The corresponding nested H-score is
\begin{align}
  \Hs\left(
  \begin{bmatrix}
    \fb\,\\
    f\,
  \end{bmatrix}
  ,
  \begin{bmatrix}
    \gb\\
    g
  \end{bmatrix}
  ; \cpi\right)
  =  \Hs(\fb, \gb) + \Hs\left(
  \begin{bmatrix}
    \fb\,\\
    f\,
  \end{bmatrix}
  ,
  \begin{bmatrix}
    \gb\\
    g
  \end{bmatrix}
  \right),
  \label{eq:h:nest:cpi}
\end{align}
 defined on the domain [cf. \eqref{eq:dom:c}]
\begin{align}
  \dom(\cpi) = \left\{ 
\left(\begin{bmatrix}
      \fb\,\\
      f\,
    \end{bmatrix}
  ,
      \begin{bmatrix}
    \gb\\
    g
  \end{bmatrix}
  \right)
  \colon \fb \in \sspcn{\cX}{\kb},   \gb \in \spcn{\cY}{\kb}, f \in \spcn{\cX}{k}, g \in \spcn{\cY}{k}
  \right\},
  \label{eq:dom:cpi}
\end{align}
where for convenience, we explicitly express the first-level features as $\fb, \gb$, both of dimension $\kb$. We can use a nested network architecture to compute the nested H-score \eqref{eq:h:nest:cpi}, as shown in \figref{fig:nest:two}.

\begin{figure}[!t]
  \centering
  \resizebox{.65\textwidth}{!}{%
\begin{tikzpicture}[auto, thick, node distance=2cm, >=latex'%
  ]
  \foreach \c in {1, 2}
  {
    \draw node [block] (H-\c) at (6.5, -3 * \c cm) {\Large $\Hs$};
  }

  \draw node [ext, %
  , left of = H-1, node distance=3.7cm, rotate = -90] (f-1)  {};
  \draw node [ext, right of = H-1, node distance=3.7cm, rotate = 90] (g-1)  {};

  \draw node [ext, left of = H-2, node distance=3.7cm, rotate = -90] (f-2)  {};
  \draw node [ext, right of = H-2, node distance=3.7cm, rotate = 90] (g-2)  {};

    \draw node at (f-1)  {\Large $\fb$}
     node at (g-1)  {\Large $\gb$};
       \draw node at (f-2)  {\Large $f$}
       node at (g-2)  {\Large $g$};

\draw node [con, left of = H-2, node distance = 2cm] (conl) {\con}
node [con, right of = H-2, node distance = 2cm] (conr) {\con};
\draw [->] (f-1) -- (H-1);
\draw [->] (g-1) -- (H-1);
  \draw [->] (f-2) -- (conl);
  \draw [->] (g-2) -- (conr);
  \draw [->] (conl) -- node[below] {%
  }(H-2);
  \draw [->] (conr) -- node[below] {%
  }(H-2);

  \draw [->] (f-1) -| (conl);
  \draw [->] (g-1) -| (conr);

\draw node at (conl |- f-1) [bullet] {\textbullet} ;
\draw node at (conr |- g-1) [bullet] {\textbullet} ;

\draw node [input] (x) at (1,-4.5) {}
     node [xshift = -0.5cm] at (x) {\Large $X$}
     node [xshift = -0.6mm] at (x) {\Large $\circ$};
        
     \draw node [input] (y) at (12,-4.5) {}
     node [xshift = 0.5cm] at (y){\Large $Y$} 
     node [xshift = 0.7mm] at (y) {\Large $\circ$} ;

  \foreach \s [count = \c] in {1, 3}
  {
    \draw [->] (x) -- + (.5, 0) node [bullet] {\textbullet} |- (f-\c);
    \draw [->] (y) -- + (-.5, 0) node [bullet] {\textbullet} |- (g-\c);
  }

\end{tikzpicture}

}
  \caption[]{Two-level Nested Network With Features  $\begin{bmatrix}
      \fb\,\\
      f\,
     \end{bmatrix}$ and $\begin{bmatrix}
      \gb\\
      g
     \end{bmatrix}$}
  \label{fig:nest:two} 
\end{figure}

To see the roles of the two H-score terms in \eqref{eq:h:nest:cpi}, note that if we maximize only the first term $\Hs(\fb, \gb)$ of the nested H-score over the domain \eqref{eq:dom:cpi}, we will obtain the solution to a constrained dependence approximation problem (cf. \secref{sec:h:const}): $\fb \otimes \gb = \mdnl{\kb}(\lpmi_{X; Y}|\sspc{\cX}, \spc{\cY})$. Specifically, if $\kb$ is sufficiently large, we would get $\fb \otimes \gb = \proj{\lpmi_{X; Y}}{\sspc{\cX} \otimes \spc{\cY}} = \pif{}(\lpmi_{X; Y})$, which gives the aligned component. With such $\fb$ and $\gb$, we can express the second H-score term as
\begin{align*}
  \Hs\left(
    \begin{bmatrix}
      \fb\,\\
      f\,
    \end{bmatrix}
    ,
    \begin{bmatrix}
      \gb\\
      g
    \end{bmatrix}\right) 
&=  
  \frac{1}{2}\cdot 
  \Bigl[\norm{\lpmi_{X; Y}}^2 - \norm{\underbrace{\lpmi_{X; Y} - \fb \otimes \gb}_{=\pif{\perp}(\lpmi_{X; Y})} - f \otimes g}^2\Bigr].%
\end{align*}

Therefore, if we maximize the second H-score over only $f$ and $g$, we would get the orthogonal dependence component: $f \otimes g = \mdnl{k}(\pif{\perp}(\lpmi_{X; Y}))$. This gives a two-phase training strategy for computing the decomposition \eqref{eq:sspc:dcmp}.

In contrast, the nested H-score  \eqref{eq:h:nest:cpi} provides a single training objective to obtain both dependence components simultaneously. We formalize the result as the following theorem, of which a proof is provided in \appref{app:thm:sspc:opt}.

\begin{theorem}
  \label{thm:sspc:opt}  
  Given
  $\kb \geq \rank(\proj{\lpmi_{X; Y}}{\sspc{\cX} \otimes \spc{\cY}})$,
  $\ds \Hs\left(
    \begin{bmatrix}
      \fb\,\\
      f\,
    \end{bmatrix}
    ,
    \begin{bmatrix}
      \gb\\
      g
    \end{bmatrix}
    ; \cpi
    \right)$
is maximized if and only if
\begin{subequations}
  \label{eq:sspc:opt}
  \begin{align}
    \fb \otimes \gb &= \proj{\lpmi_{X; Y}}{\sspc{\cX} \otimes \spc{\cY}},\\
    f \otimes g &= \mdnl{k}(\lpmi_{X; Y}|\spc{\cX} \orthm \sspc{\cX},  \spc{\cY}).
  \end{align}
\end{subequations}
\end{theorem}
We can further consider the modal decomposition of dependence components, to obtain features with nested optimality. To learn such features, it suffices to consider the refined configuration  $\cpir =   \left\{\ones{\kb + k};~ (\sspcn{\cX}{\kb}, \spcn{\cX}{k}) ;~ \spc{\cY}\right\}$, which we formalize as follows. A proof is provided in \appref{app:thm:sspc:opt:n}.

\begin{theorem}
  \label{thm:sspc:opt:n}  
  Given
  $\kb \geq \rank(\proj{\lpmi_{X; Y}}{\sspc{\cX} \otimes \spc{\cY}})$,
  $\ds \Hs\left(
    \begin{bmatrix}
      \fb\,\\
      f\,
    \end{bmatrix}
    ,
    \begin{bmatrix}
      \gb\\
      g
    \end{bmatrix}
    ; \cpir
  \right)$
  is maximized if and only if
  \begin{subequations}
    \label{eq:sspc:opt:n}
    \begin{gather}
      \fb_i \otimes \gb_i = \mdn{i}(\lpmi_{X; Y}|\sspc{\cX}, \spc{\cY})\quad \text{for all $ i \in [\kb]$,}\\
      f_i \otimes g_i = \mdn{i}(\lpmi_{X; Y}|\spc{\cX} \orthm \sspc{\cX}, \spc{\cY})\quad \text{for all $ i \in [k]$.}
    \end{gather}
  \end{subequations}
\end{theorem}

\subsection{Learning With Orthogonality Constraints}
\label{sec:nest:orth}
We conclude this section by discussing an application of the nesting technique, where the goal is to learn optimal features uncorrelated to some given features. {Such settings arise in many learning scenarios, where the given features correspond to some prior knowledge. For instance, when the given features are extracted from pre-trained models, we can use this formulation to avoid information overlapping and extract only the part not carried by these models. Another example is to extract privacy-preserving features, where we use the given features to indicate the sensitive information. 
}

In feature geometry, the uncorrelatedness conditions correspond to orthogonality constraints, and we can formalize the learning problem as follows. Given a $\kb$-dimensional feature $\phi \in \spcn{\cX}{\kb}$, our goal is to learn $k$-dimensional feature $f$ from $X$ for inferring $Y$, with the constraint that  $\spn\{f\} \perp \spn\{\phi\}$, i.e.,  $\E{f_i(X)\phi_j(X)} = \ip{f_i}{\phi_j} = 0$, for all $i \in [k], j \in [\kb]$. We therefore consider the constrained low-rank approximation problem
\begin{align}
  \minimize_{f \in \spcn{\cX}{k}, g  \in \spcn{\cY}{k}\colon  \spn\{f\} \perp \spn\{\phi\}} ~\bnorm{{\lpmi_{X; Y}} - f \otimes g}.
  \label{eq:opt:ortho} 
\end{align}

We can demonstrate that the solution to \eqref{eq:opt:ortho} corresponds to learning the decomposition \eqref{eq:sspc:dcmp}, with the choice $\sspc{\cX} = \spn\{\phi\}$. To see this, we rewrite \eqref{eq:opt:ortho} as
 \begin{align}
   \lpmi_{X; Y} = \pif{\phi}(\lpmi_{X; Y}) + \left(\lpmi_{X; Y} - \pif{\phi}(\lpmi_{X; Y})\right).
   \label{eq:pif:dcmp}
 \end{align}
where we have denoted the aligned component $\pif{\phi}(\lpmi_{X; Y}) \defeq \proj{\lpmi_{X; Y}}{\spn\{\phi\} \otimes \spc{\cY}}.$
In addition, note that the orthogonality constraint of \eqref{eq:opt:ortho} is $f \in \left(\spc{\cX} \orthm \spn\{\phi\}\right)^{k}, g  \in \spcn{\cY}{k}$. Therefore, it follows from  \propref{prop:md:constraint} that the solution to \eqref{eq:opt:ortho} is 
\begin{align}
f \otimes g  = \mdnl{k}(\lpmi_{X; Y}|\spc{\cX}\orthm \spn\{\phi\}, \spc{\cY})
  &= \mdnl{k}(\proj{\lpmi_{X; Y}}{(\spc{\cX}\orthm \spn\{\phi\}) \otimes \spc{\cY}})\notag\\
  &= \mdnl{k}\left(\lpmi_{X; Y} - \pif{\phi}(\lpmi_{X; Y}) \right),\label{eq:ortho:sol}
\end{align}
where to obtain the last equality we have used the orthogonal decomposition \eqref{eq:pif:dcmp}, as well as the fact that $(\spc{\cX}\orthm \spn\{\phi\}) \otimes \spc{\cY} = \spc{\cX \times \cY}\orthm (\spn\{\phi\}\otimes \spc{\cY})$.

\begin{remark}
From the decomposition \eqref{eq:pif:dcmp}, we can characterize the amount of dependence information captured by feature $\phi$,
as
the energy $\norm{\pif{\phi}(\lpmi_{X; Y})}^2$. %
This quantity (with a $1/2$ scaling factor)
is also referred to as
the \emph{single-sided H-score} \citep{xu2020maximal, xu2022information} of $\phi$, due to the connection: $\ds \max_{\gb \in \spcn{\cY}{\kb}} \Hs(\phi,\gb) = \frac12\norm{\pif{\phi}(\lpmi_{X; Y})}^2$.
\end{remark}

\begin{figure}[!ht]
  \centering
  \resizebox{.65\textwidth}{!}{%
\begin{tikzpicture}[auto, thick, node distance=2cm, >=latex'%
  ]
  \foreach \c in {1, 2}
  {
    \draw node [block] (H-\c) at (6.5, -3 * \c cm) {\Large $\Hs$};
  }

  \draw node [ext, frozenfill, left of = H-1, node distance=3.7cm, rotate = -90] (f-1)  {};
  \draw node [ext, right of = H-1, node distance=3.7cm, rotate = 90] (g-1)  {};

  \draw node [ext, left of = H-2, node distance=3.7cm, rotate = -90] (f-2)  {};
  \draw node [ext, right of = H-2, node distance=3.7cm, rotate = 90] (g-2)  {};

    \draw node at (f-1)  {\Large $\phi$}
     node at (g-1)  {\Large $\gb$};
       \draw node at (f-2)  {\Large $f$}
       node at (g-2)  {\Large $g$};

\draw node [con, left of = H-2, node distance = 2cm] (conl) {\con}
node [con, right of = H-2, node distance = 2cm] (conr) {\con};
\draw [->] (f-1) -- (H-1);
\draw [->] (g-1) -- (H-1);
  \draw [->] (f-2) -- (conl);
  \draw [->] (g-2) -- (conr);
  \draw [->] (conl) -- node[below] {%
  }(H-2);
  \draw [->] (conr) -- node[below] {%
  }(H-2);

  \draw [->] (f-1) -| (conl);
  \draw [->] (g-1) -| (conr);

\draw node at (conl |- f-1) [bullet] {\textbullet} ;
\draw node at (conr |- g-1) [bullet] {\textbullet} ;

\draw node [input] (x) at (1,-4.5) {}
     node [xshift = -0.5cm] at (x) {\Large $X$}
     node [xshift = -0.6mm] at (x) {\Large $\circ$};
        
     \draw node [input] (y) at (12,-4.5) {}
     node [xshift = 0.5cm] at (y){\Large $Y$} 
     node [xshift = 0.7mm] at (y) {\Large $\circ$} ;

  \foreach \s [count = \c] in {1, 3}
  {
    \draw [->] (x) -- + (.5, 0) node [bullet] {\textbullet} |- (f-\c);
    \draw [->] (y) -- + (-.5, 0) node [bullet] {\textbullet} |- (g-\c);
  }

\end{tikzpicture}

}
  \caption{Nesting technique for learning features orthogonal to feature $\phi$, where the $\phi$ block is frozen during learning. The feature $\phi$ can be given either in its analytical expressions, or as a pretrained neural network.
}
  \label{fig:fbar} 
\end{figure}
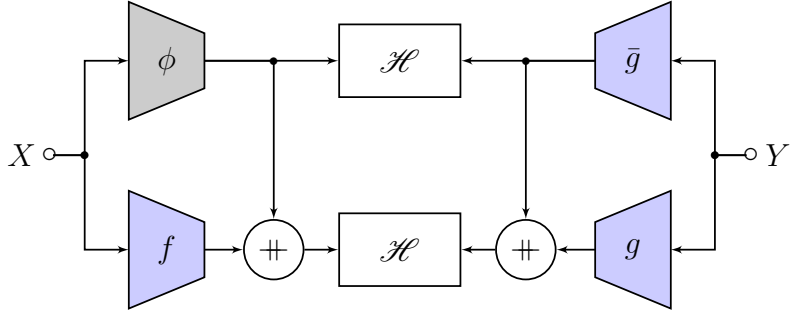

To learn the features \eqref{eq:ortho:sol}, we can apply the nesting technique and maximize the nested H-score configured by $\cpi$ with $\sspc{\cX} = \spn\{\phi\}$. Specifically, from \thmref{thm:sspc:opt}, $\fb = \phi$ is already in the optimal solution set. Therefore, we can fix $\fb$ to $\fb = \phi$, and optimize
\begin{align}
  \left. \Hs\left(
    \begin{bmatrix}
      \fb\,\\
      f\,
    \end{bmatrix}
    ,
    \begin{bmatrix}
      \gb\\
      g
    \end{bmatrix}
    ; \cpi
  \right)\right|_{\fb = \phi}
  =
  \Hs(\phi, \gb) + \Hs\left(
    \begin{bmatrix}
      \phi\,\\
      f\,
    \end{bmatrix}
    ,
    \begin{bmatrix}
      \gb\\
      g
    \end{bmatrix}
    \right)
  \label{eq:Hsf:def}
\end{align}
 over $\gb \in \spcn{\cY}{\kb}, f \in \spcn{\cX}{k}$, and $g \in \spcn{\cY}{k}$. We can compute the objective \eqref{eq:Hsf:def} by the nested network structure as shown in \figref{fig:fbar}.

It is also worth noting that from \propref{prop:dnn}, we can also interpret the solution to \eqref{eq:opt:ortho} as the features extracted by classification DNNs subject to the same orthogonality constraints. However, compared with the H-score optimization, putting such 
equality constraints in DNN training typically requires non-trivial implementation.

\section{Learning With Side Information}
\label{sec:side}
In this section, we study a multivariate learning problem involving  
external knowledge and demonstrate
learning algorithm design based on the nesting technique. Specifically, we consider the problem of learning features from $X$ to infer $Y$, and assume some external knowledge $S$ is available for the inference. We refer to $S$ as the side information, which corresponds to extra data sources for facilitating the inference. %
In particular, we consider the setting where we cannot obtain $(X, S)$ joint pair during the information processing, e.g., $X$ and $S$ are collected and processed by different agents in a distributed system. Otherwise, we can apply the bivariate dependence learning framework, by treating the $(X, S)$ pair as a new variable and directly learn features to predict $Y$.

\begin{figure}[!ht]
  \centering
   \resizebox{.45\textwidth}{!}{%
\begin{tikzpicture}[thick, node distance=2.8cm,auto,>=latex']
    \node [blk, pin={[pini]left:$X$}] (x) {{\sf Feature Extractor}};     %
    \node [blk] (c) [right of=x, pin={[pino]right:$\hat{Y}$}] {{ \sf Inference }};
    \node [] (s) [above of =c, node distance=1.2cm] {$S$};
    \node [coordinate] (end) [right of=c, node distance=1.2cm]{};
    \path[->] (x) edge node [pos=.5, above, scale = .8] {}  (c); %
    \path[->] (s) edge node [pos=.5, above, xshift = .2cm] {}  (c);
\end{tikzpicture}

}
   \caption{%
     Learning Setting With
     Side Information $S$  }
  \label{fig:side-info}
\end{figure}
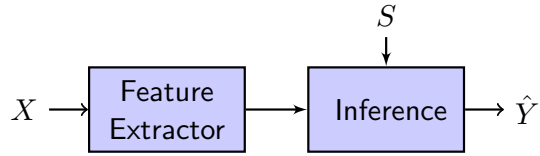

We depict this learning setting in \figref{fig:side-info}, where the inference is based on the both features extracted from $X$ and the side information $S$. Our goal is to design an efficient feature extractor which carries only the information not captured by $S$. In addition, we need to design a fusion mechanism for the inference module to combine such features with the side information $S$, and provide inference results conditioned on the side information.

Let $P_{X, S, Y}$ denote the joint distribution of $X$, $S$, $Y$. Throughout our development in this section, we consider the feature geometry on $\spc{\cX \times \cS \times \cY}$ with the metric distribution $R_{X, S, Y} \defeq P_{X}P_{S, Y}$.

\subsection{Dependence Decomposition and Feature Learning}
\label{sec:side:dcmp}
To begin, we represent the joint dependence in function space as the CDK function $\lpmi_{X; S, Y} \in \spc{\cX \times \cS \times \cY}$. Since the side information $S$ is provided for the inference, we focus on the dependence between $X$ and target $Y$ not captured by the side information. To this end, we separate the $(X; S)$ dependence from the joint dependence, by considering the orthogonal decomposition of function space [cf. \factref{fact:orthm} and \eqref{eq:spcs:def}]:   
\begin{align}
\spc{\cX \times \cS \times \cY} =
  \spc{\cX \times \cS} 
  \orthp
  \spcs{\cY}{(\cX \times \cS)}.\label{eq:side:dcmp}
\end{align}
This induces an orthogonal decomposition of the joint dependence
\begin{align}
  \lpmi_{X; S, Y} = \pim(\lpmi_{X; S, Y}) + \pic(\lpmi_{X; S, Y}),
  \label{eq:side:dcmp:i}
\end{align}   %
where we have defined $\pim(\gamma) \defeq \proj{\gamma}{\spc{\cX \times \cS}}$ and  $\pic(\gamma) \defeq \proj{\gamma}{\spcs{\cY}{(\cX \times \cS)}}$  for all $\gamma \in \spc{\cX \times \cS \times \cY}$. We characterize the decomposed components as follows, a proof of which is provided in \appref{app:prop:pim:pic}. 
\begin{proposition}
  \label{prop:pim:pic}
  We have $\pim(\lpmi_{X; S, Y})
  = \lpmi_{X; S} = \llrt{\Pm_{X, S, Y}}$, where  $\Pm_{X, S, Y} \defeq P_{X|S} P_S P_{Y|S}$. %
\end{proposition}

From \propref{prop:pim:pic}, we have $\pim(\lpmi_{X; S, Y})
  = \lpmi_{X; S} = \llrt{\Pm_{X, S, Y}}$, where  $\Pm_{X, S, Y} \defeq P_{X|S} P_S P_{Y|S}$. %
More generally, the space $\spc{\cX \times \cS}$ characterizes CDK functions associated with such Markov distributions, which we formalize as follows. 
A proof of which is provided in \appref{app:prop:markov}.

\begin{proposition}
  \label{prop:markov}
  Given $Q_{X, S, Y}$ with $Q_{X} = P_X, Q_{S, Y} = P_{S, Y}$, let $\lpmi_{X; S, Y}^{(Q)}$ denote the corresponding CDK function. Then, $\lpmi_{X; S, Y}^{(Q)} \in \spc{\cX \times \cS}$ if and only if $Q_{X, S, Y} = Q_{X|S}Q_SQ_{Y|S}$.
\end{proposition}

Hence, we refer to the dependence component $\pim\bigl(\lpmi_{X; S, Y}\bigr) = \lpmi_{X; S}$ as the {\sf M}arkov component. Then, we have $\pic\bigl(\lpmi_{X; S, Y}\bigr) = \lpmi_{X; S, Y} - \lpmi_{X; S}$, which characterizes the joint dependence not captured by $S$. We refer to it as the {\sf C}onditional dependence component, and also denote it by $\lpmi_{X; Y|S}$, i.e., 
\begin{align}
\lpmi_{X; Y|S}(x, s, y)  \defeq   \lpmi_{X; S, Y}(x, s, y) -  \lpmi_{X; S}(x, s) =  \left[\frac{P_{X, S, Y} - \Pm_{X, S, Y}}{R_{X, S, Y}}\right](x, s, y).\label{eq:lpmi:cond}  
\end{align}
Therefore, the conditional dependence component $\lpmi_{X; Y|S}$ vanishes, i.e.,  $\norm{\lpmi_{X; Y|S}} = 0$, if and only if $X$ and $Y$ are conditionally independent given $S$. In general, from the Pythagorean relation, we can write 
\begin{align}
   \bbnorm{\lpmi_{X; Y|S}}^2 = \bbnorm{\lpmi_{X; S, Y}}^2 - \bbnorm{\lpmi_{X; S} }^2,
  \label{eq:pyth}
\end{align}
analogous to the expression of the conditional mutual information $I(X; Y|S) = I(X; S, Y) - I(X; S)$. Indeed, 
we can establish an explicit connection in the local regime where $X$ and $(S, Y)$ are $\eps$-dependent, i.e., $\bnorm{\lpmi_{X; S, Y}} = O(\eps)$.
Then, from \lemref{lem:mutual} we obtain $\bnorm{\lpmi_{X; S, Y}}^2 = 2 \cdot I(X; S, Y) + o(\eps^2)$,  and similarly, $\bnorm{\lpmi_{X; S} }^2
  =  2\cdot I(X; S) + o(\eps^2)$. %
Therefore,   \eqref{eq:pyth} becomes
\begin{align*}
  \norm{\lpmi_{X; Y|S}}^2
  = \bnorm{\lpmi_{X; S, Y}}^2 -  \bnorm{\lpmi_{X; S} }^2
  = 2\cdot I(X; Y|S) + o(\eps^2).
\end{align*}%

We then discuss learning these two dependence components by applying the nesting technique. To begin, note that since %
\begin{align}
  \lpmi_{X; S} =  \pim(\lpmi_{X; S, Y}) &= \proj{\lpmi_{X; S, Y}}{\spc{\cX} \otimes \spc{\cS}},\\
  \lpmi_{X; Y|S} = \pic(\lpmi_{X; S, Y}) &= \proj{\lpmi_{X; S, Y}}{\spc{\cX} \otimes \spcs{\cY}{\cS}},
\end{align}
we recognize the decomposition $\lpmi_{X; S, Y} = \lpmi_{X; S} + \lpmi_{X; Y|S}$ as a special case of \eqref{eq:sspc:dcmp}. Therefore, similar to our discussions in \secref{sec:nest:proj},
 we consider the nesting configuration $\cside$ and its refinement $\csider$, where [cf. \eqref{eq:cpi:def}]
\begin{align}
  \cside \defeq  \left\{(\kb, k);\, \spc{\cX} ;\, (\spc{\cS}, \spc{\cS \times \cY})\right\}.
  \label{eq:cside:def}
\end{align}
The corresponding nested H-scores are defined on 
\begin{align}
  \dom(\cside) =  \dom(\csider) = 
\left\{\left(
  \begin{bmatrix}
    \fb\,\\
    f\,
  \end{bmatrix}
,
  \begin{bmatrix}
    \gb\\
    g
\end{bmatrix}
\right)\colon
 \fb \in \spcn{\cX}{\kb}, \gb \in \spcn{\cS}{\kb}, f \in \spcn{\cX}{k}, g \in\spcn{\cY \times \cS}{k}\right\}.
\end{align}
In particular, we can compute the nested H-score configured by $\cside$ from a nested network structure as shown in \figref{fig:nn:side}. Then, we can obtain both dependence components by optimizing the nested H-scores. %
Formally, we have the following corollary of  \thmref{thm:sspc:opt} and \thmref{thm:sspc:opt:n}.

\begin{figure}[!t]
  \centering
  \resizebox{.7\textwidth}{!}{%
\begin{tikzpicture}[auto, thick, node distance=2cm, >=latex'%
  ]
  \foreach \i/\j/\k [count = \y] in {1/1/0, 2/3/3}
  {
    \draw node [block] (H-\i) at (6.5, -3 * \y cm) {\Large $\Hs$};
  }

  \draw node [ext, left of = H-1, node distance=4cm, rotate = -90] (f-1)  {};
  \draw node [ext, right of = H-1, node distance=4cm, rotate = 90] (g-1)  {};
  \draw node [ext, left of = H-2, node distance=4cm, rotate = -90] (f-2)  {};
  \draw node [ext, right of = H-2, node distance=4cm, rotate = 90] (g-2)  {};

  \draw node at (f-1)  {\Large $\fb$}
  node at (g-1)  {\Large $\gb$}
  node at (f-2)  {\Large $f$}
  node at (g-2)  {\Large $g$};

        \draw
        node [con, left of = H-2, node distance = 2.3cm] (conl) {\con}
        node [con, right of = H-2, node distance = 2.3cm] (conr) {\con};

    \draw [->] (f-1) -- (H-1);
    \draw [->] (g-1) -- (H-1);

    \draw [->] (f-2) -- (conl);
    \draw [->] (g-2) -- (conr);
    \draw [->] (conl) -- node[below] { $\ds 
      \begin{bmatrix}
        \fb\\
        f
      \end{bmatrix}
      $
    }(H-2);
    \draw [->] (conr) -- node[below] {$\ds 
      \begin{bmatrix}
        \gb\\
        g
      \end{bmatrix}
      $
    }(H-2);

    \draw [->] (f-1) -| (conl);
    \draw [->] (g-1) -| (conr);
    \draw node at (conl |- f-1) [bullet] {\textbullet} ;
    \draw node at (conr |- g-1) [bullet] {\textbullet} ;

    \draw node [input] (x) at (0.5,-4.5) {}
    node [xshift = -0.5cm] at (x) {\Large $X$}
    node [xshift = -0.6mm] at (x) {\Large $\circ$};

    \draw node [input] (s) at (13.5,-3) {}
    node [xshift = 0.5cm] at (s){\Large $S$} 
    node [xshift = 0.7mm] at (s) {\Large $\circ$} ;
    
    \draw node [input] (y) at (13.5,-6) {} %
    node [xshift = 0.5cm] at (y){\Large $Y$} 
    node [xshift = 0.7mm] at (y) {\Large $\circ$} ;

    \draw node [con, left of = y, node distance=1cm, xshift = -0 cm] (con-ys)  {\con};

    \foreach \i/\j/\k  in {1/1/0, 2/3/3}
    {
      \draw [->] (x) -- + (.5, 0) node [bullet] {\textbullet} |- (f-\i);
    }
    \draw [->] (s) |- (g-1);
    \draw [->] (s) -| (con-ys);
    \draw [->] (y) -- (con-ys);
    \draw [->] (con-ys) -- node[below] {$\ds 
      \begin{bmatrix}
        S\\
        Y
      \end{bmatrix}
      $} (g-2);
    \draw node at (con-ys |- s) [bullet] {\textbullet} ;

\end{tikzpicture}

}
  \caption{Nesting Technique for Learning With Side Information $S$} %
  \label{fig:nn:side}  
\end{figure}
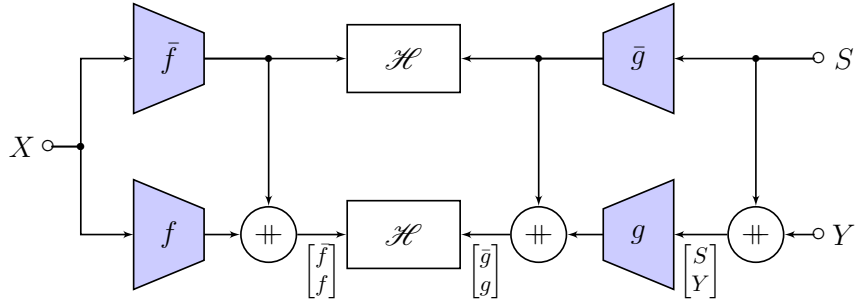

\begin{corollary}
Given $\kb \geq \rank(\lpmi_{X; S})$,   $\ds \Hs\left(
    \begin{bmatrix}
      \fb\,\\
      f\,
    \end{bmatrix}
    ,
    \begin{bmatrix}
      \gb\\
      g
    \end{bmatrix}
    ; \cside
    \right)$
is maximized if and only if
  \begin{subequations}
    \label{eq:opt:Hside}
    \begin{gather}
      \fb \otimes \gb = \lpmi_{X; S},
      \label{eq:opt:Hside:1}\\
      f \otimes g =  \mdnl{k}(\lpmi_{X; Y|S}).\label{eq:opt:Hside:2}
    \end{gather}
  \end{subequations}
In addition,    $\ds \Hs\left(
    \begin{bmatrix}
      \fb\,\\
      f\,
    \end{bmatrix}
    ,
    \begin{bmatrix}
      \gb\\
      g
    \end{bmatrix}
    ; \csider
    \right)$
is maximized if and only if %
\begin{subequations}
  \label{eq:opt:Hside:n}
  \begin{gather}
    \fb_i \otimes \gb_i = \mdn{i}(\lpmi_{X; S}), \quad i \in [\kb].\\
    f_i \otimes g_i = \mdn{i}(\lpmi_{X; Y|S}), \quad i \in [k]. %
  \end{gather}
\end{subequations}  
\end{corollary}

\subsection{Feature Assembling and Inference Models}%

We then assemble the features for inference tasks, particularly the inference conditioned on $S$. We first consider the case where we have learned both dependence components $\lpmi_{X; S}$ and $\lpmi_{X; Y|S}$, for which we have the following characterization (cf. \propref{prop:pred:est}). A proof is provided in 
\appref{app:prop:pred:est:side}.

\begin{proposition}
\label{prop:pred:est:side}
Suppose features $\fb \in \spcn{\cX}{\kb}, \gb \in \spcn{\cS}{\kb}$ and %
$ f \in \spcn{\cX}{k}, g \in \spcn{\cS \times \cY }{k}$ satisfy $\fb \otimes \gb = \lpmi_{X; S}, f \otimes g = \lpmi_{X; Y|S}$. Then, we have $\norm{\lpmi_{X; S}}^2 = \tr(\La_{\fb} \La_{\gb}), \norm{\lpmi_{X; Y|S}}^2 = \tr(\La_f  \La_g)$, and
\begin{align}
  P_{Y|X, S}(y|x, s) = 
P_{Y|S}(y|s) \cdot \left(1 + \frac{f^\T(x)  g(s, y)}{1 + \fb^{\,\T}(x)  \gb(s)} \right).
  \label{eq:post:est}
\end{align}
In addition, for any function $\psi \in \spcn{\cY}{d}$,
\begin{align}
  \E{\psi(Y)|X = x, S = s} = \E{\psi(Y)|S = s} +  \frac{1}{1 + \fb^{\,\T}(x)  \gb(s)}\cdot \La^{(s)}_{\psi, g} f(x),
  \label{eq:side:estimation}
\end{align}
where we have defined $\La^{(s)}_{\psi, g} \defeq \E{\psi(Y)g^\T(s, Y)|S = s}$ for each $s \in  \cS$.
\end{proposition}

Therefore, we can compute the strength of both the Markov component $\lpmi_{X; S}$ and the conditional component $\lpmi_{X: Y|S}$ from the features. Similarly, we can further compute the spectrum of the dependence components, by learning the modal decomposition according to \eqref{eq:opt:Hside:n}.

From \propref{prop:pred:est:side}, we can obtain inference models conditioned on the side information $S$. In particular, for classification task, we can use \eqref{eq:post:est} to compute the posterior probability, with the resulting MAP estimation conditioned on $S = s$  [cf. \eqref{eq:map}]:
\begin{align}
   \yh_{\MAP}(x; s) = \argmax_{y \in \cY}P_{Y|X, S}(y|x, s) = 
  \argmax_{y \in \cY}P_{Y|S}(y|s) \cdot \left(1 + \frac{f^\T(x)  g(s, y)}{1 + \fb^{\,\T}(x)  \gb(s)} \right).
   \label{eq:side:map}
 \end{align}
Specifically, $P_{Y|S}$ can be obtained by a separate discriminative model that predicts $Y$ from the side information $S$. In addition, when $Y$ is continuous, we can obtain the MMSE estimator of $\psi(Y)$ conditioned on $S = s$
from \eqref{eq:side:estimation}, where we can learn $\E{\psi(Y)|S = s}$ and
 $\La^{(s)}_{\psi, g} = \E{\psi(Y)g^\T(s, Y)|S = s}$ separately from $(S, Y)$ pairs. As we construct both models  by assembling learned features, the model outputs depend on input data $X$ only through the features $\fb$ and $f$ of $X$, as desired.

Moreover, we can conduct a conditional independence test %
without learning the complete conditional dependence $\lpmi_{X; Y|S}$. In particular, suppose we have learned features $ f \in \spcn{\cX}{k}, g \in \spcn{\cS \times \cY }{k}$ with $f \otimes g = \mdnl{k}(\lpmi_{X; Y|S})$ for some $k \geq 1$. Then we obtain
$\tr(\La_f \La_g) = \norm{\mdnl{k}(\lpmi_{X; Y|S})}^2 \geq 0$, where the equality holds
if and only if $\lpmi_{X; Y|S} = 0$, i.e., $X$ and $Y$ are conditionally independent given $S$.

\subsection{Theoretical Properties and Interpretations} 
 We conclude this section by demonstrating theoretical properties of the learned features. In particular, we focus on the conditional dependence component $\lpmi_{X; Y|S}$ and associated features, as the Markov component $\lpmi_{X; S}$ shares the same properties as discussed in 
the bivariate case.

To begin, let $K \defeq \rank(\lpmi_{X; Y|S})$, and let the modal decomposition of $\lpmi_{X; Y|S}$ be %
\begin{align}
  \mdn{i}(\lpmi_{X; Y|S}) = \sigma_i \cdot (f_i^* \otimes g_i^*), i \in [K],
  \label{eq:dcmp:pic}
\end{align}
where we have represented each mode in the standard form.
Then, we can interpret the $\sigma_i, f_i^*, g_i^*$ as the solution to a constrained maximal correlation problem. To see this, note that from
 $ \lpmi_{X; Y|S} =  \proj{\lpmi_{X; S, Y}}{ \spcs{\cY}{\cX, \cS}} = \proj{\lpmi_{X; S, Y}}{\spc{\cX} \otimes \spcs{\cY}{\cS}}$,
we can obtain
  $\sigma_i (f_i^* \otimes g_i^*) = \mdn{i}(\lpmi_{X; Y|S}) = \mdn{i}(\lpmi_{X; S, Y}|\spc{\cX}, \spcs{\cY}{\cS})$.
  Therefore, $f_i^*, g_i^*$ are the constrained maximal correlation function of $X$ and $(S, Y)$ as defined in \propref{prop:max:corr:const}, with the subspaces $\sspc{\cX} = \spc{\cX}, \sspc{\cS  \times \cY} = \spcs{\cY}{\cS}$.

\subsubsection{Local Posterior Distribution and Conditional Dependence}

In a local analysis regime, we can simplify the posterior distribution $P_{Y|X, S}$ as follows. %
A proof is provided in \appref{app:prop:posterior}. %

  \begin{proposition}
  \label{prop:posterior}
    If $X$ and $(S, Y)$ are $\eps$-dependent, we have %
  \begin{align}
    P_{Y|X, S}(y|x, s)%
      =  P_{Y|S}(y|s) \left(1 + \sum_{i = 1}^{K}\sigma_i f^*_i(x)g^*_i(s, y)\right) + o(\eps),
  \label{eq:posterior}
  \end{align}
  where $\sigma_i, f_i^*, g_i^*$ are as defined in \eqref{eq:dcmp:pic}.
\end{proposition}
 From \eqref{eq:posterior}, the dominant term of 
 $P_{Y|X, S}(y|x, s)$ depends on $x$ only through $f_i^*(x)$, $i = 1, \dots, K$. Therefore, the feature $f_{[K]}^*(X) = (f_1^*(X), \dots, f_K^*(X))^\T$ captures the conditional dependence between $X$ and $Y$ given $S$ except
possibly higher-order terms of $\epsilon$.

\subsubsection{Relationship to Multitask Classification DNNs }
We can also establish a connection between the side information problem and deep neural networks for multitask learning. Specifically, we consider a multitask classification task where $X$ and $Y$ denote the input data and target label to predict, respectively, and $S$ denotes the index for tasks. When conditioned on different values of $S$, the dependence between data and label are generally different. %
We then demonstrate that  a multitask DNN also learns
 the optimal approximation of
 the conditional dependence component $\lpmi_{X; Y|S}$.

\begin{figure}[!ht]
  \centering  \resizebox{.6\textwidth}{!}{%
\def\dy{3.5}
\def\dx{1.5}
\begin{tikzpicture}[auto, thick, node distance=2cm, >=stealth%
  ]
  \foreach \i/\loc in {1/0, 2/3.5, missing/5.5, 3/7.5}
  {
    \draw node [block/.try, \i/.try] (sl-\i) at (8, -\loc cm) {};
}
  \foreach \i/\sub in {1/1, 2/2, 3/{|\cS|}}
  {
    \draw node at (sl-\i) [] {\Large \sf softmax}
    node [extl, left of = sl-\i, node distance=3cm, text opacity = 0] (G-\i) {\Large $G_{1}^\T$} %
    node at (G-\i) {\Large $G_{\sub}^\T$}
    node [bias, below of = G-\i, node distance=1.4cm, text opacity = 0] (b-\i) {\Large $\vec{b}_{1}$} %
    node at (b-\i) {\Large $\vec{b}_{\sub}$}
;
  }

  \draw node [ext, left of = G-2, node distance = 3.5cm,  rotate = -90] (f) {};%

  \draw node [input, left of = f, node distance=1.5cm] (x)  {}
  node [xshift = -0.5cm] at (x) {\Large $X$}
  node [xshift = -0.7mm] at (x) {\Large \textopenbullet}
  node at (f) {\Large$f$};

  \draw [->] (x) -- (f);

  \foreach  \s/\sub in {1/1, 2/2, 3/{|\cS|}}
  {
    \draw [-] (G-\s) -- (b-\s);
    \draw[->] (f) -- ++ (\dx, 0) |- (G-\s);
    \draw[->] (G-\s) -- (sl-\s);
    \draw[->] (sl-\s) -- ++ (1.5, 0) node [right] {\Large $\Pt^{(f, g, b)}_{Y|X, S=\sub}$};
  }

  \draw node at ($(f) + (\dx, 0)$ ) {\textbullet} ;

\end{tikzpicture}

}
  \caption{A multihead network for extracting  feature $f$ shared among different tasks in $\cS = \{1, \dots, |\cS|\}$. Each task $s \in \cS$ corresponds to a separate classification head with weight matrix $G_s^\T$ and bias vector $\vec{b}_s$ for generating the associated posterior $\Pt_{Y|X, S = s}^{(f, g, b)}$.
}
  \label{fig:dnn:mh}
\end{figure}
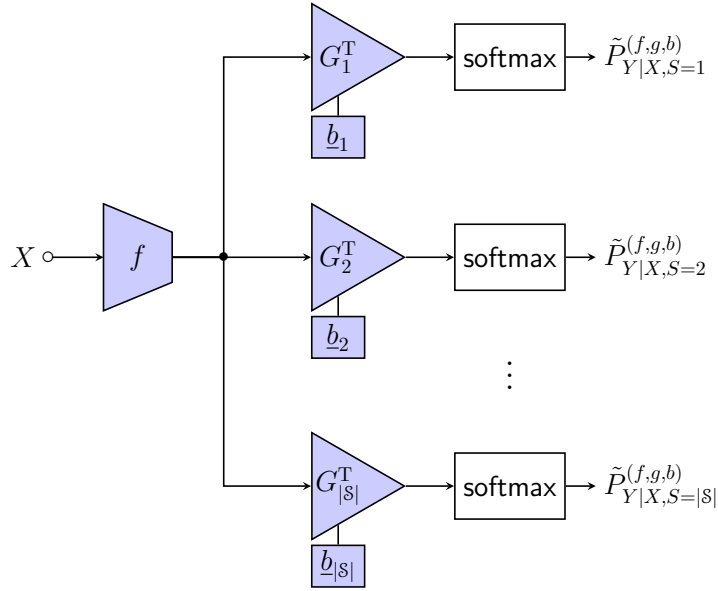

We consider a classical multitask classification DNN design \citep{caruana1993multitask, ruder2017overview}, as shown in \figref{fig:dnn:mh}.
In this figure, feature $f \in \spcn{\cX}{k}$ of $X$ is shared among all tasks. For each task $s \in \cS$, the corresponding classification head
 with weight matrix $G^\T_s \in \mathbb{R}^{|\cY| \times k}$ and bias $\vec{b}_s \in \mathbb{R}^{|\cY|}$ are applied to compute the corresponding posterior probability
\begin{align}
  \Pt^{(f, g, b)}_{Y|X, S}(y|x, s)
  \defeq
  \frac{\exp\left(f(x)\cdot g(s, y) + b(s, y)\right)}{\sum_{y'\in \cY} \exp\left(f(x)\cdot g(s, y') + b(s, y')\right)},
  \label{eq:Pt:s}
\end{align}
where $g \in \spc{\cS\times \cY}$ and $b \in \spc{\cY}$ are related to $G_s$ and $\vec{b}_s$ via [cf. \eqref{eq:def:G:b}] 
\begin{align}
  G_s(i, y) = g_i(s, y) \text{ for all $i \in [k], y \in \cY$}, \qquad \vec{b}_s = [b(s, 1), \dots, b(s, |\cY|)]^\T. 
  \label{eq:def:G:b:side}
\end{align}

Given data samples $\{(x_i, s_i, y_i)\}_{i = 1}^n$ with the empirical distribution %
$P_{X, S, Y}$, we write the corresponding likelihood function as
\begin{align}
 \llog_S(f, g, b) \defeq \frac1n \sum_{i = 1}^n \log \Pt^{(f, g, b)}_{Y|X, S}(y_i|x_i, s_i)
  = \Ed{(\Xh, \Yh, \Sh) \sim P_{X, Y, S}}{\log \Pt^{(f, g, b)}_{Y|X, S}(\Yh|\Xh, \Sh)}.
  \label{eq:likelihood:side:def}
\end{align}
Note that we can relate the posterior probability $\Pt^{(f, g, b)}_{Y|X, S}$ to
the posterior
 $\Pt^{(f, g, b)}_{Y|X}$ in the ordinary classification DNN, as defined in \eqref{eq:Pt}. To see this, note that for all $s \in \cS$, we have  $\Pt^{(f, g, b)}_{Y|X, S=s} = \Pt^{(f, g^{(s)}, b^{(s)})}_{Y|X}$, where 
we have defined $g^{(s)} \in \spcn{\cY}{k}$  and $b^{(s)} \in \spc{\cY}$ for each $s \in \cS$, as $g^{(s)}(y) \defeq g(s, y), b^{(s)}(y) \defeq b(s, y)$. Then, we rewrite \eqref{eq:likelihood:side:def} as $\llogs(f, g, b) = \sum_{s \in \cS} P_S(s)\llogsv{s}(f, g^{(s)}, b^{(s)})$, where $ \llogsv{s}(f, g, b) \defeq \Ed{(\Xh, \Yh) \sim P_{X, Y| S = s}}{\log \Pt^{(f, g, b)}_{Y|X}(\Yh|\Xh)}$
is the expected likelihood value conditioned on $S = s$. We further assume the all bias terms are trained to their optimal values with respect to $f$ and $g$. This gives the likelihood 
\begin{align}
\llogs(f, g) \defeq \sum_{s \in \cS} P_S(s)\llogsv{s}(f, g^{(s)}) = \max_{b \in \spc{\cS \times \cY}} \llogs(f, g, b), 
\end{align}
where we have denoted $\llogsv{s}(f, g) \defeq \max_{b \in \spc{\cY}}\llogsv{s}(f, g, b)$ for each $s \in \cS$. Then, from \proptyref{propty:dnn:center} we can verify that $\llogs(f, g)$ depends only on some centered features, formalized as follows. %
\begin{property}
\label{propty:dnn:center:side}
We have
  $\llogs(f, g) = \llogs(\ft, \gt)$,
where we have defined $\ft \defeq \proj{f}{\spct{\cX}}$, and $\gt \defeq \proj{g}{\spcs{\cY}{\cS}}$, i.e., $\ft(x) = f(x) - \E{f(X)}$ and
 $\gt(s, y) = g(s, y) - \E{g(s, Y)|S = s}$. %
\end{property}
Therefore, we can focus on centered features $f \in \spctn{\cX}{k}$ and $g \in \spcsn{\cY}{\cS}{k}$, i.e.,   $\E{f(X)} = 0$ and $\E{g(s, Y)|S = s} = 0$
for all $s \in \cS$.   %
We also restrict to features $f, g$ that perform better than the trivial choice of zero features, by assuming that %
\begin{align}
  \llogsv{s}(f, g^{(s)})
  \geq \llogsv{s}(0, g^{(s)}) = \llogsv{s}(0, 0) =
 -H(Y|S =s), \quad \text{for all $s \in \cS$}.
  \label{eq:better:than:trivial:s}
\end{align}

Then, we have the following characterization, which extends \propref{prop:dnn} to the multitask setting. A proof is provided in \appref{app:thm:dnn:mulithead}. \begin{theorem}
    \label{thm:dnn:mulithead}
Suppose $X$ and $(S, Y)$ are $\eps$-dependent. For $f \in \spctn{\cX}{k}$ and $g \in \spcsn{\cY}{\cS}{k}$ with \eqref{eq:better:than:trivial:s}, 
\begin{align}
  \llogs(f, g) = \llogs(0, 0) + \frac{1}{2} \cdot \left(\norm{\lpmi_{X; Y|S}}^2 - \bnorm{\lpmi_{X; Y|S} - f \otimes g}^2
 \right) + o(\eps^2),
\end{align}
which is maximized %
if and only if $f \otimes g = \mdnl{k}(\lpmi_{X; Y|S}) + o(\eps)$.
\end{theorem}

Therefore, the multitask classification network essentially learns features approximating the conditional dependence $\lpmi_{X; Y|S}$. Different from the nested H-score implementation, this multitask network implements the conditioning by directly applying a separate classification head for each task $S = s$. As a consequence, this design requires $|\cS|$ many different heads, and is not applicable when the side information $S$ is continuous or has complicated structures.

\section{Multimodal Learning With Missing Modalities}
\label{sec:mm}%
In this section, we demonstrate another multivariate learning application, where we need to conduct inferences based on different data sources. In particular, we focus on the setting where the goal is to infer $Y$ from two different data sources, denoted by $X_1$ and $X_2$. %

We refer to such problems as the multimodal learning\footnote{The literature typically uses ``multimodal'' to refer to the different forms (modalities) of data sources, e.g., video, audio, and text. However, such distinction is insignificant in our treatment, when we model each data source as a random variable.} problems, and are particularly interested in the  cases where we have missing modalities: either $X_1$ or $X_2$ can be missing during the inference. Our goal is to design a learning system to solve all the three problems: (i) inferring $Y$ based on $X_1$, (ii) inferring $Y$ based on $X_2$, and (iii) inferring $Y$ based on $(X_1, X_2)$. %

Throughout our discussions in this section, we use $P_{X_1, X_2, Y}$ to denote the joint distribution of $X_1 \in \cX_1, X_2 \in \cX_2, Y \in \cY$. For convenience, we also denote $X \defeq (X_1, X_2) \in \cX \defeq \cX_1 \times \cX_2$. We consider the feature geometry on $\spc{\cX \times \cY} = \spc{\cX_1 \times \cX_2 \times \cY}$, with the metric distribution $R_{X_1, X_2, Y} = P_{X_1,X_2}P_Y$, or equivalently, $R_{X, Y} = P_{X}P_Y$.

\subsection{Dependence Decomposition}
To begin, we %
decompose the joint dependence $\lpmi_{X_1, X_2; Y} \in  \spc{\cX_1\times\cX_2 \times \cY}$ as
\begin{align}
 \lpmi_{X_1, X_2; Y}
  = \pib(\lpmi_{X_1, X_2; Y}) + \pii(\lpmi_{X_1, X_2; Y}),
  \label{eq:pib:pii}
\end{align}
where we have defined
  $\pib(\gamma) \defeq \proj{\gamma}{\spc{\cX_1 \times \cY} +  \spc{\cX_2 \times \cY}}, \pii(\gamma) \defeq \gamma - \pib(\gamma)$. We refer to $\pib(\lpmi_{X_1, X_2; Y})$ 
as the {\sf B}ivariate dependence component, and refer to $\pii(\lpmi_{X_1, X_2; Y})$ as the
 {\sf I}nteraction component.

The bivariate dependence component $\pib(\lpmi_{X_1, X_2; Y})$ is uniquely determined by all pairwise dependencies among $X_1, X_2, Y$. Formally, let $\cqb$ denote the collection of distributions with the same pairwise marginal distributions as $P_{X_1, X_2, Y}$, i.e.,
\begin{align}
  \cqb \defeq \{Q_{X_1, X_2, Y} \in \cP^{\cX_1 \times \cX_2 \times \cY}\colon Q_{X_1, X_2} = P_{X_1, X_2}, Q_{X_1, Y} = P_{X_1, Y}, Q_{X_2, Y} = P_{X_2, Y}\}.
  \label{eq:def:q2}
\end{align}
Then we have the following result. A proof is provided in \appref{app:prop:pib}. 
\begin{proposition}
  \label{prop:pib}
  For all $Q_{X_1, X_2, Y}\in \cqb$, %
  we have
    $\pib(\lpmi^{(Q)}_{X_1, X_2; Y}) = \pib(\lpmi_{X_1, X_2; Y})$, where $\lpmi^{(Q)}_{X_1, X_2; Y}$ denotes the CDK function associated with $Q_{X_1, X_2, Y}$. %
\end{proposition}

\begin{figure}[!ht]
  \centering
  \resizebox{.6\textwidth}{!}{%
\tdplotsetmaincoords{67}{25} %
\begin{tikzpicture}[tdplot_main_coords,  >=latex', scale = .5]
  \coordinate (O) at (0,0,0);
  \def\x{7}
  \def\tscale{.7}
  \filldraw[
        draw=none,%
        fill=colorp,
        fill opacity = .6,%
        ](-1, -1.5, 0)
        -- (\x, -1.5, 0)
        -- (\x, \x, 0)
        -- (-1, \x, 0)
        -- cycle;
        \coordinate (I) at (5,5,5.5);
        \coordinate (Ixy) at (5,5,0);
        \coordinate (B1) at (4,1,0);
        \coordinate (B2) at (1,4,0);
        \coordinate (M1) at (5.88, 1.47, 0); %
        \draw[->, thick, draw = coloro] (O) -- node[pos = .7, left, scale = \tscale] {$\lpmi_{X_1, X_2;Y}$} (I);
        \draw[->, thick, colori] (Ixy) -- node[right, scale = \tscale] {$\pii$}(I);
        \draw[->, thick, draw = colorb] (O) -- node[right, scale = \tscale, pos = 1] {$\textcolor{colorb}{\pib} = \textcolor{colorb1}{\pibn{1}} + \textcolor{colorb2}{\pibn{2}}$}(Ixy);
        \draw[->, thick, colorb1] (O) -- node[pos = .8, below, scale = \tscale] {$\pibn{1}$} (B1);
        \draw[->, colorb1, opacity = .7] (O) -- node[pos = .9, below, scale = \tscale] {$\pimn{1}$} (M1); %
        \draw[->, colori, opacity = .7] (M1) -- node[left, xshift = .2em, scale = \tscale] {$\picn{1}$} (I);
        \draw[->, thick, colorb2] (O) -- node[pos = 1, xshift = .1em, above, scale = \tscale] {$\pibn{2}$} (B2);
        \draw[draw opacity = .5, colorb1] (M1) -- ++ (1, .25, 0) node [pos = 1, right, scale = \tscale] {$\spc{\cX_1 \times \cY}$};
        \draw[draw opacity = .5, colorb2] (B2) -- ++ (.5, 2, 0) node [pos = 1, above, scale = \tscale] {$\spc{\cX_2 \times \cY}$};
        \draw[dashed] (B1) -- (Ixy);
        \draw[dashed] (B2) -- (Ixy);
        \draw[dashed] (M1) -- (Ixy);

        \node [ scale = \tscale] at (6.5, -1.5, 0) {$\spc{\cX_1 \times\cY} + \spc{\cX_2 \times\cY}$};

        \node [below, scale = \tscale] at (\x, -1, 0) {};
        \RightAngle{(I)}{(M1)}{(O)}; 
        \RightAngle{(Ixy)}{(M1)}{($(M1) +
          (1, .25, 0)$)}; 
\end{tikzpicture}

}
  \caption{Decompose the joint dependence $\lpmi_{X_1, X_2; Y}$ into bivariate dependence component $\pib$ and interaction dependence component $\pii$.
The plane denotes the sum of $\spc{\cX_1 \times\cY}$ and $\spc{\cX_2 \times\cY}$.  }
  \label{fig:mm:dcmp}
\end{figure}
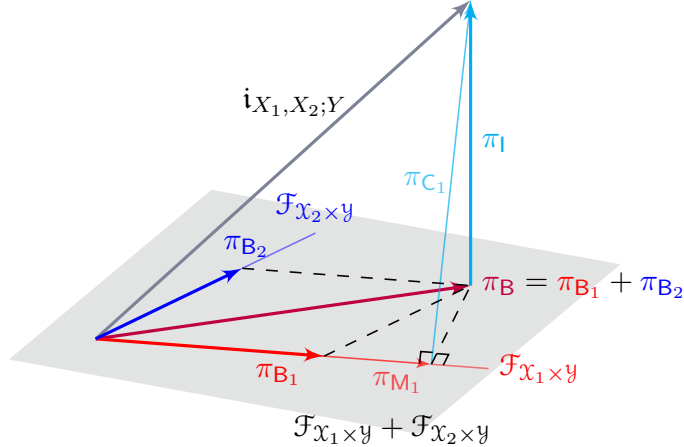

We show the relation between different dependence components in \figref{fig:mm:dcmp}, where we have further decomposed the bivariate dependence component $\pib(\lpmi_{X_1, X_2; Y})$ as
$\pib(\lpmi_{X_1, X_2; Y}) = \pibn{1}(\lpmi_{X_1, X_2; Y}) + \pibn{2}(\lpmi_{X_1, X_2; Y})$ for some $\pibn{i}(\lpmi_{X_1, X_2; Y}) \in \spc{\cX_i\times \cY}, i = 1, 2$. For comparison, we have also demonstrated $\pimn{1}(\lpmi_{X_1, X_2; Y}) = \lpmi_{X_1; Y}$ and $\picn{1}(\lpmi_{X_1, X_2; Y}) = \lpmi_{X_1, X_2; Y} - \lpmi_{X_1; Y}$,  obtained from the decomposition introduced in \secref{sec:side:dcmp}. Note that since the interaction component $\pii(\lpmi_{X_1; X_2; Y})$ does not capture any bivariate dependence, we can also obtain $\pimn{1}(\lpmi_{X_1, X_2; Y})$ directly from the bivariate component $\pib(\lpmi_{X_1, X_2; Y})$ via a projection: $\pimn{1}(\lpmi_{X_1, X_2; Y}) = \pimn{1}(\pib(\lpmi_{X_1, X_2; Y}))$.

\subsection{Feature Learning from Complete Data}

We consider learning the features representations for the two dependence components. Here, we assume the data are complete $(X_1, X_2, Y)$ triplets with the empirical distribution $P_{X_1, X_2, Y}$. We will discuss the learning with incomplete data later. 

Again, we apply the nesting technique to design the training objective. Note that since $X = (X_1, X_2)$, with $\sspc{\cX} = \spc{\cX_1} + \spc{\cX_2}$, we can express the two components as [cf. \eqref{eq:sspc:dcmp}]
\begin{align}
  \pib(\lpmi_{X; Y}) &= \proj{\lpmi_{X; Y}}{\sspc{\cX}\otimes \spc{\cY}},\\
  \pii(\lpmi_{X; Y}) &= \proj{\lpmi_{X; Y}}{(\spc{\cX} \orthm \sspc{\cX})\otimes \spc{\cY}}.
\end{align}
Therefore, we consider the nesting configuration $\cbi$ and its refinement $\cbir$, as [cf. \eqref{eq:cpi:def}]
\begin{align}
  \cbi \defeq \left\{(\kb, k);\, (\spc{\cX_1} + \spc{\cX_2}, \spc{\cX});\, \spc{\cY}\right\}.
  \label{eq:cbi:def}
\end{align}
The corresponding nested H-scores are defined on 
\begin{align}
  \dom(\cbi) = \dom(\cbir) = \left\{\left(
    \begin{bmatrix}
      \fb\,\\
      f\,
    \end{bmatrix}
    ,
    \begin{bmatrix}
      \gb\\
      g
    \end{bmatrix}\right)
  \colon  \fb \in \spcn{\cX_1}{\kb} + \spcn{\cX_2}{\kb}, \gb \in \spcn{\cY}{\kb}, f \in \spcn{\cX}{k}, g \in \spcn{\cY}{k}\right\}.
  \label{eq:dom:cbi}
\end{align}

\begin{figure}[!t]
  \centering
  \resizebox{.85\textwidth}{!}{%

\begin{tikzpicture}[auto, thick, node distance=2cm, >=stealth%
  ]
  \draw  node [input] (x1) at (0,0) {}
  node [xshift = -6mm] {\Large$X_1$}
  node [xshift = -.8mm]  {\Large \textopenbullet} 
  node [input, below of = x1, node distance = 6cm] (x2) {$X_2$}
  node at (x2) [xshift = -6mm] {\Large $X_2$}
  node at (x2) [xshift = -.8mm]{\Large \textopenbullet} 
  node [ext, right of = x1, node distance = 4cm, rotate = -90] (f_p1) {}
  node at (f_p1) {\Large $\fb^{(1)}$}
  node [ext, below of = f_p1, node distance = 3cm, rotate = -90] (f_p2) {}
  node at (f_p2) {\Large $\fb^{(2)}$}
  node [ext, below of = f_p2, node distance = 3cm, rotate = -90] (f_p12) {}
  node at (f_p12) {\Large $f$}
  node [con, left of = f_p12, node distance = 2.5cm] (con12) {\con}
  node [sum, right of=f_p1, xshift = -.5cm, yshift = -1.5cm] (sum0) {\suma}
  node [block, right of = sum0, node distance = 4cm] (H1) {\Large $\Hs$}
;

\draw
   node [block, name=H2] at (H1|-f_p12)  {\Large $\Hs$}
   node [con, left of = H2, node distance = 2.8cm] (conl) {\con}
   node [con, right of = H2, node distance = 2.8cm] (conr) {\con}
   node [ext, right of = H1, xshift = 2.5cm, rotate = 90] (g1) {}
   node [ext, right of = H2, xshift = 2.5cm, rotate = 90] (g2) {}
   node [input, name=y, right of = g1, yshift = -2.2cm] {}
   node at (g1) {\Large $\gb$}
   node at (g2) {\Large $g$}
   node at (y) [xshift = .8mm] {\Large \textopenbullet} 
   node at (y) [xshift = 6mm] {\Large $Y$};

        \draw[->](x1) -- node {}(f_p1);
        \draw[->](x2) -- node {}(con12);
        \draw[->](con12) -- node [below]{\large $X$}(f_p12);%
        \draw[->](x1) -| node {}(con12); %
        \draw[->](x2) -- + (.4, 0) node {\textbullet}  |- node {}(f_p2); %
        \draw[->](f_p1) -| node [near start] {} (sum0);
        \draw[->](f_p2) -| node [near start, below]{} (sum0);
        \draw[->](sum0) -- node {\Large $\fb$} (H1);
        \draw[->](sum0) -| node {} (conl);
        \draw[->](f_p12) -- node {} (conl);
        \draw[->](conl) -- node [below] {\large$\ds 
          \begin{bmatrix}
            \fb\,\\
            f
          \end{bmatrix}
$} (H2);
        \draw[->](conr) -- node {\large$\ds 
          \begin{bmatrix}
            \gb\\
            g
          \end{bmatrix}
$} (H2);
        \draw[->](g1) -- node [above] {} (H1);
        \draw[->](g1) -| node {} (conr);
        \draw[->](g2) -- node {} (conr);
        \draw[->] (y) -- +(-.5, 0) node {\textbullet} |- node {} (g1);
        \draw[->] (y) -- +(-.5, 0) node {\textbullet} |- node {} (g2);

\draw
        node at (con12|-x1) {\textbullet} 
        node at (conr|-g1) {\textbullet} 
        node at (conl|-sum0) {\textbullet} 
        ;
\end{tikzpicture}

}
  \caption{Nesting Technique for Learning Features from Multimodal Data}%
  \label{fig:nn:mm}
\end{figure}
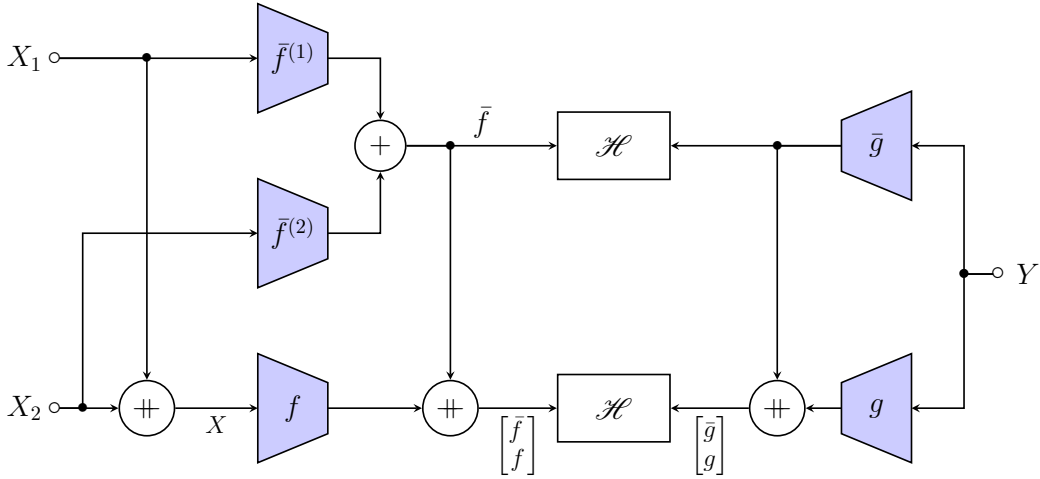

Specifically, we can compute the nested H-score configured by $\cbi$ using a nested network structure as shown in \figref{fig:nn:mm}. Then we can obtain both dependence components by maximizing the corresponding nested H-scores, formalized as follows (cf. \thmref{thm:sspc:opt} and \thmref{thm:sspc:opt:n}).

    \begin{corollary}
      \label{cor:hsbinest}
      Given $\kb \geq \rank(\pib(\lpmi_{X_1, X_2; Y}))$, the nested H-score 
$\ds \Hs\left(
    \begin{bmatrix}
      \fb\,\\
      f\,
    \end{bmatrix}
    ,
    \begin{bmatrix}
      \gb\\
      g
    \end{bmatrix}
    ; \cbi
    \right)$
is maximized if and only if
\begin{subequations}
  \begin{gather}
    \fb \otimes \gb = \pib(\lpmi_{X_1, X_2; Y}),\label{eq:pib:learn}
\\
    f \otimes g = \mdnl{k}\left(\pii(\lpmi_{X_1, X_2; Y})\right).
  \end{gather}
\end{subequations}
 In addition, $\ds \Hs\left(
    \begin{bmatrix}
      \fb\,\\
      f\,
    \end{bmatrix}
    ,
    \begin{bmatrix}
      \gb\\
      g
    \end{bmatrix}
    ; \cbir
    \right)$
 is maximized if and only if %
 \begin{subequations}
   \begin{gather}
     \fb_i \otimes \gb_i = \mdn{i}\left(\pib(\lpmi_{X_1, X_2; Y})\right), \quad i \in [\kb],
     \\
     f_i \otimes g_i = \mdn{i}\left(\pii(\lpmi_{X_1, X_2; Y})\right), \quad i \in [k].
   \end{gather}
 \end{subequations}
    \end{corollary}

\subsection{Feature Assembling and Inference Models}
\label{sec:mm:assemble}
We then illustrate how to assemble the learned features for the inference tasks and deal with incomplete data. For convenience, we define the conditional expectation operators $\opxn{i}, i = 1, 2$, 
such that for  $f \in \spcn{\cX_1\times \cX_2}{k}$ with $k \geq 1$, we have
\begin{align}
  [\opxn{i}(f)](x_i) \defeq \E{f(X_1, X_2)|X_i = x_i}, \qquad\text{for all }x_i \in \cX_i.
  \label{eq:tau:i:def}  
\end{align}
Note that we can also interpret $\opxn{i}$ as to the projection onto $\spc{\cX_i}$, i.e., $\opxn{i}(f)= \proj{f}{\spc{\cX_i}}$. 
Then, we have the following result. A proof is provided in \appref{app:prop:pred:est:mm}.

\begin{proposition}
  \label{prop:pred:est:mm}
  Suppose we have  $\fb \otimes \gb = \pib(\lpmi_{X_1, X_2; Y}), f \otimes g = \pii(\lpmi_{X_1, X_2; Y})$ for features $\fb = \fb^{(1)} + \fb^{(2)}$ with $\ds\fb^{(i)} \in \spcn{\cX_i}{\kb}, i = 1, 2$,  $\gb \in \spcn{\cY}{\kb}$, $f \in \spcn{\cX}{k}$, and $g \in \spcn{\cY}{k}$.
 Then, we have
 \begin{subequations}
   \label{eq:p:mm}
   \begin{gather}
     P_{Y|X_1, X_2}(y|x_1, x_2) = P_Y(y) \left[1 + \fb^{\,\T}(x_1, x_2)\gb(y) + f^\T(x_1, x_2)g(y)\right],\label{eq:p:mm:12}\\
     P_{Y|X_1}(y|x_1) = P_Y(y) \left[1 + \left(\fb^{(1)}(x_1) + [\opxn{1}(\fb^{(2)})](x_1)\right)^\T \gb(y)\right],\label{eq:p:mm:1}\\
     P_{Y|X_2}(y|x_2) = P_Y(y) \left[1 + \left(\fb^{(2)}(x_2) + [\opxn{2}(\fb^{(1)})](x_2)\right)^\T \gb(y)\right].\label{eq:p:mm:2}
   \end{gather}
 \end{subequations}
  In addition, for all $\psi \in \spcn{\cY}{d}$, we have
  \begin{subequations}
    \label{eq:est:mm}
  \begin{gather}
    \E{\psi(Y)|X_1 = x_1, X_2 = x_2} = \E{\psi(Y)} + \La_{\psi, \gb} \fb(x_1, x_2) + \La_{\psi, g} f(x_1, x_2),\\
    \E{\psi(Y)|X_1 = x_1} = \E{\psi(Y)} + \La_{\psi, \gb} \left(\fb^{(1)}(x_1) + [\opxn{1}(\fb^{(2)})](x_1)\right),\\
    \E{\psi(Y)|X_2 = x_2} = \E{\psi(Y)} + \La_{\psi, \gb} \left(\fb^{(2)}(x_2) + [\opxn{2}(\fb^{(1)})](x_2)\right).
  \end{gather}
\end{subequations}
\end{proposition}

From \propref{prop:pred:est:mm}, 
we can obtain inference models for all three different types of input data, 
by simply assembling the learned features in different ways. %
The resulting inference models also reveal the different roles of two dependence components. For example, the features associated with the interaction dependence component, i.e., $f$ and $g$, are used only when we have both $X_1$ and $X_2$ observations.

In practice, we can use \eqref{eq:p:mm} and \eqref{eq:est:mm} for classification and estimation tasks, respectively. To apply \eqref{eq:est:mm}, we can compute $\La_{\psi, \gb}$ and $\La_{\psi, g}$ from the corresponding empirical averages over the training dataset, and learn features $\opxn{1}(\fb^{(2)})$ and  $\opxn{2}(\fb^{(1)})$ from $(X_1, X_2)$ pairs. For example, we can
use \propref{prop:pred:est} to implement the conditional expectation operators $\opxn{1}$ and $\opxn{2}$ [cf. \eqref{eq:estimation}].

\subsection{Theoretical Properties and Interpretations}
We then introduce several theoretical properties of the dependence decomposition and induced feature representations, including their connections to the principle of maximum entropy \citep{jaynes1957informationa, jaynes1957informationb} and the optimal transformation of variables \citep{breiman1985estimating}.

\subsubsection{Dependence Decomposition}
We can relate the bivariate-interaction decomposition \eqref{eq:pib:pii} to decomposition operations in both the probability distribution space and the data space. 
\paragraph{Decomposition in Distribution Space}
We assume that for all  $(x_1, x_2, y) \in \cX_1 \times \cX_2 \times \cY$, 
\begin{align}
  [\pib(\lpmi_{X_1, X_2; Y})](x_1, x_2, y) \geq -1,\quad   [\pii(\lpmi_{X_1, X_2; Y})](x_1, x_2, y) \geq -1,
  \label{eq:gtr:n1}
\end{align}
and define the associated distributions
\begin{subequations}
  \label{eq:pbi:pint}
  \begin{align}
    \pbi_{X_1, X_2, Y}(x_1, x_2, y) &\defeq P_{X_1, X_2}(x_1, x_2)P_Y(y) (1 + [\pib(\lpmi_{X_1, X_2; Y})](x_1, x_2, y)),  \label{eq:pbi:pint:1}
\\
    \pint_{X_1, X_2, Y}(x_1, x_2, y) &\defeq P_{X_1, X_2}(x_1, x_2)P_Y(y) (1 + [\pii(\lpmi_{X_1, X_2; Y})](x_1, x_2, y)).
  \end{align}
\end{subequations}
Then, we have the following characterization, a proof of which is provided in \appref{app:prop:pbi:pint}.

\begin{proposition}
  \label{prop:pbi:pint}
  Under assumption \eqref{eq:gtr:n1}, we have  $\pbi_{X_1, X_2, Y}, \pint_{X_1, X_2, Y} \in \cP^{\cX_1 \times \cX_2 \times \cY}$, with marginal distributions $\pbi_{X_1, X_2} = \pint_{X_1, X_2} = P_{X_1, X_2}$ and $\pbi_{X_i, Y} = P_{X_i, Y},  \pint_{X_i, Y} = P_{X_i}P_{Y}$ for $i = 1, 2$.
\end{proposition}

From \propref{prop:pbi:pint}, $\pint_{X_1, X_2, Y}$  has  marginal distributions 
 $\pint_{X_i, Y} = P_{X_i}P_{Y}$, $i = 1, 2$, %
and %
does not capture $(X_1; Y)$ or  $(X_2; Y)$ dependence.
On the other hand, $\pbi_{X_1, X_2, Y}$ 
has the same 
 pairwise marginal distributions as $P_{X_1, X_2, Y}$, i.e., 
  $\pbi_{X_1, X_2, Y} \in \cqb$ with $\cqb$ as defined in \eqref{eq:def:q2}. We can show that $\pbi_{X_1, X_2, Y}$ also achieves the maximum entropy in $\cqb$ in the local analysis regime. Formally, let 
    \begin{align}
      \pment_{X_1, X_2, Y} \defeq
      \argmax_{Q_{X_1, X_2, Y} \in \cqb} H(Q_{X_1, X_2, Y})
      \label{eq:pbi:def}
    \end{align}
denote the entropy maximizing distribution on $\cqb$, where $H(Q_{X_1, X_2, Y})$ denotes the entropy of $(X_1, X_2, Y)\sim Q_{X_1, X_2, Y}$. Then we have the following result.
    A proof is provided in \appref{app:prop:max-ent}.
    \begin{proposition}
      \label{prop:max-ent}
      Suppose $X = (X_1, X_2)$ and $Y$ are $\eps$-dependent, and let $\lpmi^{(\ent)}_{X_1, X_2; Y}$ denote the CDK function associated with $ \pment_{X_1, X_2, Y}$. Then, we have
      $\bnorm{\pib(\lpmi_{X_1, X_2; Y}) - \lpmi^{(\ent)}_{X_1, X_2; Y}} = o(\eps)$, 
or equivalently,
    \begin{align} %
      \pment_{X_1, X_2, Y}(x_1, x_2, y) = \pbi_{X_1, X_2, Y}(x_1, x_2, y) + o(\eps), \quad \text{for all $x_1, x_2, y$}.
      \label{eq:max:ent:bi}
    \end{align}
  \end{proposition}

\paragraph{Decomposition in Data Space} %
For each triplet $(x_1, x_2, y) \in \cX_1 \times \cX_2 \times \cY$, we consider the decomposition 
\begin{align}
  (x_1, x_2, y) \mapsto (x_1, x_2), (x_1, y), (x_2, y).
  \label{eq:dcmp:triple}
\end{align}

Suppose the dataset\footnote{Though the dataset is modeled as a multiset without ordering, we introduce the index $i$ for the convenience of presentation, which corresponds to a specific realization for traversing the dataset.} 
 $\Ds \defeq \left\{\bigl(x_1^{(i)}, x_2^{(i)}, y^{(i)}\bigr)\right\}_{i \in [n]}$ 
has the empirical distribution $P_{X_1, X_2, Y}$, where each tuple $\bigl(x_1^{(i)}, x_2^{(i)}, y^{(i)}\bigr) \in \cX_1 \times \cX_2 \times \cY$. Then, by applying this decomposition on $\Ds$ and grouping the decomposed pairs, we obtain 
three separate datasets 
\begin{align}
\left\{\bigl(x_1^{(i)}, x_2^{(i)}\bigr)\right\}_{i \in [n]},
\left\{\bigl(x_1^{(i)}, y^{(i)}\bigr)\right\}_{i \in [n]},
\left\{\bigl(x_2^{(i)}, y^{(i)}\bigr)\right\}_{i \in [n]},
\label{eq:ds:pair}  
\end{align}
which have empirical distributions $P_{X_1, X_2}$, $P_{X_1, Y}$, and $P_{X_2, Y}$, respectively.

Therefore, we can interpret the decomposition \eqref{eq:dcmp:triple} as extracting the bivariate dependence component from the joint dependence: the new pairwise datasets retain all pairwise dependence, but do not capture any interaction among $X_1, X_2, Y$. Indeed, it is easy to see that, for each dataset with an empirical distribution taken from  $\cqb$, the decomposition \eqref{eq:dcmp:triple} leads to the same pairwise datasets. Reversely, we can reconstruct $\pbi_{X_1, X_2, Y}$ from the pairwise datasets \eqref{eq:ds:pair}. We will discuss the reconstruction algorithm design later.

\subsubsection{Feature Representations} 
Let $\Kb \defeq \rank(\pib(\lpmi_{X_1, X_2; Y}))$ and $K \defeq \rank(\pii(\lpmi_{X_1, X_2; Y}))$. Then, we can represent the dependence modes of the bivariate component $\pib(\lpmi_{X_1, X_2; Y})$ and  $\pii(\lpmi_{X_1, X_2; Y})$ in their standard forms, as
\begin{subequations}
  \label{eq:mm:fb:f}
  \begin{align}
   \mdn{i}(\pib(\lpmi_{X_1, X_2; Y})) &=  \sigmab_i\, (\fb_i^*\otimes \gb_i^*), \quad i \in [\Kb],  \label{eq:mm:fb}\\
   \mdn{i}(\pii(\lpmi_{X_1, X_2; Y})) &=  \sigma_i \,(f_i^*\otimes g_i^*),\quad i \in [K].
  \end{align}
\end{subequations}

By applying \propref{prop:max:corr:const}, we can interpret these features as solutions to corresponding constrained maximal correlation problems. For example, since $\mdn{i}(\pib(\lpmi_{X_1, X_2; Y})) = \mdn{i}(\lpmi_{X_1, X_2; Y}|\spc{\cX_1} + \spc{\cX_2}, \spc{\cY})$,  $(\fb_i^*, \gb_i^*)$ is the $i$-th constrained maximal correlation function pair of $X = (X_1, X_2)$ and $Y$ restricted to subspaces $ \spc{\cX_1} + \spc{\cX_2}$ and $\spc{\cY}$, respectively.

The top mode $(\sigmab_1, \fb_1^*, \gb_1^*)$ in \eqref{eq:mm:fb} also characterizes the optimal solution to a classical regression formulation. Specifically, given input variables $X_1, X_2$ and the output variable $Y$, \cite{breiman1985estimating} formulated the  regression problem
  \begin{align}
 \minimize_{\substack{\phi^{(1)} \in \spct{\cX_1}, \phi^{(2)} \in \spct{\cX_2}\\\psi \in \spct{\cY}\colon \norm{\psi} = 1}} \E{\left(\psi(Y) - \phi^{(1)}(X_1) - \phi^{(2)}(X_2)\right)^2},
    \label{eq:opt:tran}
  \end{align}
where the minimization is over zero-mean functions $\phi^{(1)}, \phi^{(2)}$, and $\psi$. The solution of \eqref{eq:opt:tran}, referred to as the optimal transformations \citep{breiman1985estimating}, can be characterized as follows. A proof is provided in \appref{app:prop:opt:tran}.

\begin{proposition}
  \label{prop:opt:tran}
  The minimum value of optimization problem \eqref{eq:opt:tran} is $1 - \sigmab_1^2$, which can be achieved by $\phi^{(1)} + \phi^{(2)}= \sigmab_1 \cdot \fb_1^*$ and $\psi = \gb_1^*$.
\end{proposition}

Therefore, the optimal transformations depend on, and thus characterize only, the top mode of the bivariate dependence component $\pib(\lpmi_{X_1, X_2; Y})$. %

\subsection{Learning With Missing Modalities}%
We conclude this section by briefly discussing feature learning based on incomplete samples. %

\subsubsection{Learning from Pairwise Samples}
\label{sec:mm:miss:pair}
A special case of the incomplete samples is the pairwise datasets \eqref{eq:ds:pair} obtained
from the decomposition
 \eqref{eq:dcmp:triple}. Specifically, suppose we obtain \eqref{eq:ds:pair} from $\Ds \defeq \left\{\bigl(x_1^{(i)}, x_2^{(i)}, y^{(i)}\bigr)\right\}_{i \in [n]}$, and let  $P_{X_1, X_2, Y}$ denote the empirical distribution of $\Ds$.  Since the bivariate dependence is retained in the decomposition \eqref{eq:dcmp:triple}, we can learn  $\pib(\lpmi_{X_1, X_2; Y})$ from the pairwise datasets \eqref{eq:ds:pair}. 

In particular, when we set $k = 0$ in $\cbi$ [cf. \eqref{eq:cbi:def}], we have %
   $ \Hs\left(
    \begin{bmatrix}
      \fb\,\\
      f\,
    \end{bmatrix}
    ,
    \begin{bmatrix}
      \gb\\
      g
    \end{bmatrix}
    ; \cbi
    \right)
   = 2 \cdot \Hs(\fb, \gb)$, and
 \begin{align}
  \Hs(\fb, \gb)&= \Hs\left(\fb^{(1)} + \fb^{(2)}, \gb\right)\notag\\
  &= \E{\left(\fb^{(1)}(X_1) + \fb^{(2)}(X_2)\right)^\T \gb(Y)} - \left(\E{\fb^{(1)}(X_1) + \fb^{(2)}(X_2)}\right)^\T \E{\gb(Y)}- \frac12 \tr\left(\La_{\fb} \La_{\gb}\right)\notag\\
  &= \Hs(\fb^{(1)}, \gb) + \Hs(\fb^{(2)}, \gb) - \tr\left( \La_{\fb^{(1)}, \fb^{(2)}}  \La_{\gb}\right).\label{eq:Hs:fb:gb} 
\end{align}
Therefore, we can evaluate \eqref{eq:Hs:fb:gb} from the pairwise datasets \eqref{eq:ds:pair}, since each $\Hs(\fb^{(i)}, \gb)$ depends only on $P_{X_i, Y}$ for $i = 1, 2$, and  $\La_{\fb^{(1)}, \fb^{(2)}}$ depends only on $P_{X_1, X_2}$. Then, from \corolref{cor:hsbinest}, we can obtain $\pib(\lpmi_{X_1, X_2; Y})$ and the same set of features. %

\subsubsection{General Heterogeneous Training Data}
\label{sec:mm:miss:gen}
We then consider general forms of heterogeneous training data, as shown in \tabref{tab:mm:data}. In particular, suppose there are $n \defeq n_0 + n_1 + n_2$ training samples, and we group them into separate datasets: $\Ds_0$ contains $n_0$ complete observations of $(X_1, X_2, Y)$, and, for $i =1, 2$, each $\Ds_i$ has $n_i$ sample pairs of $(X_i, Y)$. Our goal is to learn features from the heterogeneous data and obtain similar inference models as we introduced in \secref{sec:mm:assemble}.

\begin{table}[!ht]
  \centering
  \begin{tabular}[h]{lccc}
    \toprule
   \quad\qquad Datasets  & Empirical Distribution & Remark\\
    \midrule
    $\ds\Ds_0 = \bigl\{(x^{(i)}_1, x^{(i)}_2, y^{(i)})\bigr\}_{i = 1}^{n_0}$ %
                                           & $\Ph^{(0)}_{X_1, X_2, Y}$ & Complete Observation\\
    $\ds\Ds_1 = \bigl\{(x^{(i)}_1, y^{(i)})\bigr\}_{i =  n_0 + 1}^{n_0 + n_1}$  %
                                           &  $\Ph^{(1)}_{X_1, Y}$ & $X_2$ missing \\
    $\ds\Ds_2 = \bigl\{(x^{(i)}_2, y^{(i)})\bigr\}_{i = n_0+n_1+1}^{n_0+n_1+n_2}$   %
                                           & $\Ph^{(2)}_{X_2, Y}$ & $X_1$ missing\\
    \bottomrule
  \end{tabular}
  \caption{Heterogeneous Training Data With Missing Modalities}
  \label{tab:mm:data}
\end{table}

In this case, in addition to the empirical distributions of these datasets, we also need to consider the sample sizes $n_0, n_1, n_2$ that indicate the relative qualities. To begin, we use a metric distribution of the product form $R_{X_1, X_2, Y} = R_{X_1, X_2}R_Y$, where $R_{X_1, X_2}$ and $R_Y$ correspond to some empirical distributions of training data. For example, we can set $R_{X_1, X_2} = \Ph^{(0)}_{X_1, X_2}$ and $R_Y = \eta_0 \Ph^{(0)}_Y + \eta_1 \Ph^{(1)}_Y + \eta_2\Ph^{(2)}_{Y}$ with $\eta_i \defeq n_i/n$ for $i = 0, 1, 2$, which correspond to the empirical distributions of all $(X_1, X_2)$ sample pairs and all $Y$ samples, respectively.

Then, for any given $Q_{X_1, X_2, Y} \in \cP^{\cX_1 \times \cX_2 \times Y}$, we
use a weighted sum $L(Q_{X_1, X_2, Y})$ to characterize the difference between $Q_{X_1, X_2, Y}$ and the heterogeneous observations, defined as  %
\begin{align}
 L(Q_{X_1, X_2, Y}) \defeq \eta_0 \cdot \bnorm{\llrt{\Ph^{(0)}_{X_1, X_2, Y}} - \llrt{Q_{X_1, X_2, Y}}}^2  +   \eta_1 \cdot \bnorm{ \llrt{\Ph^{(1)}_{X_1, Y}} - \llrt{Q_{X_1, Y}}}^2 +   \eta_2 \cdot \bnorm{\llrt{\Ph^{(2)}_{X_2, Y}} - \llrt{Q_{X_2, Y}}}^2.
  \label{eq:L:bi}
\end{align}
We use $\Pest_{X_1, X_2, Y}$ to denote the optimal distribution that minimizes \eqref{eq:L:bi}.

We can again apply the nesting technique to learn the feature representations associated with $\Pest_{X_1, X_2, Y}$.
 To begin, we use $ \Hs(f, g; Q_{X, Y})$ to denote the H-score computed over the joint distribution $Q_{X, Y}$, defined as %
\begin{align*}
  \Hs(f, g; Q_{X, Y}) &\defeq
  \frac12\left(\bnorm{\llrt{Q_{X, Y}}}^2 - \bbnorm{\llrt{Q_{X, Y}} - f \otimes g}^2\right)\\
    &= \Ed{Q_{X, Y}}{f^{\T}(X) g(Y)}
    - \left(\Ed{R_X}{f(X)}\right)^{\T}\Ed{R_Y}{g(Y)}%
     - \frac12  \cdot \tr\left(\La_{f}\La_{g}\right),
\end{align*}
with $\La_f = \Ed{R_X}{f(X)f^\T(X)}$ and  $\La_g = \Ed{R_Y}{g(Y)g^\T(Y)}$. %

Then, we define the H-score associated with 
the heterogeneous datasets shown in \tabref{tab:mm:data}, as
  \begin{align}
   \Hsm(f, g) \defeq \eta_0 \cdot \Hs\bigl(f, g; \Ph^{(0)}_{X_1, X_2, Y}\bigr)
    +   \eta_1 \cdot \Hs\left(\opxn{1}(f), g; \Ph^{(1)}_{X_1, Y}\right) +   \eta_2 \cdot  \Hs\left(\opxn{2}(f), g; \Ph^{(2)}_{X_2, Y}\right),
    \label{eq:hsm:def}
 \end{align}
 where we have defined conditional expectation operators $\opxn{i}, i = 1, 2$ as in \eqref{eq:tau:i:def}, with respect to the distribution $R_{X_1, X_2}$.  By applying the  same nesting configuration $\cbi$, we can obtain the corresponding nested H-score 
  \begin{align}
    \Hsm\left(   \begin{bmatrix}
     \fb\,\\
     f\, 
   \end{bmatrix},
      \begin{bmatrix}
     \gb\\
     g
   \end{bmatrix}; \cbi\right) =
   \Hsm(\fb, \gb) + \Hsm\left(
   \begin{bmatrix}
     \fb\,\\
     f\, 
   \end{bmatrix},
      \begin{bmatrix}
     \gb\\
     g
   \end{bmatrix}
   \right).
    \label{eq:hsbi:miss}
 \end{align}

 Then, we have the following theorem, which extends \corolref{cor:hsbinest} to the heterogeneous datasets. A proof is provide in \appref{app:thm:hsbih}.
 \begin{theorem}%
   \label{thm:hsbih}
   Given $\kb \geq \rank\Bigl(\pib\Bigl(\llrt{\Pest_{X_1, X_2, Y}}\Bigr)\Bigr)$, 
the nested H-score
   $ \Hsm\left(   \begin{bmatrix}
     \fb\,\\
     f\, 
   \end{bmatrix},
      \begin{bmatrix}
     \gb\\
     g
   \end{bmatrix}; \cbi\right)$ as defined in  \eqref{eq:hsbi:miss} is maximized  if and only if
 \begin{align}
    \fb \otimes \gb = \pib\Bigl(\llrt{\Pest_{X_1, X_2, Y}}\Bigr),\qquad
   f \otimes g = \mdnl{k}\left(\pii\Bigl(\llrt{\Pest_{X_1, X_2, Y}}\Bigr)\right).
   \label{eq:hsbih:opt}
 \end{align}
\end{theorem}
We can also use the refined configuration $\cbir$ to obtain modal decomposition of the dependence components. The inference models can be built by assembling learned features, as we have discussed in \secref{sec:mm:assemble}.

Furthermore, we can show that the estimation $\Pest_{X_1, X_2, Y}$ coincides with the maximum likelihood estimation (MLE) in a local analysis regime. Formally, let $\prob{\Ds_0, \Ds_1, \Ds_2; Q_{X_1, X_2, Y}}$ denote the probability of observing datasets $\Ds_0, \Ds_1, \Ds_2$, when all data samples are independently generated by $Q_{X_1, X_2, Y}$. Then, 
we can write the MLE solution as
\begin{align}%
  \Pml_{X_1, X_2, Y} \defeq \argmax_{Q_{X_1, X_2, Y}}~  \prob{\Ds_0, \Ds_1, \Ds_2; Q_{X_1, X_2, Y}},
  \label{eq:pml}
\end{align}
and we have the following characterization. A proof is provided in \appref{app:thm:pest:pml}. 
\begin{theorem}
  \label{thm:pest:pml}
  If $L(R_{X_1, X_2, Y}) = O(\eps^2)$, then we have %
  \begin{align}
    \Pml_{X_1, X_2, Y}(x_1, x_2, y) = \Pest_{X_1, X_2, Y}(x_1, x_2, y) + o(\eps), \quad \text{for all $x_1, x_2, y$}.
    \label{eq:pml:pest}
  \end{align}
\end{theorem}

\section{Experimental Verification}
\label{sec:exp}

To verify the learning algorithms as well as established theoretical properties, we design a series of experiments with various types of data. 
{The main goal is to compare the features learned by neural feature extractors and the corresponding theoretical results. To allow such comparisons, we generate data from given probability distributions of which we know the analytical form of optimal features.} The source codes for all experiments are available at \href{https://github.com/XiangxiangXu/NFE}{\texttt{github.com/XiangxiangXu/NFE}}, and
we defer the implementation details to \appref{app:exp}.

\subsection{Learning Maximal Correlation Functions}
\label{sec:exp:mc}
We first consider learning dependence modes, i.e., maximal correlation functions from sample pairs of $(X, Y)$, by maximizing the nested H-score \eqref{eq:H:cnest}. We verify the effectiveness by experiments on both discrete and continuous data, and also discuss one application in analyzing sequential data.     

\subsubsection{Discrete Data}
\label{sec:exp:mc:discrete}
The simplest case for dependence learning is when $X$ and $Y$ are both discrete with small alphabet sizes $|\cX|$ and $|\cY|$. In this case, we can design neural feature extractors with ideal expressive powers. Suppose $\cX = \{1, \dots, |\cX|\}$, then we can express $f \in \spcn{\cX}{k}$ on $\cX$ by first mapping each $i \in \cX$ to $i$-th standard basis vector in $\mathbb{R}^{|\cX|}$, also known as the ``one-hot encoding'' in practice, and then applying a linear function $\mathbb{R}^{|\cX|} \to \mathbb{R}^{k}$ to the mapped result, which we implement by using a linear layer. Then, any $f \in \spcn{\cX}{k}$ can be expressed in this way by setting corresponding weights in the linear layer. Similarly, we can express  $g \in \spcn{\cY}{k}$ using another linear layer.

\begin{figure}[!ht]
  \centering   
  \subfloat[$f_i^*(x)$ for each $i = 1, 2, 3$]{\includegraphics[height = .33\textwidth]{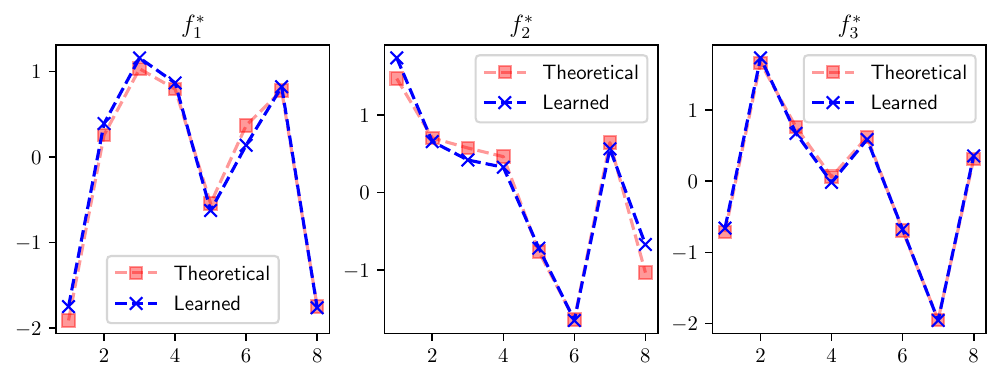}}\linebreak
  \subfloat[$g_i^*(y)$ for each $i = 1, 2, 3$]{\includegraphics[height = .35\textwidth]{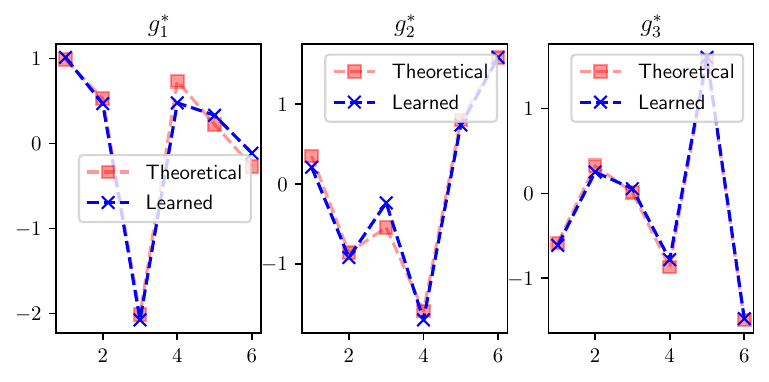}}
  \hspace{2em}
  \subfloat[$\sigma_i$]{\includegraphics[height = .33\textwidth]{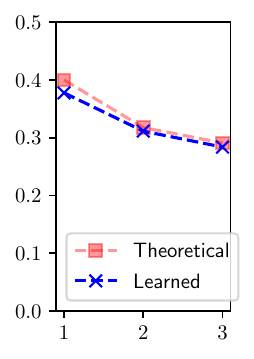}}
  \caption{Top three dependence modes learned from a discrete dataset, which are consistent with the theoretical results.}
  \label{fig:exp:discrete}
\end{figure}

In the experiment, we set $|\cX| = 8$, $|\cY| = 6$, and randomly generate a $P_{X, Y} \in \cP^{\cX \times \cY}$. We generate 
 $N = 30\,000$ training samples
from $P_{X, Y}$, and learn $k = 3$ dimensional features $f, g$ by maximizing $\Hs(f, g; \cnest)$. Then, we normalize each $f_i, g_i$ to obtain  corresponding estimations of $f_i^*$, $g_i^*$, and $\sigma_i$ by applying \eqref{eq:f:g:simga:i}. We show the estimated features and singular values in \figref{fig:exp:discrete}, which are consistent with the corresponding theoretical values computed from $P_{X, Y}$.

\subsubsection{Continuous Data}
\label{sec:exp:mc:cosine}
We proceed to consider a continuous dataset with degenerate dependence modes, i.e., the singular values $\sigma_i$'s are not all distinct.

 In particular, we consider $X, Y$ taking values from $\cX = \cY = [-1, 1]$, where the joint probability density function $p_{X, Y}$ takes a 
raised cosine form:
  \begin{align}
    p_{X, Y}(x, y) = \frac{1}{4} \cdot \left[ 1 + \cos(\pi(x- y))\right], \quad (x, y) \in [-1, 1]^2.
    \label{eq:pdf:cos}
  \end{align}
Then, it can be verified that the corresponding marginal distributions of $X, Y$ are uniform distributions $p_X = p_Y =\Unif([-1, 1])$. In addition, the resulting CDK function is
  \begin{align}
    \lpmi_{X; Y}(x, y) = \frac{p_{X, Y}(x, y) - p_X(x)p_Y(y)}{p_X(x)p_Y(y)}
    &=\cos(\pi(x - y)).%
      \label{eq:lpmi:cosine}
  \end{align}

  Note that we have
     $\cos(\pi(x - y)) = \cos(\pi x + \theta_0)\cdot\cos(\pi y + \theta_0) + \sin(\pi x + \theta_0)\cdot\sin(\pi y + \theta_0)$,
 for any $\theta_0 \in [-\pi, \pi)$. Therefore, we have   $\rank(\lpmi_{X; Y}) = 2$, and the associated dependence modes are given by 
 $\sigma_1 = \sigma_2 = 1/2$ and the maximal correlation functions
 \begin{subequations}
   \label{eq:cos:fg}
   \begin{gather}
     f_1^*(x) = \sqrt{2} \cdot \cos(\pi x + \theta_0), \qquad f_2^*(x) = \sqrt{2} \cdot \sin(\pi x + \theta_0),\\
     g_1^*(y) = \sqrt{2} \cdot \cos(\pi y + \theta_0), \qquad g_2^*(y) = \sqrt{2} \cdot \sin(\pi y + \theta_0),
   \end{gather}
\end{subequations}
for any $\theta_0 \in [-\pi, \pi)$.

\begin{figure}[!ht]
  \centering
  \subfloat[Histogram of Training Data]{\includegraphics[height = .31\textwidth]{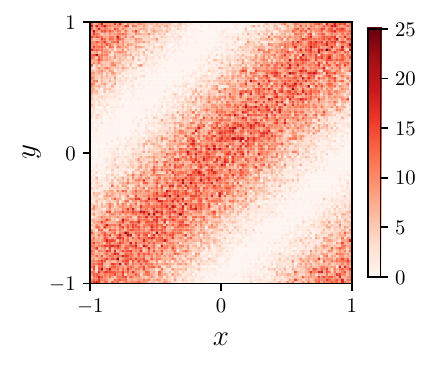}\label{fig:exp:cos:hist}}
  \subfloat[Learned Dependence Modes]{\includegraphics[height = .31\textwidth]{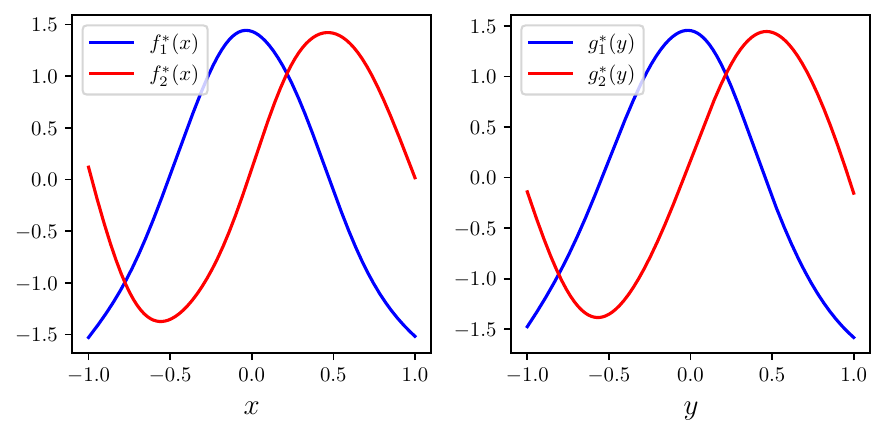}\label{fig:exp:cos:fg}}
  \caption{Dependence modes learned from continuous data, in consistent with theoretical results.}
  \label{fig:exp:cos}
\end{figure}

During this experiment, we first generate $N = 50\,000$ sample pairs of $X, Y$ for training, with histogram showing in \figref{fig:exp:cos:hist}. Then, we learn $k = 2$ dimensional features $f_1, f_2$ of $\cX$ and $g_1, g_2$ of $\cY$ by maximizing the nested H-score \eqref{eq:H:cnest}, where $f$ and $g$ are parameterized neural feature extractors detailed in \appref{app:exp:mc:cosine}. \figref{fig:exp:cos:fg} shows the learned functions after the normalization \eqref{eq:f:g:simga:i}. The learned results well match the theoretical results \eqref{eq:cos:fg}: (i) The learned $f_1^*$ and $f_2^*$ are sinusoids differ in phase by $\pi/2$, and (ii) $g_i^*$ coincides with $f_i^*$, for each $i = 1, 2$. It is also worth mentioning that due to the degeneracy $\sigma_1 = \sigma_2$, the initial phase $\theta_0$ in learned sinusoids \eqref{eq:cos:fg} can be different during each run of the training algorithm.

\begin{figure}[!ht]
  \centering
  \includegraphics[width = .7\textwidth]{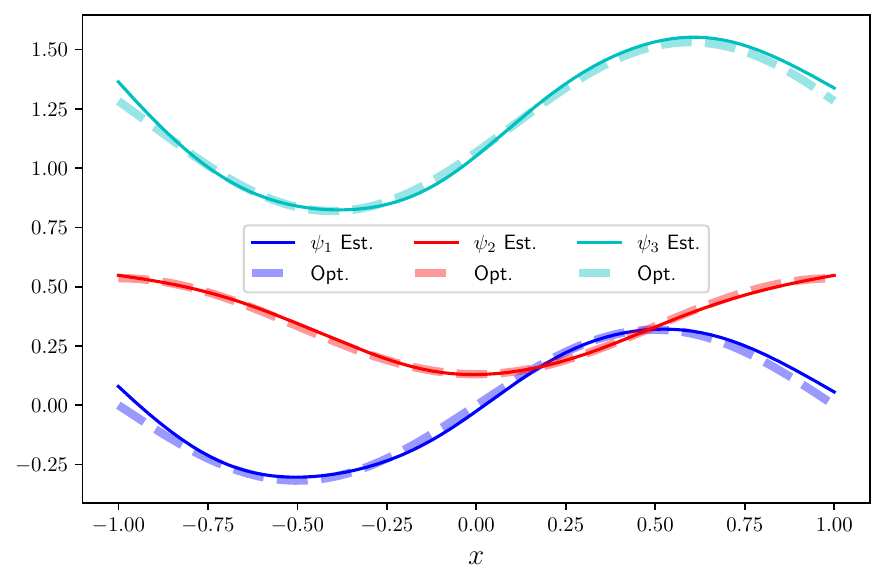}
  \label{fig:exp:cosine:est}
  \caption{MMSE estimators $\E{\psi_i(Y)|X = x}$ obtained from learning dependence modes, in comparison with theoretical results.}
\end{figure}

Based on the learned dependence modes, we then demonstrate estimating functions of $Y$ based on observed $X = x$. Here, we consider the functions $\psi_1(y) = y, \psi_2(y) = y^2, \psi_3(y) = e^y$. From \propref{prop:pred:est}, we can compute the learned MMSE estimator $\E{\psi_i(Y)|X = x}$ for each $i$, by estimating $\E{\psi_i(Y)}$ and $\La_{\psi_i, g} = \E{\psi_i(Y) g^\T (Y)}$ from the training set and then applying \eqref{eq:estimation}.

For comparison, we compute the theoretical values 
  \begin{align*}
    \E{\psi_i(Y)|X = x} = \int_{-1}^1 p_{Y|X}(y|x) \psi_i(y) \, dy,
  \end{align*}
with $p_{Y|X}(y|x) = \frac{1}{2} \cdot \left[ 1 + \cos(\pi(y- x))\right]$, which gives
  \begin{subequations}
    \label{eq:ce:psi}
  \begin{gather}
    \E{Y\middle|X = x}= \frac{1}{\pi}\cdot \sin(\pi x),\qquad
    \ds\E{Y^2\middle|X = x} = \frac13 - \frac{2}{\pi^2} \cos(\pi x),\\
    \E{e^Y\middle|X = x} = \frac{e^2 - 1}{2e(1 + \pi^2)} \cdot (\pi \sin(\pi x) - \cos(\pi x) + \pi^2 + 1).
  \end{gather}
\end{subequations}

\figref{fig:exp:cosine:est} shows the estimators {obtained by applying \eqref{eq:estimation} on the learned features}, which are consistent with the theoretically optimal estimators given by \eqref{eq:ce:psi}.

\subsubsection{Sequential Data}
\label{sec:exp:mc:seq}

We proceed with an example of learning dependence modes among sequence pairs. For simplicity, we consider binary sequences $\vec{X}$ and $\vec{Y}$, of lengths $l$ and $m$, respectively. Suppose we have the Markov relation $\vec{X} - U - V - \vec{Y}$ for some unobserved binary factors $U, V \in \cU = \cV = \{0, 1\}$. In addition, we assume\footnote{For convenience, we adopt the vector notation to represent sequences.} $\vec{X} = (X_1, \dots, X_{l})^\T, \vec{Y} = (Y_1, \dots, Y_{m})^\T$ satisfy
  \begin{align}
    \vec{X}|U = i \sim \bms(l, q_i), \quad \vec{Y}|V = i \sim \bms(m, q_i), \quad\text{for $i = 0, 1$},
    \label{eq:xy:bms}
  \end{align}
where $\bms(l, q)$ denotes the distribution of a binary first-order Markov sequence of length $l$ with state flipping probability $q$. The corresponding state transition diagram is shown in \figref{fig:exp:seq:state}. Therefore, if $\vec{Z} \sim \bms(l, q)$, then $Z_1 \sim \Unif(\{0, 1\})$ and $(Z_1, \dots, Z_l)$ forms a first order Markov chain over binary states $\{0, 1\}$, with flipping probability $\prob{Z_{i+1} \neq Z_i|Z_i = z} = q$ for both $z = 0, 1$. Formally, $\vec{Z} \sim \bms(l, q)$ if and only if
\begin{align*}
  P_{\vec{Z}}(\vec{z}) =   
\frac12 \cdot \prod_{i = 1}^{l -1}
  \left[(1-q)^{\delta_{z_i z_{i+1}}} \cdot q^{1 - \delta_{z_i z_{i+1}}}\right] \quad\text{for all $(z_1, \dots, z_l) \in \{0, 1\}^{l}$}.
\end{align*}

\begin{figure}[!ht]
  \centering
  \raisebox{-2.35cm}{ \subfloat[State Transitions of $\bms(l,q)$]{\raisebox{1.7cm}{\resizebox{.3\textwidth}{!}{\DrawMarkov{q}}}\label{fig:exp:seq:state}}}
  \hspace{.7em}
  \subfloat[Generated Samples of $\vec{X}, \vec{Y}$ Pair  ]{ %
    \begin{minipage}{.64\linewidth}
      \includegraphics[width = \textwidth]{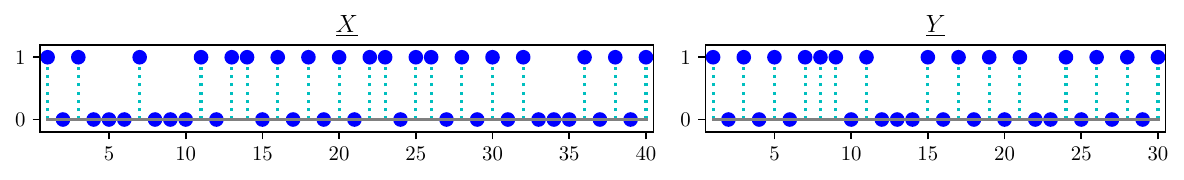}\textcolor{gray}{\hrule}
      \includegraphics[width = \textwidth]{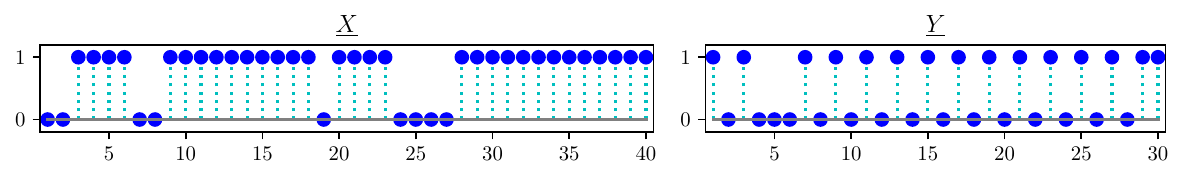}\textcolor{gray}{\hrule}
      \includegraphics[width = \textwidth]{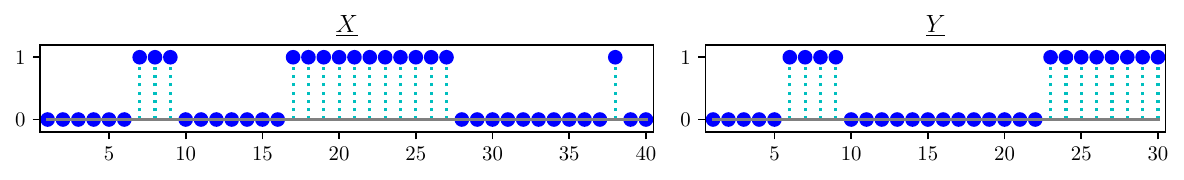}
    \end{minipage}
  \label{fig:exp:seq:sample}}
  \caption{Sequential Data: Model and Generated $\vec{X}, \vec{Y}$ Samples Pairs  }
  \label{fig:exp:seq:data}
\end{figure}

As a consequence, the resulting alphabets are $\cX = \{0, 1\}^{l}, \cY = \{0, 1\}^{m}$, with sizes $|\cX| = 2^l, |\cY| = 2^m$. In our experiment, we set $l = 40, m = 30, q_0 = 0.1, q_1 = 0.9$, and use the following joint distribution $P_{U, V}$:
\begin{center}
  \begin{tabular}{ccc}
    \toprule  
    Prob. & $U = 0$ & $U = 1$\\
    \midrule
    $V = 0$ &0.1 &0.2\\
    $V = 1$ &0.4 &0.3\\
    \bottomrule
  \end{tabular}
\end{center}
We generate $N = 50\,000$ training sample pairs of $\vec{X}, \vec{Y}$, with instances shown in \figref{fig:exp:seq:sample}. We also generate $N' = 10\,000$ sample pairs in the same manner, as the testing dataset.

\begin{figure}[!ht]
  \centering
  \includegraphics[width = .65\textwidth]{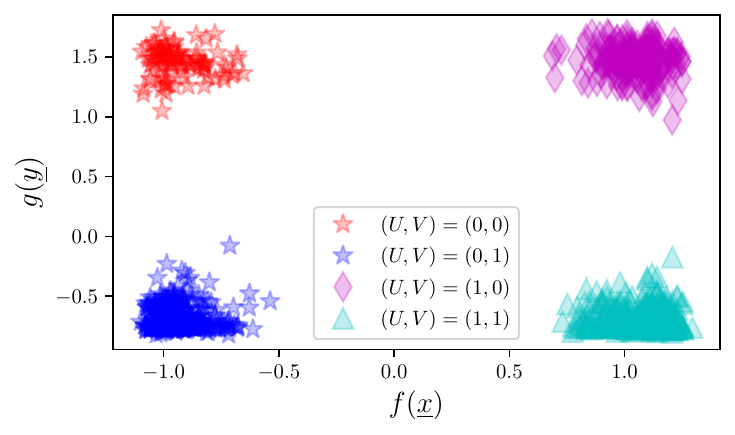}  
  \caption{Features Learned from Sequential Data $\vec{X}$ and $\vec{Y}$}
  \label{fig:exp:seq:fg}
\end{figure}

Then, we learn $k = 1$ dimensional features $f$ and $g$ by maximizing $\Hs(f, g)$ over the training set. We plot the extracted features in \figref{fig:exp:seq:fg}. In particular, each point represents an $(f(\vec{x}), g(\vec{y}))$ pair evaluated on an instance from testing set, with corresponding values of binary factors $(U, V)$ shown for comparison. For the ease of demonstration, here we plot only $1,000$ sample pairs randomly chosen from the testing set. From the figure, the learned features are clustered according to the underlying factors. This  essentially reconstructs the hidden factors $U, V$. For example, one can apply a standard clustering algorithm on the features, e.g., $k$-means \citep{hastie2009elements}, then count the proportion of each cluster, to obtain an estimation of $P_{U, V}$ up to permutation of symbols.

For a closer inspection, we can compare the learned features with the theoretical results, formalized as follows. A proof is provided in \appref{app:prop:seq}.
\begin{proposition}
\label{prop:seq}
Suppose $\vec{X}, \vec{Y}$ satisfy the Markov relation $\vec{X} - U - V - \vec{Y}$ with $U, V \in \{0, 1\}$ and the conditional distributions \eqref{eq:xy:bms}. Then, we have $\rank(\lpmi_{X; Y}) \leq 1$, and the corresponding maximal correlation functions $f_1^*, g_1^*$ satisfy
\begin{subequations}
  \label{eq:seq:fg}
  \begin{align}
    f_1^*(\vec{x}) &= c \cdot \left[ \tanh\left(2 w  \cdot \varphi(\vec{x}) + b_U\right) - \tanh(b_U)\right],\\
    g_1^*(\vec{y}) &= c' \cdot \left[\tanh\left(2 w  \cdot \varphi(\vec{y}) + b_V\right) - \tanh(b_V)\right],
  \end{align}
\end{subequations}
for some $c, c' \in \mathbb{R}$, where $w \defeq \frac12\log\frac{q_1}{q_0},  b_U \defeq \frac12 \log\frac{P_{U}(1)}{P_U(0)}, b_V \defeq \frac12 \log\frac{P_{V}(1)}{P_V(0)}$, and where
we have defined $\varphi \colon \{0, 1\}^* \to \mathbb{R}$, such that for each $\vec{z} = (z_1, \dots, z_l)^\T \in \{0, 1\}^l$, we have
  $\ds\varphi(\vec{z}) \defeq\frac{l - 1}2 -  \sum_{i = 1}^{l - 1}\delta_{z_i z_{i + 1}}.$
\end{proposition}

Then, we compute the correlation coefficients between 
$f(\vec{X})$ and $f_1^*(\vec{X})$, and between $g(\vec{Y})$, $g_1^*(\vec{Y})$, respectively, using sample pairs in the testing set. The absolute values of both correlation coefficients are greater than $0.99$, demonstrating the effectiveness of the learning algorithm.

\subsection{Learning With Orthogonality Constraints}
\label{sec:exp:fbar}
We verify the feature learning with orthogonality constraints presented in \secref{sec:nest:orth} on the same dataset used for \secref{sec:exp:mc:cosine}. Here, we consider the settings
 $\kb = k = 1$, i.e., we learn one-dimensional feature $f(x)$ uncorrelated to given one dimensional feature $\phi \in \spc{\cX}$. 

Note that without any orthogonality constraints [cf. \eqref{eq:cos:fg}], the optimal feature will be sinusoids with any initial phase, i.e., 
 $f_1^*(x) = \sqrt{2} \cdot \cos(\pi x + \theta_0)$ for any $\theta_0 \in [-\pi, \pi)$. Here, we consider the following two choices of $\phi$, $(x \mapsto  x)$ and $(x \mapsto x^2)$, which are even and odd functions, respectively. Since the underlying $p_X$ is uniform on $[-1, 1]$, we can verify the optimal features under the two constraints are 
  $f^*_1(x) = \sqrt{2}\cos(\pi x)$ for $\phi(x) = x$, and
  $f^*_1(x) = \sqrt{2}\sin(\pi x)$ for $\phi(x) = x^2$.

\begin{figure}[!ht]
  \centering
  \subfloat[$\phi(x) = x$]{\includegraphics[width = .4\textwidth]{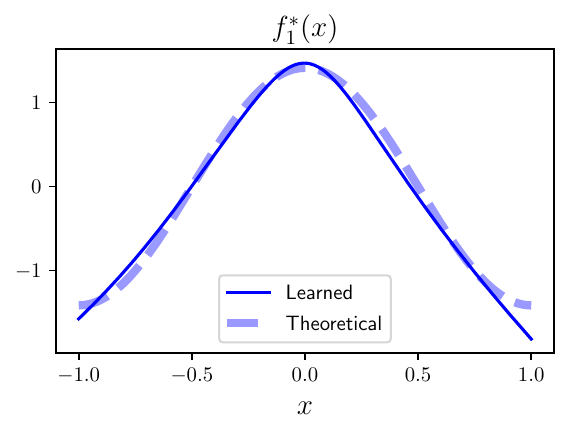}}
  \subfloat[$\phi(x) = x^2$]{\includegraphics[width = .4\textwidth]{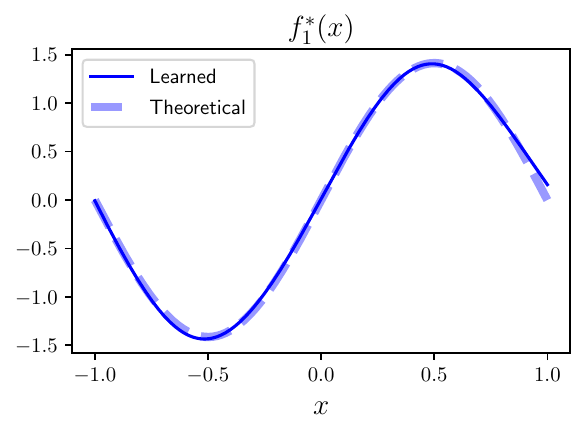}}
  \caption{Learning features uncorrelated to given $\phi$ under different settings of $\phi$. The learned results are compared with theoretical results.
}
  \label{fig:exp:cosine:fbar}
\end{figure}

By maximizing the nested H-score
restricted to $\fb = \phi$
 [cf.   \eqref{eq:Hsf:def}], we can learn the optimal feature $f_1^*$, as shown in \figref{fig:exp:cosine:fbar}. The learned features are in consistent with the theoretical ones, validating the effectiveness of the learning algorithm.

\subsection{Learning With Side Information}
\label{sec:exp:side}
We design an experiment to verify the connection between our learning algorithm and the  multitask classification DNN, as demonstrated in \thmref{thm:dnn:mulithead}.

\begin{figure}[!ht]
  \centering
  \subfloat[Learned feature $f(x)$]{\includegraphics[height = .28\textwidth]{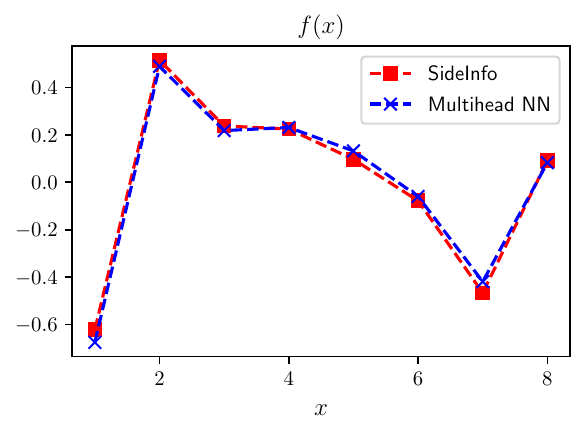}}
  \hspace{2em}
  \subfloat[Learned feature $g(s, y)$]{\includegraphics[height = .28\textwidth]{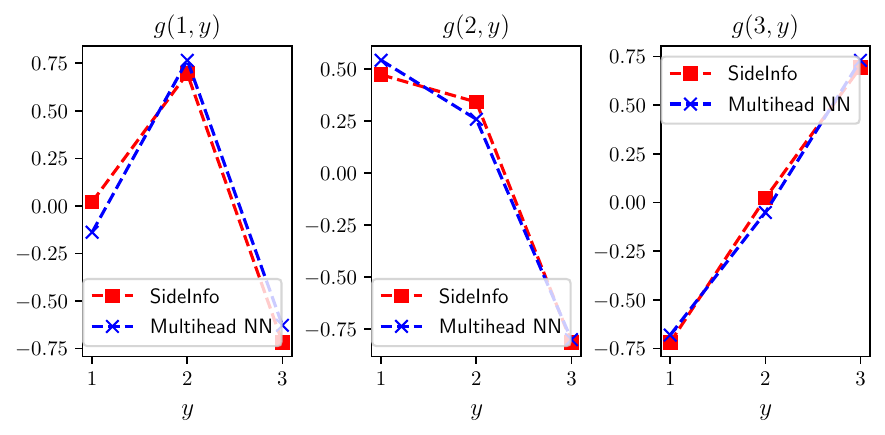}}
  \caption{Experimental verification of the connection between learning with side information and training a multihead neural network.
}
  \label{fig:exp:multihead}
\end{figure}

In particular, we consider the discrete $X, Y, S$ with 
$|\cX| = 8, |\cS| = |\cY| = 3$, and randomly choose a joint distribution $P_{X, S, Y}$ on $\cX \times \cS \times \cY$. Then, we generate $N = 50\,000$ training samples of $(X, S, Y)$ triples. In our implementation, we set $\kb = |\cS| - 1 = 2, k = 1$ and maximize the nested H-score configured by $\cside$ [cf. \eqref{eq:cside:def}] on the  training set.

For comparison, we also train the multihead network shown in \figref{fig:dnn:mh}, where we maximize the log-likelihood function \eqref{eq:likelihood:side:def} to learn the corresponding feature $f$ and weight matrices $G_s$ for all $s \in \cS$. Then, we convert the weights to  $g \in \spc{\cS \times \cY}$ via the correspondence [cf.   \eqref{eq:def:G:b:side}] $g(s, y) = G_s(1, y)$. The features learned by our algorithm (labeled as ``SideInfo'') and the multihead neural network are shown in \figref{fig:exp:multihead}, where the results are consistent.

\subsection{Multimodal Learning With Missing Modalities}
\label{sec:exp:mm}
To verify the multimodal learning algorithms presented in \secref{sec:mm}, we consider multimodal classification problems in two different settings. Suppose $X_1, X_2$ are multimodal data variables, and $Y \in \cY = \{-1, 1\}$ denotes the binary label to predict. In the first setting, we consider the training set with  complete $(X_1, X_2, Y)$ samples. In the second setting, only the pairwise observations of $(X_1, X_2)$, $(X_1, Y)$, and $(X_2, Y)$
 are available, presented in three different datasets. 

In both settings, we set
$\cX_1 = \cX_2 = [-1, 1]$ with
\begin{align}
  P_{X_1, X_2}(x_1, x_2) = \frac14 \cdot\left[ 1 + \cos(2\pi(x_1 - x_2))\right].
  \label{eq:cos:mm:dist}
\end{align}
We consider predicting $Y$ based on the learned features, where some modality in $X_1$, $X_2$ might be missing during the prediction. 
\subsubsection{Learning from Complete Observations}
\label{sec:exp:mm:comp}
We consider the $X_1, X_2, Y$ dependence specified by \eqref{eq:cos:mm:dist} and the conditional distribution
\begin{align}
  P_{Y|X_1, X_2}(y|x_1, x_2) = \frac12 + \frac{y}{4} \cdot \left[
\cos(\pi x_1) + \cos(\pi x_2)  +  \cos(\pi(x_1 + x_2))
\right]
\label{eq:cos:mm:cond}
\end{align}
for $x_1, x_2 \in [-1, 1]$ and $y = \pm 1$. It can be verified that $P_Y$ satisfies $P_Y(1) = P_Y(-1) = \frac12$. The corresponding CDK function and dependence components
[cf. \eqref{eq:pib:pii}]
 are given by
 \begin{subequations}
   \label{eq:exp:cos:pmi}
   \begin{align}
     \lpmi_{X_1, X_2; Y}(x_1, x_2, y)
     &= \frac{y}{2} \cdot \left[
       \cos(\pi x_1) + \cos(\pi x_2)  +  \cos(\pi(x_1 + x_2))
       \right]\\
     [\pib(\lpmi_{X_1, X_2; Y})](x_1, x_2, y) 
     &= \frac{y}{2} \cdot \left[
       \cos(\pi x_1) + \cos(\pi x_2)\right], 
     \label{eq:exp:cos:pmi:2}\\
     [\pii(\lpmi_{X_1, X_2; Y})](x_1, x_2, y)
     &= \frac{y}{2} \cdot
       \cos(\pi (x_1  +  x_2)).
   \end{align}
\end{subequations}
Therefore, we have $\rank(\pib(\lpmi_{X_1, X_2; Y})) = 
\rank(\pii(\lpmi_{X_1, X_2; Y})) = 1$, and the functions obtained from modal decompositions [cf. \eqref{eq:mm:fb:f}] are $\gb_1^*(y) = g_1^*(y) = y$ and
\begin{align}
  \fb_1^*(x_1, x_2) = \cos(\pi x_1) + \cos(\pi x_2),\qquad f_1^*(x_1, x_2) = \sqrt{2}\cdot \cos(\pi (x_1  +  x_2)).
  \label{eq:cos:fb:f}
\end{align}

\begin{figure}[!ht]
  \centering  \subfloat[$X_1, X_2$ Histogram ]{\raisebox{.13\textwidth}{\includegraphics[width = .24\textwidth]{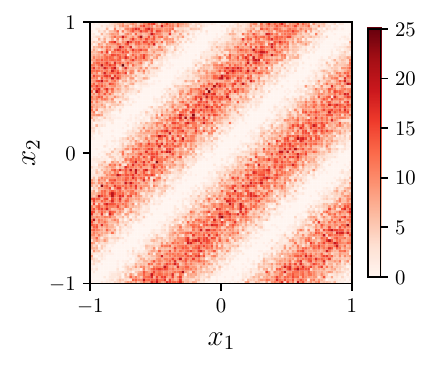}}
\label{fig:exp:mm:cosine:hist}
}
  \subfloat[$\fb_1^*(x_1, x_2)$]{\includegraphics[width = .25\textwidth]{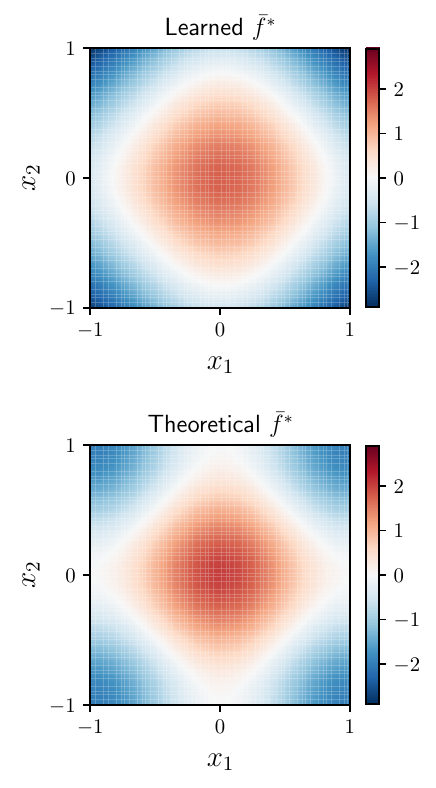}}
  \subfloat[$f_1^*(x_1, x_2)$]{\includegraphics[width = .25\textwidth]{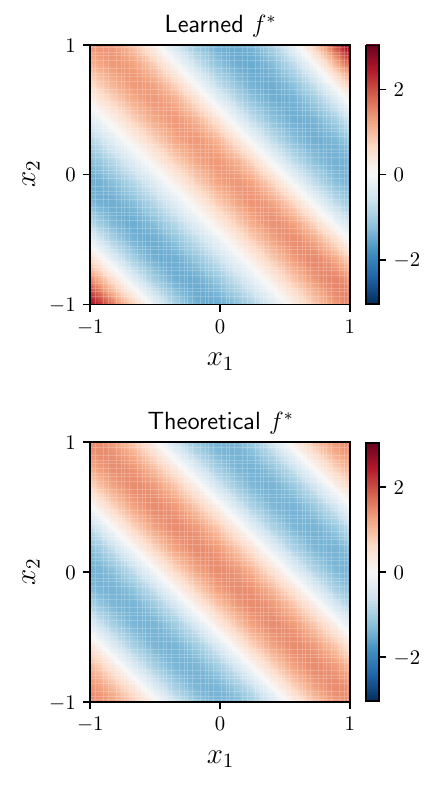}}
  \subfloat[$P_{Y|X_1, X_2}(1|x_1, x_2)$]{\includegraphics[width = .25\textwidth]{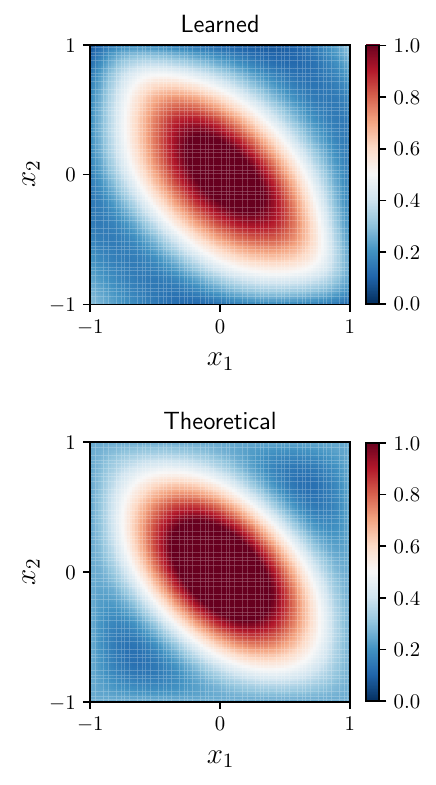}}
  \caption{Features and posterior probability learned from multimodal data $X_1, X_2, Y$, in comparison with theoretical results.}
  \label{fig:exp:mm:cosine}
\end{figure}

In the experiment, we first generate $N = 50\,000$ triples of $(X_1, X_2, Y)$ for training. The histogram of $(X_1, X_2)$ pair is shown in \figref{fig:exp:mm:cosine:hist}. Then, we set $\kb = k = 1$, and learn the features $\fb, f, \gb, g$ by maximizing the nested H-score configured by $\cbi$ [cf. \eqref{eq:cbi:def}]. We then normalize $\fb$, $f$ to obtain the estimated $\fb_1^*$ and $f_1^*$, and compute the posterior probability $P_{Y|X_1, X_2}(1|x_1, x_2)$ based on \eqref{eq:p:mm:12}. The results of learned features $\fb_1^*$, $f_1^*$ and posterior $P_{Y|X_1, X_2}(1|x_1, x_2)$ are shown in \figref{fig:exp:mm:cosine}, which are consistent with the theoretical results.

We then consider the prediction problem with missing modality, i.e., predict label $Y$ based on unimodal data $X_1$ or $X_2$. In particular, based on learned $\fb = \fb^{(1)} + \fb^{(2)}$, we train two separate networks to operate as $\opxn{1}$ and $\opxn{2}$, then apply \eqref{eq:p:mm:1} and \eqref{eq:p:mm:2} to estimate the posteriors $P_{Y|X_1}$ and $P_{Y|X_2}$.
Then, for each $i = 1, 2$, the MAP prediction of $Y$ based on observed $X_i = x_i$ can be obtained by comparing $P_{Y|X_i}(1|x_i)$ with the threshold $1/2$, via
\begin{align*}
  \argmax_{y \in \cY} P_{Y|X_i}(y|x_i)
  =
  \begin{cases}
    1& \text{if $P_{Y|X_i}(1|x_i) > 1/2$},\\
    -1& \text{if $P_{Y|X_i}(1|x_i) \leq 1/2$}.
  \end{cases}
\end{align*}

\begin{figure}[!ht]
  \centering
  \subfloat[$P_{Y|X_1}(1|x_1)$]{\includegraphics[width = .4\textwidth]{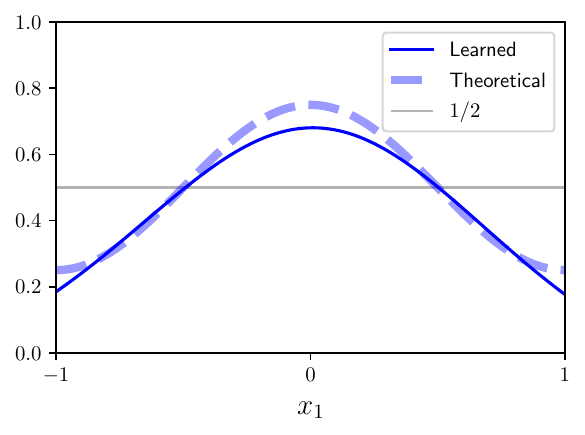}}
  \hspace{2em}
  \subfloat[$P_{Y|X_2}(1|x_2)$]{\includegraphics[width = .4\textwidth]{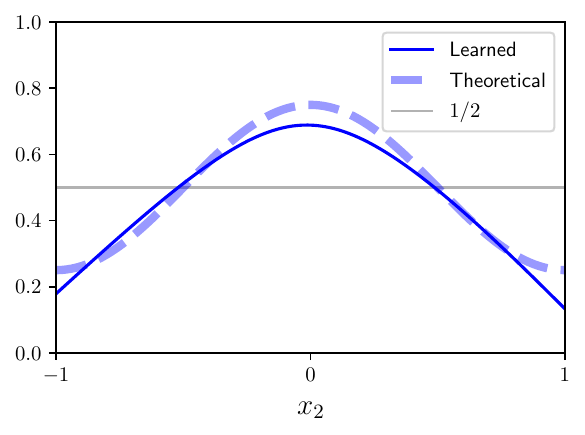}}
  \caption{Label prediction from unimodal data using learned features}
  \label{fig:exp:mm:x1x2}
\end{figure}

We plot the estimated results in \figref{fig:exp:mm:x1x2}, in comparison with the threshold $1/2$ and the theoretical values 
\begin{align}
  P_{Y|X_i}(y|x) = \frac12 + \frac{y}{4} \cdot
  \cos(\pi x),\quad \text{for $i = 1, 2$.}
  \label{eq:pygx:1:2}
\end{align}
From the figure, the estimated posteriors $P_{Y|X_1}, P_{Y|X_2}$ have consistent trends with the ground truth posteriors, and the induced $Y$ predictions are well aligned.

\subsubsection{Learning from Pairwise Observations}
\label{sec:exp:mm:pair}
We proceed to consider the multimodal learning with only pairwise observations. Specifically, we consider the joint distribution of $(X_1, X_2, Y)$, specified by  \eqref{eq:cos:mm:dist} and
\begin{align}
  P_{Y|X_1, X_2}(y|x_1, x_2) = \frac12 + \frac{y}{4} \cdot \left[
\cos(\pi x_1) + \cos(\pi x_2)
\right].
  \label{eq:cos:mm:cond:pair}
\end{align}
for $x_1, x_2 \in [-1, 1]$ and $y = \pm 1$. It can be verified that $P_Y(1) = P_Y(-1) = \frac12$, and the associated CDK function satisfies
\begin{align}
  \lpmi_{X_1, X_2; Y}(x_1, x_2, y)
  = \frac{y}{2} \cdot \left[
  \cos(\pi x_1) + \cos(\pi x_2)\right].
  \label{eq:exp:cos:pmi:pair}
\end{align}
and $\lpmi_{X_1, X_2; Y} = \pib(\lpmi_{X_1, X_2; Y})$. Therefore, the interaction dependence component $\pii(\lpmi_{X_1, X_2; Y}) = 0$, and the joint dependence can be learned from all pairwise samples, as we discussed in \secref{sec:mm:miss:pair}. 

We then construct an experiment to verify learning joint dependence from all pairwise observations. Specifically, we generate $N = 50\,000$ triples of $(X_1, X_2, Y)$ from \eqref{eq:cos:mm:dist} and \eqref{eq:cos:mm:cond:pair}. Then, we adopt the  decomposition \eqref{eq:dcmp:triple} on each triple, to obtain three pairwise datasets with samples of $(X_1, X_2)$, $(X_1, Y)$, $(X_2, Y)$, where each dataset has $N$ sample pairs.

\begin{figure}[!ht]
  \centering
  \subfloat[$\fb_1^*(x_1, x_2)$]{\includegraphics[width = .28\textwidth]{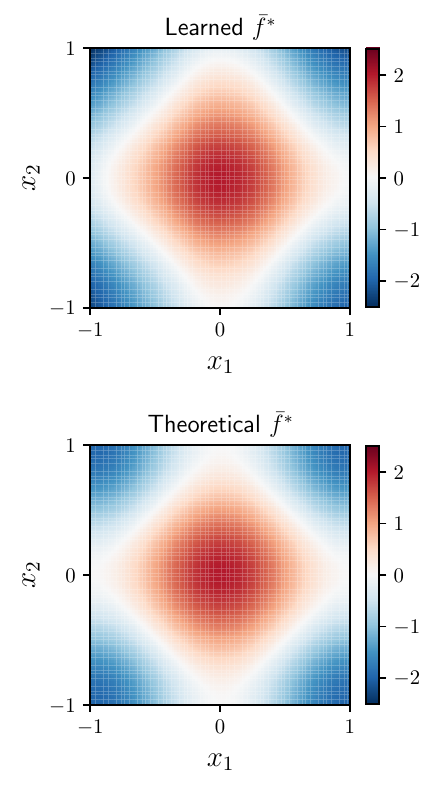}}
  \hspace{3em}
  \subfloat[$P_{Y|X_1, X_2}(1|x_1, x_2)$]{\includegraphics[width = .28\textwidth]{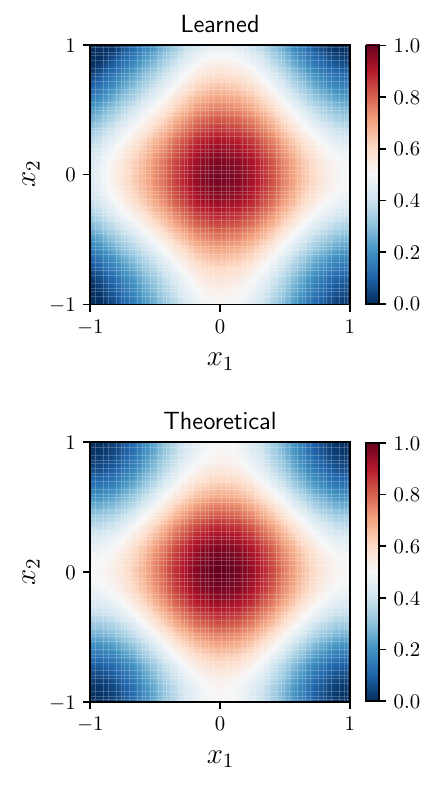}}
  \caption{Features and posterior probability learned from pairwise datasets of $(X_1, X_2), (X_1, Y)$, and $(X_2, Y)$, in comparison with theoretical results.}
  \label{fig:exp:mm:p:cosine}
\end{figure}

We use these three pairwise datasets for training and learn one dimensional $\fb \in \spc{\cX}, \gb \in \spc{\cY}$ that maximize $\Hs(\fb, \gb)$. Here, we compute $\Hs(\fb, \gb)$ based on the minibatches from the three pairwise datasets, according to \eqref{eq:Hs:fb:gb}. Based on learned $\fb, \gb$, we then compute the normalized $\fb_1^*$ and posterior distribution $P_{Y|X_1, X_2}(1|x_1, x_2)$, as shown in \figref{fig:exp:mm:p:cosine}, where the learned results match the theoretical values. 

\begin{figure}[!ht]
  \centering
  \subfloat[$P_{Y|X_1}(1|x_1)$]{\includegraphics[width = .4\textwidth]{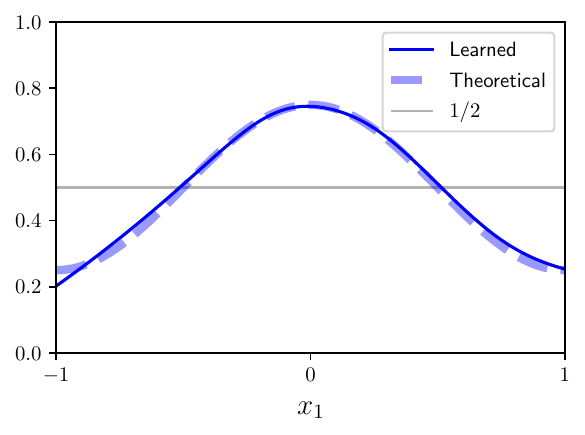}}
  \hspace{2em}
  \subfloat[$P_{Y|X_2}(1|x_2)$]{\includegraphics[width = .4\textwidth]{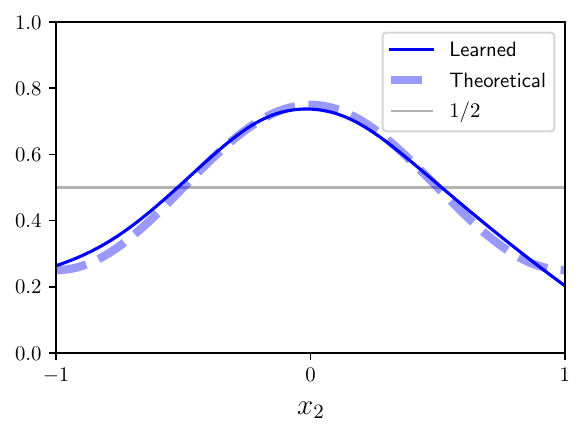}}
  \caption{Label prediction from unimodal data, with features learned from pairwise datasets.}
  \label{fig:exp:mm:p:x1x2}
\end{figure}

Similar to the previous setting, we consider the unimodal prediction problem, and show the estimated results in \figref{fig:exp:mm:p:x1x2}. It is worth noting that from \eqref{eq:exp:cos:pmi:2}, \eqref{eq:exp:cos:pmi:pair}, 
 the joint distributions $P_{X_1, X_2, Y}$ in both settings contain the same bivariate dependence component. Therefore, the theoretical results for $\fb_1^*$ and $P_{Y|X_1}, P_{Y|X_2}$ are the same.

\section{Related Works}
\label{sec:relate}
\paragraph{Maximal Correlation Functions: Optimality, Learning Algorithms, and Applications}
The Hirschfeld--Gebelein--R\'{e}nyi (HGR) maximal correlation \citep{hirschfeld1935connection, gebelein1941statistische, renyi1959measures} provides an important connection between statistical dependence and function space. The same concept has been studied with various formulations, and often in different terminologies, such as correspondence analysis \citep{greenacre2017correspondence}, functional canonical variates \citep{buja1990remarks}, and principal inertial components \citep{du2017principal}. See, e.g., \citep[Section II]{HuangMWZ2024}, for a detailed discussion. The optimality of HGR maximal correlation functions has been extensively studied in literature, e.g., as the optimal transformations in regression problems \citep{breiman1985estimating, buja1990remarks}. In particular, if one variable is categorical and represents the class label, the maximal correlation functions also provide the best inter-class separability \citep[Section III-B]{xu2020maximal}. The maximal correlation functions also play a fundamental role in the discussion of informative feature extraction and local information geometry, as we will detail next.
The first practical algorithm for learning such functions is the alternating conditional expectations (ACE) algorithm \citep{breiman1985estimating, buja1990remarks}, which learns the top dependence mode and is mostly used for processing low-dimensional data. 
{To tackle high-dimensional data, recent developments focused on learning maximal correlation functions by DNNs.} Specifically, the H-score was derived from the low-rank approximation of the canonical dependence kernel, referred as the Soft-HGR objective \citep{wang2019efficient}. It was also introduced as the local approximation of the log loss function \citep{xu2022information}, as we have discussed in \secref{sec:h:dnn}. In an independent line of work \citep{haochen2021provable}, a special form of H-score was proposed for self-supervised learning tasks, referred to as the \emph{spectral contrastive loss} therein. {There are also other formulations and algorithms proposed for computing the maximal correlation functions; see, e.g., \cite{andrew2013deep} and \cite{hsu2021generalizing}. Compared to the low-rank  approximation formulation, such approaches typically require explicit whitening procedures or computations of matrix inverses, which are less applicable for learning high-dimensional features. The sample complexity for learning maximal correlation functions have been investigated recently; see, e.g., \cite{huang2020sample, makur2020estimation}. The maximal correlation learning, particularly the H-score maximization formulation, has been widely applied in feature extraction for multimodal data, e.g., \cite{wang2019efficient} and the follow-up works. For example, \cite{xu2020maximal} developed maximal correlation regression (MCR), which uses maximal correlation functions to solve classification tasks.}

\paragraph{Informative Feature Extraction and Local Information Geometry}%

\cite{HuangMWZ2024} has provided an in-depth characterization of informative features by applying information-theoretic tools, with a focus on bivariate learning problems and the local analysis regime. Under such settings, it was shown that a series of statistical and learning formulations lead to the same optimal features, characterized as the HGR maximal correlation functions. {Examples of these formulations include the information bottleneck \citep{tishby2000information} and a cooperative game. There are similar information-theoretic discussions on multivariate problems. For example, \cite{xu2022multivariate} presented the information-theoretic optimality of the features defined in \eqref{eq:dcmp:pic}. The features introduced in \eqref{eq:mm:fb} were also studied by \cite{xu2021information} in characterizing distributed hypothesis testing problems.}

\paragraph{Decomposition of Probability Distributions}
The decomposition of probability distributions has been studied, particularly in the context of information geometry \citep{amari2000methods}. For example, \cite{amari2001information} established a decomposition in distribution space and also investigated the maximum entropy formulation \eqref{eq:pbi:def}, cf. \citep[Theorem 7]{amari2001information}. {The information geometry induced by  classification DNNs, with softmax activation and log loss function, was investigated in \cite{xu2018geometric}, which corresponds to the optimal features  without the weak dependence assumption [cf. \propref{prop:dnn}].  }

\paragraph{}  %

{The learning framework we presented is built upon these existing works and the perspectives from information theory, statistical analysis, and machine learning. The key component is a unified function-space view of feature learning design, which has provided nontrivial extensions to existing works. In particular, with the function-space viewpoint and nesting technique, we have developed several operations on general multivariate statistical dependence, with the existing bivariate learning algorithm \citep{wang2019efficient, xu2020maximal} being atomic operations and special cases. Our developments have also revealed the learning implications of existing theoretical developments, e.g., \cite{amari2000methods} and \cite{HuangMWZ2024}.
 }

\section{Conclusions and Discussion} %
\label{sec:conclusion}
We have presented a framework for designing learning systems with neural feature extractors, which allows us to learn informative features and assemble them to build different inference models. Based on the feature geometry, we 
 use feature representations as the interface, and convert learning problems to corresponding function-space operations. We then introduce the nesting technique for implementing such operations, which provides a systematic design of both feature learning  and feature assembling modules. We demonstrate its applications in learning multivariate dependence by considering conditional inference and multimodal learning problems.

{The developed framework has provided a unified solution to general feature-centric learning problems, where we have used DNNs as building blocks to directly represent the statistical dependence. 
The connection between feature and statistical dependence enables principled learning algorithm designs, especially in tackling complicated multivariate dependence. 
Such designs also provide a learning-based computational approach for investigating the statistical dependence behind high-dimensional data with often complicated structures.
}

\subsection{{Applications and Extensions}} %
{For simplicity, our presentation focused on  theoretical concepts and basic settings, which can be readily applied to practical learning scenarios and extended to more complicated settings. For instance, we can establish the feature optimality in multivariate problems by extending the bivariate characterizations \citep{HuangMWZ2024, xu2020maximal}; see, e.g, the discussions in \cite{xu2022multivariate}.

For applications, the results in \propref{prop:pred:est} can be applied to classification and estimation tasks, e.g., image classification \citep{xu2020maximal}. Similarly, we can employ the results in \secref{sec:nest:orth} to address physical or engineering constraints induced by practical learning applications. We can also apply the learning framework to process multimodal data dependencies and learn useful features, as discussed in \secref{sec:side} and \secref{sec:mm}. In all such examples, we can incorporate pre-trained models to reduce data requirements and computational costs.

Our discussions on the dependence decomposition framework can also be extended to general settings. In particular, we can get refined dependence components by iteratively applying the decomposition \eqref{eq:sspc:dcmp}. One example is to decompose the dependence of a random process into different dependence components, with each component representing the dependence at a certain time delay, as discussed in \cite{xu2023sequential}. More generally, the nesting technique can be extended by considering a configuration (cf. \defref{def:nest}) with subspace sequences of both $\spc{\cX}$ and $\spc{\cY}$.

We can also extend our discussions and analyses on supervised learning examples to other scenarios, e.g., contrastive \citep{chen2020simple} and self-supervised \citep{haochen2021provable} learning problems. In addition to the neural feature extractors, the feature geometry can also be applied to implicit features, e.g., the feature subspaces in kernel methods. In particular, \cite{xu2023kernel} developed a quantitative measure of kernel choices and demonstrated a connection between kernel methods and existing feature learning approaches.
}

\subsection{{Future Directions}} %
{
The established framework provides abundant opportunities for further explorations in both theoretical analyses and applications. To highlight the general framework, our development has adopted simplified or idealized assumptions on learning factors, e.g., the network expressive power.  By relaxing such idealized assumptions, further investigations can give a better theoretical understanding and also provide practical guidance.
Preliminary examples include our discussions on the expressive power of feature extractors (cf.  \secref{sec:h:const}) and
the sample size of training data (cf. \secref{sec:mm:miss:gen}).
More in-depth characterizations on the learning behaviors, e.g., generalizability, can be built upon the spectrum decomposition nature of this learning framework, by using existing analyses on linear operators \citep{dunford1988linear} or the spectral methods for matrices \citep{chen2021spectral}.
Another interesting direction is to investigate the interaction between feature geometry, the manifold structure of the neural feature extractor, and the geometry of data spaces \citep{bronstein2017geometric, bronstein2021geometric}. Moreover, the nested H-score and the induced optimization behaviors can be of independent interest to optimization studies, e.g., the landscape and convergence analyses.
In terms of applications, the established framework can be extended to other operations on statistical dependence beyond decomposing or learning given dependence. For instance, a potential application is to generate data satisfying certain dependence constraints, e.g., independence or conditional independence, by employing generative models. The integration  with generative models, as well as the learning algorithm design, is of practical interest for further studies. 
}

\acks{This work was supported by the Office of Naval Research (ONR) under grant N00014-19-1-2621.}

\appendix

\section{Data Alphabets in Feature Geometry}%
\label{app:notation}
{We discuss the feature geometry on general data alphabets, e.g., continuous alphabets in \appref{app:notation:con}.}
In \appref{app:notation:finite}, we briefly summarize the vector and matrix notation conventions established for discrete alphabets, which have been extensively used in related works, e.g., \cite{HuangMWZ2024, xu2022information}.

\subsection{{General Alphabets: Regularity Conditions and Examples}}
\label{app:notation:con}

{
Our development on feature spaces can be extended to general Hilbert spaces  \citep{young1988introduction, weidmann2012linear}, where the alphabet $\cZ$ and the metric distribution $P_Z$ (cf. \secref{sec:feature-space}) are extended to a measurable set and the measure \cite[Example 6]{weidmann2012linear}, respectively.

One particularly important example is $\cZ = \mathbb{R}^{d}$ with $d \geq 1$. We consider 
the random variable $\vec{Z} \in \cZ$. For simplicity, we restrict to the case where  
$\vec{Z}$ has the probability density function $p_{\vec{Z}}$. Then, we define the feature space $\spc{\cZ}$ and the inner product on $\spc{\cZ}$ as [cf. \secref{sec:feature-space}]
\begin{align*}
  \spc{\cZ} \defeq \left\{f \colon \cZ \to \mathbb{R},  \int f^2(\vec{z}) p_{\vec{Z}}(\vec{z}) d\vec{z} < \infty\right\}, \quad\text{and}\quad \ip{f_1}{f_2} \defeq \int f_1(\vec{z})f_2(\vec{z}) p_{\vec{Z}}(\vec{z}) d\vec{z} \text{ for }f_1, f_2 \in \spc{\cZ},
\end{align*}
respectively. %

We can define the joint function similarly for  $\cX$, $\cY$ under the corresponding regularity conditions. For instance, we consider $(X, Y)$
with alphabets $\cX = \cY = \mathbb{R}$ and 
the joint probability density function $p_{X, Y}$. Then, the corresponding definition of the CDK function becomes [cf. \eqref{eq:cdk:def}]
\begin{align}
  \lpmi_{X; Y}(x, y) = \frac{p_{X, Y}(x, y)}{p_X(x)p_Y(y)} - 1.
\end{align}

Note that to have $\lpmi_{X; Y} \in \spc{\cX \times \cY}$, we need to assume
\begin{align}
  \norm{\lpmi_{X; Y}}^2 = \iint \frac{[p_{X, Y}(x, y)]^2}{p_X(x)p_Y(y)} \,dx dy - 1 < \infty,
\end{align}
where the norm is defined with respect to the metric distribution $p_Xp_Y$.

\begin{remark}
The results can be further generalized to the case where density functions do not necessarily exist. See, e.g., \cite{lancaster1958structure} and \cite[Proposition 3.1]{buja1990remarks} for 
 detailed discussions.
\end{remark}
}

{

In particular, we have the following result for bivariate normal variables, which has been extensively discussed in literature; see, e.g., \cite{lancaster1958structure} and \cite{HuangMWZ2024}.

\begin{example}
  For bivariate normal variables
  $
  \begin{bmatrix}
    X\\Y
  \end{bmatrix}
  \sim
  \cN\left(  \begin{bmatrix}
      0\\0
    \end{bmatrix},
    \begin{bmatrix}
      1 & \rho\\
      \rho& 1
    \end{bmatrix}\right)
  $
  with $|\rho| < 1$, we have
  \begin{align}
    \lpmi_{X; Y}(x, y) = \sum_{i = 1}^\infty \frac{\rho^i}{i!} \cdot \He_i(x) \cdot \He_i(y).
    \label{eq:mehler}
  \end{align}
  where for each $i \geq 1$, $\He_i$ denotes the $i$-th (probabilist's) Hermite polynomial, defined as
  \begin{align*}
    \He_i(x) \defeq (-1)^{i} \cdot e^{\frac {x^{2}}{2}}{\frac {d^{i}}{dx^{i}}}e^{-{\frac {x^{2}}{2}}}.
  \end{align*}
  Then, each $i$-th mode of $\lpmi_{X; Y}$ can be written in the standard form as  $\mdn{i}(\lpmi_{X; Y}) = \sigma_i (f_i^* \otimes g_i^*)$, with
  \begin{align}
    \sigma_i = |\rho|^i,\quad
    f_i^* = \frac{1}{\sqrt{i!}} \cdot \He_i,\quad
    g_i^* = \frac{[\sign(\rho)]^i}{\sqrt{i!}} \cdot \He_i.
  \end{align}
  Finally, $\ds\norm{\lpmi_{X; Y}}^2 = \sum_{i \geq 1} \sigma_i^2 = \frac{\rho^2}{1-\rho^2}$.
\end{example}

The equation \eqref{eq:mehler} is referred to as Mehler’s identity or Mehler’s decomposition \citep{mehler1866ueber}.

In some scenarios, $\cX$ and $\cY$ can be of different types, e.g., $\cX$ is continuous, and $\cY$ is discrete. A classical example is mixture models, where $X$ and $Y$ corresponds to the observation variable and the latent categorical variable, respectively. Specifically, we introduce the following Gaussian mixture example, which is a multidimensional extension of an example discussed in \cite{buja1990remarks}.

\begin{example}
  \label{ex:mix:gaussian}
  We consider the probability model with $\cY = \{1, -1\}$, $P_Y \equiv \frac12$, and $\vec{X}|Y = y \sim \cN(y\cdot \vec{\mu}, \Lac)$ for $y \in \cY$, where $\vec{\mu} \in \mathbb{R}^d$ and $\Lac$ is a positive definite matrix of order $d$, for some $d \geq 1$. Then, we have $\rank( \lpmi_{\vec{X}; Y}) = 1$ with the standard form $\lpmi_{\vec{X}; Y} = \sigma_1 (f_1^*\otimes g_1^*)$, where  $g_1^*(y) = y$ and $f_1^*(\vec{x}) = c \cdot \tanh( \vec{\mu}^\T\Lac^{-1} \vec{x})$ for some $c \in \mathbb{R}$.
\end{example}

\begin{proof}
  We have
  \begin{align*}
    \lpmi_{\vec{X}; Y}(\vec{x}, y) 
    &=  \frac{p_{\vec{X}|Y}(\vec{x}| y)}{p_{\vec{X}}(\vec{x})} - 1\\
    &= 2\cdot \frac{\exp(- \frac12(\vec{x} - y \vec{\mu})^\T \Lac^{-1} (\vec{x} - y \vec{\mu}))}{\exp(- \frac12(\vec{x} + \vec{\mu})^\T \Lac^{-1} (\vec{x} + \vec{\mu})) + \exp(- \frac12(\vec{x} -  \vec{\mu})^\T \Lac^{-1} (\vec{x} -  \vec{\mu}))} - 1\\
    &=  \frac{2\exp(y \vec{\mu}^\T \Lac^{-1} \vec{x})}{\exp( \vec{\mu}^\T \Lac^{-1} \vec{x}) + \exp(- \vec{\mu}^\T \Lac^{-1} \vec{x})} - 1\\
    &=  y \cdot \frac{\exp( \vec{\mu}^\T \Lac^{-1} \vec{x}) - \exp(- \vec{\mu}^\T \Lac^{-1} \vec{x})}{\exp( \vec{\mu}^\T \Lac^{-1} \vec{x}) + \exp(- \vec{\mu}^\T \Lac^{-1} \vec{x})}\\
    &=  y \cdot \tanh(\vec{\mu}^\T \Lac^{-1} \vec{x}),
  \end{align*}
    which implies that $\rank( \lpmi_{\vec{X}; Y}) = 1$, $f_1^*(\vec{x}) \propto \tanh(\vec{\mu}^\T \Lac^{-1} \vec{x})$, and $g_1^*(y) = y$.
\end{proof}

From \exref{ex:mix:gaussian},
 the maximal correlation function $f_1^*$ can be represented
by a $d\times 1$ linear layer activated by $\tanh(\cdot)$, with zero bias and weight $\Lac^{-1}\vec{\mu}$.

}

\subsection{Finite Alphabets:  Information Vector and Canonical Dependence Matrix}%
\label{app:notation:finite}
For finite data alphabets, it can be convenient to introduce the vector and matrix representations of features. To begin, we assume the random variable $Z$ takes finite many possible values, i.e.,  $|\cZ| < \infty$, then the resulting feature space $\spc{\cZ}$ is a finite-dimensional vector space.  It is sometimes more convenient to represent features using vector and matrix notations. Specifically, each $f \in \spc{\cZ}$ can be equivalently expressed as one vector in $\mathbb{R}^{|\cZ|}$ as follows. Suppose $R_Z$ is the metric distribution, then we can construct an orthonormal basis  $\left\{\cb{Z}{z}\colon z \in \cZ\right\}$ of $\spc{\cZ}$, where
\begin{align}
  \cb{Z}{z}(z') \defeq \frac{\delta_{zz'}}{\sqrt{R_Z(z)}}\quad \text{for all $z, z' \in \cZ$}.
  \label{eq:cano:b:def}
\end{align}
We refer to this basis as the \emph{canonical basis} of $\spc{\cZ}$. %
For all $f \in \spc{\cZ}$, we can represent $f$ as a linear combination of these basis functions, i.e.,
\begin{align*}
  f = \sum_{z' \in \cZ} \xi(z') \cdot \cb{Z}{z'},
\end{align*}
where the coefficient $\xi(z)$ for each $z \in \cZ$ is given by %
 \begin{align}
   \xi(z) = \fip{f}{\cb{Z}{z}} =  f(z)\sqrt{R_Z(z)}.
   \label{eq:f:xi}
 \end{align}

In particular, when $f$ is the density ratio  $\llrt{P_Z}$ for some $P_Z \in \spc{\cZ}$, the corresponding coefficient $\xi(z)$ for each $z \in \cZ$ is %
\begin{align}
  \xi(z) = \fip{\llrt{P_Z}}{\cb{Z}{z}} =  \frac{P_Z(z) - R_Z(z)}{\sqrt{R_Z(z)}}.
  \label{eq:xi:llrt}
\end{align}
This establishes a one-to-one correspondence between $\llrt{P_Z}$ (or $P_Z$) and the  vector $\vec{\xi} \defeq [\xi(z), z \in \cZ]^\T \in \mathbb{R}^{|\cZ|}$, which is referred to as the information vector associated with  $\llrt{P_Z}$ (or $P_Z$).

Similarly, for $X$ and $Y$ with $|\cX| < \infty, |\cY| < \infty$, we can represent the CDK function $\lpmi_{X; Y} \in \spc{\cX \times \cY}$ as
an $|\cX|\times |\cY|$  matrix $\Bt_{X;Y}$:
  \begin{align}
    \Bt_{X;Y}(x; y)
    &\defeq \frac{P_{X, Y}(x, y) - P_{X}(x)P_{Y}(y)}{\sqrt{P_{X}(x)}\sqrt{P_{Y}(y)}},
      \label{eq:bt:def}
  \end{align}
which is referred to as the canonical dependence matrix (CDM) of $X$ and $Y$. With the metric distribution $P_XP_Y$ on $\spc{\cX \times \cY}$,  each $\Bt_{X;Y}(x; y)$ is the coefficient associated with the basis function $\cb{X, Y}{x, y}$ [cf. \eqref{eq:cano:b:def}].

In addition, we have $\rank(\Bt_{X; Y}) = \rank(\lpmi_{X; Y})$. Suppose $\sigma_i (f_i^* \otimes g_i^*) = \mdn{i}(\lpmi_{X; Y})$, for each $1 \leq i \leq \rank(\lpmi_{X; Y})$, then $\sigma_i$ is the $i$-th singular vector of $\Bt_{X; Y}$, and the corresponding $i$-th left and right singular vector pair are given by $\vec{\xi}_i^X \in \mathbb{R}^{|\cX|}, \vec{\xi}_i^Y \in \mathbb{R}^{|\cY|}$, with [cf. \eqref{eq:f:xi}]
\begin{align}
  \xi_i^X(x) = \sqrt{P_X(x)} \cdot f_i^*(x), \quad
  \xi_i^Y(y) = \sqrt{P_Y(y)} \cdot g_i^*(y).
  \label{eq:xi:xy}
\end{align}
Therefore, for small-scale discrete data, we can use the connection \eqref{eq:xi:xy} to obtain the modal decomposition by solving the SVD of the corresponding CDM.

\section{Characterization of Common Optimal Solutions} 
\label{app:opt}
We consider optimization problems defined on a common domain $\cD$. Specifically, given $k$ functions $h_i \colon \cD \to \mathbb{R}$, $i = 1, \dots, k$, let us consider optimization problems 
\begin{align}
  \maximize_{u \in \cD} h_i(u), \quad i = 1, \dots, k.
  \label{eq:opt:h}
\end{align}
For each $i = 1, \dots, k$, we denote the optimal solution and optimal value for $i$-th problem by
\begin{align}
  \cD_i^* \defeq \argmax_{u \in \cD} h_i(u), \quad t_i^* \defeq \max_{u \in \cD} h_i(u), 
  \label{eq:cD:star}
\end{align}
respectively, and suppose $\cD_i^* \neq \varnothing, i = 1, \dots, k$. 

Then, the set $\cD^* \defeq \cap_{i = 1}^k \cD_i^*$ represents the collection of common optimal solutions for all $k$ optimization problems \eqref{eq:opt:h}. When such common solutions exist, i.e., $\cD^*$ is nonempty, we can obtain $\cD^*$ by a single optimization program, by using an objective that aggregates original $k$ objectives. 
We formalize the result as follows. %
\begin{proposition}
  \label{prop:joint:opt}
  If $\cD^* \neq \varnothing$, we have
    $\ds\cD^* = \argmax_{u \in \cD}\,  \Gamma(h_1(u), h_2(u), \dots, h_k(u)) $
  for every $\Gamma\colon \mathbb{R}^k \to \mathbb{R}$ 
that is strictly increasing in each argument.
\end{proposition}
\begin{proof}
Let $\cD^{**} \defeq \argmax_{u \in \cD}\,  h(u)$ with $h(u) \defeq \Gamma(h_1(u), h_2(u), \dots, h_k(u))$. Then the proposition is equivalent to $\cD^* = \cD^{**}$. We then establish $\cD^{*} \subset \cD^{**}$ and $\cD^{**} \subset \cD^{*}$, respectively. 

To prove $\cD^{*} \subset \cD^{**}$, take any $u^* \in \cD^*$. Then, for all $u \in \cD$, we have $h_i(u) \leq h_i(u^*)$, $i = 1, \dots, k$, which implies that
  \begin{align*}
    h(u) = \Gamma(h_1(u), \dots, h_k(u)) \leq \Gamma(h_1(u^*), \dots, h_k(u^*)) = h(u^*).
  \end{align*}
  Therefore, we have $u^* \in \cD^{**}$. Since $u^*$ can be arbitrarily chosen from $\cD^*$, we have $\cD^* \subset \cD^{**}$.

We then establish $\cD^{**} \subset \cD^*$, which is equivalent to
  \begin{align}
    (\cD \setminus \cD^{*}) \subset (\cD \setminus \cD^{**}).
    \label{eq:cD:subset}
  \end{align}
  
  Note that \eqref{eq:cD:subset} is trivially true if $\cD \setminus \cD^* = \varnothing$. Otherwise, take any $u' \in \cD \setminus \cD^*$. Then, we have  $h_i(u') \leq h_i(u^*)$ for all $i \in [k]$, and the strict inequality holds for at least one $i \in [k]$. This implies that
  \begin{align*}
    h(u') = \Gamma(h_1(u'), \dots,  h_k(u')) < \Gamma(h_1(u^*), \dots, h_k(u^*)) = h(u^*).
  \end{align*}
  Hence, $u' \notin \cD^{**}$, and thus $u' \in (\cD \setminus \cD^{**})$, which establishes \eqref{eq:cD:subset}.
\end{proof}

To apply \propref{prop:joint:opt}, the first step is to test the existence of common optimal solutions. A naive test is to solve all optimization problems \eqref{eq:opt:h} and then check the definition, which can be difficult in practice. Instead, we can consider a related multilevel optimization problem as follows. Let $\cD_0 \defeq \cD$, and for each $ i = 1, \dots, k$ as, we define $\cD_i$ and $t_i \in \mathbb{R}$ as
\begin{align}
  \cD_{i} \defeq \argmax_{u \in \cD_{\!i-1}}\, h_i(u), \quad t_i \defeq \max_{u \in \cD_{\!i-1}}\, h_i(u).
  \label{eq:cD}
\end{align}
Note that for each $i$, the $\cD_{i}$ solved by level $i$ optimization problem gives the elements in $\cD_{i}$ that maximize $h_i$, and $t_i$ denotes the corresponding optimal value. Therefore, $\cD_k$ can be obtained by sequentially solving $k$ optimization problems as defined in \eqref{eq:cD}.

Then, the following result provides an approach to test the existence of common optimal solutions. %
  \begin{proposition}
    \label{prop:opt:equiv}
    The following three statements are equivalent:
    \begin{enumerate}
    \item The optimization problems \eqref{eq:opt:h} have common optimal solutions, i.e., $\cD^* \neq \varnothing$;
    \item $t_i = t_i^*$ for all $i = 1, \dots, k$;
    \item $\cD_{i-1} \cap \cD_i^* \neq \varnothing$ for all $i = 1, \dots, k$.   
    \end{enumerate}
    In addition, if one of these statements holds, then we have $\cD_k = \cD^*$.
  \end{proposition}

\begin{proof}
  We establish the equivalence of the statements 1 to 3, by showing that ``1'' $\implies$ ``2'', ``2'' $\implies$ ``3'', and ``3'' $\implies$ ``1''.

  \paragraph{ ``1'' $\implies$ ``2''}
  Suppose ``1'' holds, and take any $u^* \in \cD^* = \cap_{i = 1}^k \cD_i^*$. We then establish ``2'' by induction. First, note that $u^* \in \cD_0$. For the induction step, we can show that for each $i = 1, \dots, k$, if $u^* \in \cD_{i-1}$, then $u^* \in \cD_i$ and $t_i^* = t_i$. Indeed, we have
  \begin{align*}
    t_i^* \geq t_i = \max_{u \in \cD_{i - 1}} h_i(u) \geq h_i(u^*) = t_i^*,
  \end{align*}
  where the first inequality follows from the fact that  %
  $\cD = \cD_0 \supset \cD_1 \supset \dots \supset \cD_k$, %
where the second inequality follows from the inductive assumption $u^* \in \cD_{i-1}$, 
 and where the last equality follows from  that $u^* \in \cD^* \subset \cD_i^*$.  %

  \paragraph{``2'' $\implies$ ``3''} For each $i = 2, \dots, k$, $t_i = t_i^*$ implies that there exists some $u_i \in \cD_{i - 1}$, such that $h_i(u_i) = t_i = t_i^* = \max_{u \in \cD} h_i(u)$, and thus $u_i \in \cD_i^*$. Therefore, $u_i \in \cD_{i -1} \cap \cD_i^*$, which establishes ``3''.

  \paragraph{``3'' $\implies$ ``1''} For each $i = 2, \dots, k$, from $\cD_{i-1} \cap \cD_i^* \neq \varnothing$ and the definitions \eqref{eq:cD:star} and \eqref{eq:cD}, we have $\cD_i = \cD_{i-1} \cap \cD_i^*$. It can also be verified that $\cD_i = \cD_{i-1} \cap \cD_i^*$ holds for $i = 1$. Therefore, we obtain
  \begin{align}
    \cD_k = \cD_{k - 1} \cap \cD_k^* = \cD_{k - 2} \cap \cD_{k-1}^* \cap \cD_k^* = \dots = \cD_0 \cap \left(\bigcap_{i = 1}^k \cD_i^*\right) = \cD \cap \cD^* = \cD^*.
    \label{eq:cD:cDs}
  \end{align}
  This implies that $\cD^* = \cD_{k -1} \cap \cD_k^* \neq \varnothing$.  

  Finally, from \eqref{eq:cD:cDs} we know that statement 3 implies $\cD_k = \cD^*$. Since all three statements are equivalent, each statement implies $\cD_k = \cD^*$.
\end{proof}

\section{Proofs}
\subsection{Proof of \propref{prop:md:constraint}}
\label{app:prop:md:constraint}
  Let $\gammah \defeq \proj{\gamma}{\sspc{\cX} \otimes \sspc{\cY}}$.  We first consider the second equality of \eqref{eq:lora:constraint}, i.e., 
    \begin{align}
      \mdnl{k}(\gammah) = \argmin_{\substack{\gamma'\colon \gamma' = f \otimes g,\\
    f \in \sspcn{\cX}{k},\, g \in \sspcn{\cY}{k}}} \norm{ \gamma - \gamma'}
      \label{eq:lora:gammah}
    \end{align}
  For each $k \leq \rank(\gammah)$, consider $ \gamma' = f \otimes g$ with
    $f \in \sspcn{\cX}{k}, g \in \sspcn{\cY}{k}$. Then, we have
\begin{align}
  \norm{ \gamma - \gamma'}^2
  =  \norm{ \gamma - \gammah + \gammah - \gamma'}^2
  =  \norm{ \gamma - \gammah}^2 + \norm{\gammah - \gamma'}^2
  \geq \norm{ \gamma - \gammah}^2 + \norm{r_k(\gammah)}^2,
  \label{eq:gamma:gammah}
\end{align}
where to obtain the second equality we have used the orthogonality principle with fact that $(\gamma - \gammah) \perp \sspc{\cX} \otimes \sspc{\cY}$ and 
$(\gammah - \gamma') \in \sspc{\cX} \otimes \sspc{\cY}$. Note that the inequality in \eqref{eq:gamma:gammah} holds with equality if and only if $\gamma' = \mdnl{k}(\gammah)$, which establishes \eqref{eq:lora:gammah}.

We then establish $\mdn{k}(\gamma|\sspc{\cX}, \sspc{\cY}) = \mdn{k}(\gammah)$ by induction. 
To begin, set $k = 1$ in \eqref{eq:lora:gammah}, and the right hand side becomes $\md(\gamma|\sspc{\cX}, \sspc{\cY})$, which implies
 $\md(\gamma|\sspc{\cX}, \sspc{\cY}) = \md(\gammah)$, i.e.,  $\mdn{1}(\gamma|\sspc{\cX}, \sspc{\cY}) = \mdn{1}(\gammah)$. As the inductive hypothesis, suppose
we have
\begin{align}
  \mdn{i}(\gamma|\sspc{\cX}, \sspc{\cY}) = \mdn{i}(\gammah),  \quad i = 1, \dots, m.
  \label{eq:ind:assumption}
\end{align}
From \eqref{eq:mdn:constrain:def}, we have
\begin{align}
  \mdn{m+1}(\gamma|\sspc{\cX}, \sspc{\cY})
  &= \md\left(\gamma - \sum_{i = 1}^{m}\mdn{i}(\gamma|\sspc{\cX}, \sspc{\cY}) \middle|\sspc{\cX}, \sspc{\cY}\right)\notag\\
  &=
\md\left(\gamma - \sum_{i = 1}^{m}\mdn{i}(\gammah) \middle|\sspc{\cX}, \sspc{\cY}\right)\label{eq:mdn:1}\\
  &= \md\left(
\proj{\gamma - \sum_{i = 1}^{m}\mdn{i}(\gammah)}{\sspc{\cX} \otimes \sspc{\cY}}
    \right)\label{eq:mdn:2}\\
  &= \md\left(
\gammah - \sum_{i = 1}^{m}\proj{\mdn{i}(\gammah)}{\sspc{\cX} \otimes \sspc{\cY}}
    \right)\label{eq:mdn:3}\\
  &= \md\left(
\gammah - \sum_{i = 1}^{m}\mdn{i}(\gammah)
    \right) = \mdn{m+1}(\gammah),
\end{align}
where \eqref{eq:mdn:1}--\eqref{eq:mdn:2} follow from the inductive assumption \eqref{eq:ind:assumption}, where \eqref{eq:mdn:3} follows from the linearity of projection operator.
To obtain the first equality of
\eqref{eq:mdn:3}, we have again applied the assumption \eqref {eq:ind:assumption}: for $ i = 1, \dots, m$,
$\mdn{i}(\gammah) = \mdn{i}(\gamma|\sspc{\cX}, \sspc{\cY}) \in \sspc{\cX} \otimes \sspc{\cY}$.

Finally, $\mdnl{k}(\gamma|\sspc{\cX}, \sspc{\cY}) = \mdnl{k}(\gammah)$ can be readily obtained by definition.\hfill\qed

\subsection{Proof of \propref{prop:max:corr:const}}
\label{app:prop:max:corr:const}
It is easy to verify that $\corr(\fh_i^*, \gh_i^*) = 
\bip{\lpmi_{X; Y}}{\fh_i^* \otimes \gh_i^*} 
= \norm{\mdn{i}(\lpmi'_{X; Y})} = \sigmah_i^*$, where $\lpmi'_{X; Y} = \proj{\lpmi_{X; Y}}{\sspc{\cX}\otimes \sspc{\cY}}$.  %

 From \factref{fact:spec}, we have
{$ (\fh_i^*, \gh_i^*) = \argmax_{(f, g) \in \cDh_i'}\, \bip{\lpmi'_{X; Y}}{f\otimes g}
$ for each $i \geq 1$, where we have defined each $\cDh_i'$ as
\begin{align*}
  \cDh_i' \defeq \{(f, g) \in  \spc{\cX} \times \spc{\cY}\colon \norm{f} = \norm{g} = 1 \text{ and } \ip{f}{\fh_j^*} = \ip{g}{\gh_j^*} = 0\text{ for all }j \in [i - 1]\}.
\end{align*}
Therefore, for each $i$ and $(f, g) \in \cDh_i \subset \cDh_i'$, we have
\begin{align*}
  \corr(f_i, g_i)
  = \ip{\lpmi_{X; Y}}{f_i \otimes g_i}
  = \bip{\lpmi'_{X; Y}}{f_i \otimes g_i} \leq \sigmah_i^* = \corr(\fh_i^*, \gh_i^*),
\end{align*}
where the second equality follows from that $f_i \otimes g_i \in \sspc{\cX}\otimes \sspc{\cY}$. Hence, we obtain
$(\fh_i^*, \gh_i^*) = \argmax_{(f, g) \in \cDh_i} \corr(f, g)$.\hfill\qed
}

\subsection{Proof of \propref{prop:pred:est}}
\label{app:prop:pred:est}
The result of $\norm{\lpmi_{X; Y}}$ directly follows from \proptyref{propty:exp}. In addition,
  \begin{align*}
    f^\T(x) g(Y) = \lpmi_{X; Y}(x, y) = \frac{P_{X, Y}(x, y) - P_X(x)P_Y(y)}{P_{X}(x)P_{Y}(y)}, \quad \text{for all }(x, y)\in \cX \times \cY,
  \end{align*}
  which implies $P_{Y|X}(y|x) =  P_Y(y) \left(1 + f^\T(x) g(y)\right)$, i.e., \eqref{eq:p:ygx}.
  
  Therefore, for all $\psi \in \spcn{\cY}{d}$, we have
  \begin{align*}
    \E{\psi(Y)|X = x} = \sum_{y \in \cY}P_{Y|X}(y|x)\psi(y)
                        &= \sum_{y \in \cY}P_{Y}(y)(1 + f^\T(x)g(y)) \psi(y)\\
                        &= \sum_{y \in \cY}P_{Y}(y)\psi(y) + \sum_{y \in \cY}P_{Y}(y)(\psi(y)g^\T(y)) f(x)\\
                        &= \E{\psi(Y)} +  \La_{\psi, g}  f(x),
  \end{align*}
  which gives \eqref{eq:estimation}. \hfill\qed

\subsection{Proof of \propref{prop:est:k}}
\label{app:prop:est:k}
We have
    \begin{align*}
    \E{\psi(Y)|X = x} &= \sum_{y \in \cY}P_{Y|X}(y|x)\psi(y)\\
                        &= \sum_{y \in \cY}P_{Y}(y)(1 + \lpmi_{X; Y}(x, y)) \psi(y)\\
                        &= \sum_{y \in \cY}P_{Y}(y)[1 + \mdnl{k}(\lpmi_{X; Y})](x, y) \psi(y) + 
\sum_{y \in \cY}P_{Y}(y)\sum_{i > k} \sigma_i^* f_i^*(x) g_i^*(y) \psi(y)
      \\
      &=\sum_{y \in \cY}P_{Y}(y)(1 + f^\T(x)g(y)) \psi(y)\\
                        &= \sum_{y \in \cY}P_{Y}(y)\psi(y) + \sum_{y \in \cY}P_{Y}(y)(\psi(y)g^\T(y)) f(x)\\
                        &= \E{\psi(Y)} +  \La_{\psi, g}  f(x),
  \end{align*}
where the fourth equality follows from the fact that
\begin{align*}
  \sum_{y \in \cY}P_{Y}(y)\sum_{i > k} \sigma_i^* f_i^*(x) g_i^*(y) \psi(y)
=  \sum_{i > k} \sigma_i^* f_i^*(x) \E{g_i^*(Y) \psi(Y)} = 0,
\end{align*}
due to $\psi_j \in \{\cst{\cY}, g_{1}^*, \dots, g_k^*\}$ for each $j \in [d]$.\hfill\qed

\subsection{Proof of \proptyref{propty:dnn:center}}
\label{app:propty:dnn:center}
The property is equivalent to  $\llog(f, g) = \llog(f + \vec{u}, g + \vec{v})$ for all $\vec{u}, \vec{v} \in \mathbb{R}^k$. Hence, it suffices to prove that 
\begin{align}
  \llog(f + \vec{u}, g + \vec{v}) \leq \llog(f, g), \qquad \text{for all $\vec{u}, \vec{v} \in \mathbb{R}^k$}.
\end{align}

Note that for all $b \in \spc{\cY}$, since
\begin{align}
    \left(f(x) + \vec{u}\right)^\T \left(g(y) + \vec{v}\right) + b(y)
  = f(x) \cdot g(y) + \vec{v}^\T f(x) +
  \vec{u}^\T g(y) + b(y) + \vec{u}^\T \vec{v},
\end{align}
we have [cf. \eqref{eq:Pt}]  $\Pt_{Y|X}^{(f + \vec{u}, g + \vec{v}, b)} = 
  \Pt_{Y|X}^{(f, g, b + \vec{u}^\T g)}$,
which implies that 
 $ \llog(f + \vec{u}, g + \vec{v}, b) = \llog(f, g, b + \vec{u}^\T g)$.

Therefore, we obtain
\begin{align}
  \llog(f + \vec{u}, g + \vec{v}) 
  = \max_{b \in \spc{\cY}}\llog(f + \vec{u}, g + \vec{v}, b)
  \leq \max_{b \in \spc{\cY}}\llog(f, g, b + \vec{u}^\T g) \leq \llog(f, g).
\end{align}
\hfill\qed

\subsection{Proof of \propref{prop:dnn}}
\label{app:prop:dnn}
  We first prove a useful lemma.
  \begin{lemma}
    \label{lem:p:q}
    Suppose $p > 0, q> 0$, $p + q = 1$, then we have
    \begin{align*}
      \log \left(p \cdot \exp \left(\frac{u}{p}\right) + q \cdot \exp \left(-\frac{u}{q}\right)\right)
      \geq
      \min\left\{
      u^2, u_0 |u|
      \right\}, \quad \text{for all $u \in \mathbb{R}$},
    \end{align*}
    where $u_0 \defeq \frac{\ln 2}{3} \cdot \min\{p, q\}$.
  \end{lemma}
  \begin{proof}[Proof of \lemref{lem:p:q}]
    Let $p_{\min} \defeq \min\{p, q\}$.
    $h(u)\defeq   \log \left(p \cdot \exp \left(p^{-1}u\right) + q \cdot \exp \left(-q^{-1}u\right)\right).$
    Then, we have
    \begin{align*}
      h'(u) =  \frac{\exp \left(p^{-1}u\right) -  \exp \left(-q^{-1}u\right)}{ p \cdot \exp \left(p^{-1}u\right) + q \cdot \exp \left(-q^{-1}u\right)}
    \end{align*}
    \begin{align*}
      h''(u) &=   \left[p \cdot \exp \left(p^{-1}u\right) + q \cdot \exp \left(-q^{-1}u\right)\right]^{-2}\\
             &\quad\cdot \left[\left(\frac1p \exp \left(\frac{u}p\right) + \frac1q \exp \left(-\frac{u}{q}\right)\right)\cdot  \left(p \exp \left(\frac{u}p\right) + q \exp \left(-\frac{u}q\right)\right)
               - \left[\exp \left(\frac{u}p\right)
               -  \exp \left(-\frac{u}q\right)\right]^2\right]\\
             &\geq \frac{\left[\exp \left(p^{-1}u\right) +  \exp \left(-q^{-1}u\right)\right]^2- \left[\exp \left(p^{-1}u\right) -  \exp \left(-q^{-1}u\right)\right]^2}{ \left[p \cdot \exp \left(p^{-1}u\right) + q \cdot \exp \left(-q^{-1}u\right)\right]^2}\\
             &= 4 \cdot \frac{\exp\left((p^{-1} - q^{-1}) \cdot u\right)}{ \left[p \cdot \exp \left(p^{-1}u\right) + q \cdot \exp \left(-q^{-1}u\right)\right]^2}.
    \end{align*}

    Moreover, for all $|u| \leq u_0$, we have 
    \begin{gather*}
      \exp\left((p^{-1} - q^{-1}) \cdot u\right) \geq 
      \exp\left(- |p^{-1} - q^{-1}| \cdot u_0\right)
      \geq  \exp\left(- \frac{u_0}{p_{\min}}\right)\\
      p \cdot \exp \left(p^{-1}u\right) + q \cdot \exp \left(-q^{-1}u\right) \leq \exp\left(\frac{u_0}{p_{\min}}\right).
    \end{gather*}
    As a result, for all $|u| \leq u_0$, we have 
      $h''(u) \geq 4 \exp\left(- \frac{3u_0}{p_{\min}}\right) = 2$.

    Therefore,
      $h'(u_0) \geq h'(0) +  2 \cdot (u_0 - 0) = 2 u_0 > u_0$, and similarly, $-h'(-u_0) = h'(0) - h'(-u_0)
    \geq 2 u_0$,
    i.e., 
      $h'(-u_0) \leq - 2u_0 \leq -u_0$.
    Moreover, for all $|u| \leq u_0$,
      $h(u) \geq h(0) + h'(0) \cdot u + \frac{1}{2} u^2 \cdot 2
      = u^2.$
    Therefore, for all $u > u_0$, 
      $h'(u) \geq h'(u_0) > u_0$,
    which implies that
      $h(u) \geq h(u_0) + u_0(u - u_0) \geq u_0^2 + u_0u - u_0^2 = u_0u$.
    Similarly, we have $h(u) \geq -u_0u$ for all $u < -u_0$.
  \end{proof}
  Proceeding to the proof of \propref{prop:dnn}, we consider zero-mean $k$-dimensional $f$, $g$. Without loss of generality, we assume $b \in \spc{\cY}$ satisfies $\E{b(Y)} = -H(Y)$. Then, let $a \in \spct{\cY}$ be $a(y) \defeq b(y) - \log P_Y(y)$, and define 
 $\gamma \in \spc{\cX \times \cY}$ as $\gamma(x, y) \defeq f(x) \cdot g(y) + a(y)$.
 Note that since
    $$\exp(f(x) \cdot g(y) + b(y)) = P_Y(y)\exp(f(x) \cdot g(y) + a(y)) = P_Y(y) \exp(\gamma(x, y)), $$
  we have
  \begin{align*}
    \Pt^{(f, g, b)}_{Y|X}(y|x)
    = \frac{\exp(f(x) \cdot g(y) + b(y))}{\sum_{y' \in \cY}\exp(f(x) \cdot g(y') + b(y'))}
    &= \frac{P_Y(y)\exp(\gamma(x, y))}{\sum_{y' \in \cY}P_Y(y')\exp(\gamma(x, y'))}.
  \end{align*}

  Therefore,
  \begin{align}
    &\llog(f, g, b)\notag\\
    &\quad= \Ed{(\Xh, \Yh)\sim P_{X}P_{Y}}{(\lpmi_{X;Y}(\Xh, \Yh) + 1) \cdot \log\Pt^{(f, g, b)}_{Y|X}(\Yh|\Xh)}\notag\\
    &\quad= \Ed{(\Xh, \Yh)\sim P_{X}P_{Y}}{(\lpmi_{X;Y}(\Xh, \Yh) + 1) \cdot \left(\log P_Y(\Yh) + \gamma(\Xh, \Yh) - \log \sum_{y' \in \cY}P_Y(y')\exp(\gamma(\Xh, y'))\right)}\notag\\
    &\quad= -H(Y) + \ip{\lpmi_{X; Y}}{\gamma} - \Ed{\Xh \sim P_X}{\log \sum_{y' \in \cY}P_Y(y')\exp(\gamma(\Xh, y'))}.
      \label{eq:likelihood}
  \end{align}

  As a result, for all $(f, g, b)$ that satisfies
$\llog(f, g, b) \geq -H(Y)$, we have
  \begin{align}
    \ip{\lpmi_{X; Y}}{\gamma} \geq \Ed{\Xh \sim P_X}{\log \sum_{y' \in \cY}P_Y(y')\exp(\gamma(\Xh, y'))}.
    \label{eq:ip:lb}
  \end{align}
  In addition, note that for all $x \in \cX$ and $y \in \cY$, we have
  \begin{align*}
    \sum_{y' \in \cY}P_Y(y')\exp(\gamma(x, y'))
    &= P_Y(y) \exp(\gamma(x, y)) + (1 - P_Y(y)) 
\sum_{y' \in \cY \setminus \{y\}}\frac{P_Y(y')}{1 - P_Y(y)}\exp(\gamma(x, y'))\\
    &\geq P_Y(y) \exp(\gamma(x, y)) + (1 - P_Y(y)) \exp\left(- \frac{P_Y(y)}{1 - P_Y(y)} \cdot \gamma(x, y)\right).
  \end{align*}
  where the inequality follows from Jensen's inequality and $\sum_{y \in \cY} P_Y(y)\gamma(x, y) = 0$.

  Let us define
  \begin{align*}
     q_X \defeq \min_{x \in \cX} P_X(x) > 0, \quad q_Y \defeq \min_{y \in \cY} P_Y(y) > 0.
  \end{align*}
  Then, from \lemref{lem:p:q} we have
  \begin{align*}
    \log \sum_{y' \in \cY}P_Y(y')\exp(\gamma(x, y'))
    &\geq \log \left[P_Y(y) \exp(\gamma(x, y)) + (1 - P_Y(y)) \exp\left(- \frac{P_Y(y)}{1 - P_Y(y)} \cdot \gamma(x, y)\right)\right]\\
    &\geq \min\left\{
      (P_Y(y)\gamma(x, y))^2, \frac{\ln 2 }{3} \cdot q_Y |P_Y(y)\gamma(x, y)|
      \right\}\\
    &\geq  \frac{\ln 2 \cdot q_Y^2 }{3} \cdot \min\left\{
      (\gamma(x, y))^2,  |\gamma(x, y)|\right\},
  \end{align*}
  which implies that
  \begin{align}
    \Ed{\Xh \sim P_X}{\log \sum_{y' \in \cY}P_Y(y')\exp(\gamma(\Xh, y'))}
    &\geq P_{X}(x) \cdot \log \sum_{y' \in \cY}P_Y(y')\exp(\gamma(x, y'))\notag\\
    &\geq \frac{\ln 2 \cdot q_X q_Y^2 }{3} \cdot \min\left\{
      (\gamma(x, y))^2,  |\gamma(x, y)|\right\}.
      \label{eq:phi:lb}
  \end{align}
  
  On the other hand, since $\norm{\lpmi_{X; Y}} = O(\eps)$, there exists a constant $C > 0$ such that $\norm{\lpmi_{X; Y}} \leq C \eps$. Therefore,
  \begin{align}
    \ip{\lpmi_{X; Y}}{\gamma} \leq \norm{\lpmi_{X; Y}} \cdot \norm{\gamma} \leq \gamma_{\max} \cdot \eps,
    \label{eq:ip:ub}
  \end{align}
  where $\gamma_{\max} \defeq \max_{x\in\cX, y\in \cY}|\gamma(x, y)|$.

  Hence, by combining \eqref{eq:ip:lb}, \eqref{eq:phi:lb} and \eqref{eq:ip:ub}, we obtain
  \begin{align}
    \frac{\ln 2 \cdot q_X q_Y^2 }{3} \cdot \min\left\{
    \gamma_{\max}^2,  \gamma_{\max}\right\} \leq C \cdot  \gamma_{\max} \cdot \eps,
    \label{eq:phi:max}
  \end{align}
  where we have taken $(x, y) = \argmax_{x', y'} |\gamma(x', y')|$ in \eqref{eq:phi:lb}. 

  From \eqref{eq:phi:max}, if $\eps < \frac{\ln 2}{3} \cdot \frac{q_X q_Y^2 }{3C}$
  we have $\gamma_{\max} < \eps$. This implies that
  \begin{align}
    |\gamma(x, y)| < \eps,\quad\text{for all }x\in \cX, y \in \cY,
    \label{eq:local:phi}
  \end{align}
  and we obtain $\norm{\gamma} < \eps$. In addition, note that since
    $\norm{\gamma}^2 =  \norm{f \otimes g + a}^2
    = \norm{f \otimes g}^2 + \norm{a}^2$,
   we obtain $\norm{f \otimes g} = O(\eps)$. %

 From \eqref{eq:local:phi}, we have
 \begin{align*}
   \sum_{y' \in \cY}P_Y(y')\exp(\gamma(x, y'))
   &=  \sum_{y' \in \cY}P_Y(y')\left( 1 + \gamma(x, y') + \frac{(\gamma(x, y'))^2}{2} + o(\eps^2)\right)\\
   &= 1 + \frac12\sum_{y' \in \cY}P_Y(y') (\gamma(x, y'))^2 + o(\eps^2)\\
   &= 1 + \frac12\cdot\Ed{\Yh \sim P_Y}{\gamma(x, \Yh)}^2 + o(\eps^2).
 \end{align*}
Therefore, 
\begin{align*}
  \Ed{\Xh \sim P_X}{\log \sum_{y' \in \cY}P_Y(y')\exp(\gamma(\Xh, y'))}
  = \frac12\cdot \Ed{(\Xh, \Yh) \sim P_XP_Y}{\gamma(\Xh, \Yh)}^2 + o(\eps^2)
  = \frac12 \cdot \norm{\gamma}^2 + o(\eps^2),
\end{align*}
and the likelihood \eqref{eq:likelihood} becomes
  \begin{align}
    \llog(f, g, b) %
    &= -H(Y) + \ip{\lpmi_{X; Y}}{\gamma} - \Ed{\Xh \sim P_X}{\log \sum_{y' \in \cY}P_Y(y')\exp(\gamma(\Xh, y'))}\notag\\
    &= -H(Y) + \ip{\lpmi_{X; Y}}{\gamma} - \frac12 \cdot \norm{\gamma}^2 + o(\eps^2)\notag\\
    &= \frac{1}{2} \cdot \left(\norm{\lpmi_{X; Y}}^2 - \norm{\lpmi_{X; Y} - \gamma}^2\right) -H(Y) + o(\eps^2)\notag\\
    &= \frac{1}{2} \cdot \left(\norm{\lpmi_{X; Y}}^2 - \norm{\lpmi_{X; Y} - f \otimes g - a }^2\right) -H(Y) + o(\eps^2)\notag\\
    &= \frac{1}{2} \cdot \left(\norm{\lpmi_{X; Y}}^2 - \norm{\lpmi_{X; Y} - f \otimes g}^2 - \norm{a }^2\right) -H(Y) + o(\eps^2),
      \label{eq:l:fgb} 
  \end{align}
where  the last equality follows from the fact that
  \begin{align*}
    \norm{\lpmi_{X; Y} - f \otimes g - a }^2 
    &=  \norm{\lpmi_{X; Y} - f \otimes g}^2 + \norm{a }^2,
  \end{align*}
due to the orthogonality $(\lpmi_{X; Y} - f \otimes g) \perp \spc{\cY} \ni a$.

Finally, from \eqref{eq:l:fgb}, for given $f, g$, $\llog(f, g, b)$ is maximized when $\norm{a} = o(\eps)$. Therefore, we have
\begin{align*}
  \llog(f, g) = \max_{b \in \spc{\cY}}\llog(f, g, b) 
    &= \frac{1}{2} \cdot \left(\norm{\lpmi_{X; Y}}^2 - \norm{\lpmi_{X; Y} - f \otimes g}^2\right) -H(Y) + o(\eps^2),
\end{align*}
which gives \eqref{eq:l:fg}. \hfill\qed

  \subsection{Proof of \thmref{thm:sspc:opt} }
  \label{app:thm:sspc:opt}
  Throughout this proof, we consider
  \begin{align*}
    L_1(\fb, \gb, f, g) \defeq \bnorm{\lpmi_{X; Y} - \fb \otimes \gb}^2,\quad
L_2(\fb, \gb, f, g) \defeq \bnorm{\lpmi_{X; Y} - \fb \otimes \gb - f \otimes g}^2
  \end{align*}
  defined on the domain
     $\cD \defeq \{ (\fb, \gb, f, g) \colon \fb \in \sspcn{\cX}{\kb}, \gb \in \spcn{\cY}{\kb}, f  \in \spcn{\cX}{k}, g   \in \spcn{\cY}{k}\}$,
and let %
  \begin{align*}
    \cD_1^* \defeq \argmin_{(\fb, \gb, f, g) \in \cD} L_1(\fb, \gb, f, g),\quad
    \cD_2^* \defeq \argmin_{(\fb, \gb, f, g) \in \cD} L_2(\fb, \gb, f, g).
  \end{align*}
  We also define $\gammab, \gamma \in \spc{\cX\times \cY}$, as
  \begin{align}
    \gammab \defeq \proj{\lpmi_{X; Y}}{\sspc{\cX} \otimes \spc{\cY}},\quad
 \gamma \defeq \lpmi_{X; Y} - \gammab =  \proj{\lpmi_{X; Y}}{(\spc{\cX} \orthm \sspc{\cX}) \otimes \spc{\cY}}.
  \end{align}

  Then, it suffices to establish that
  \begin{align}
    \cD^* \defeq \cD_1^* \cap \cD_2^* =
    \{(\fb, \gb, f, g) \in \cD \colon 
  \fb \otimes \gb = \gammab,
  f \otimes g &= \mdnl{k}(\gamma)
    \}.
    \label{eq:cD:opt}
  \end{align}

  To see this, note that the common solution $\cD^*$ \eqref{eq:cD:opt} coincides with the optimal solution \eqref{eq:sspc:opt} to be established. From \propref{prop:joint:opt}, since $\cD^* \neq \varnothing$, we have
  \begin{align}
    \cD^* = \argmin_{(\fb, \gb, f, g) \in \cD} L_1(\fb, \gb, f, g) +  L_2(\fb, \gb, f, g).
    \label{eq:opt:L1L2}
  \end{align}
  Furthermore, from the definition of the H-score (cf. \defref{def:H}), we have
  \begin{align*}
    \Hs\left(
      \begin{bmatrix}
        \fb\,\\
        f\,
      \end{bmatrix}
    ,
    \begin{bmatrix}
      \gb\\
      g
    \end{bmatrix}; \cpi
    \right)
     &= \Hs(\fb, \gb) + \Hs\left(
      \begin{bmatrix}
        \fb\,\\
        f\,
      \end{bmatrix}
    ,
    \begin{bmatrix}
      \gb\\
      g
    \end{bmatrix}
    \right)\\
    &= \frac{1}{2} \cdot \left(\norm{\lpmi_{X; Y}}^2 - L_1(\fb, \gb, f, g)\right)
      + \frac{1}{2} \cdot \left(\norm{\lpmi_{X; Y}}^2 - L_2(\fb, \gb, f, g)\right)\\
    &=  \norm{\lpmi_{X; Y}}^2 - \frac12 \left(L_1(\fb, \gb, f, g) + L_2(\fb, \gb, f, g)\right),
  \end{align*}
  which implies that
  \begin{align*}
    \argmax_{(\fb, \gb, f, g) \in \cD} \Hs\left(
      \begin{bmatrix}
        \fb\,\\
        f\,
      \end{bmatrix}
    ,
    \begin{bmatrix}
      \gb\\
      g
    \end{bmatrix}; \cpi
    \right)  = 
    \argmin_{(\fb, \gb, f, g) \in \cD} L_1(\fb, \gb, f, g) + L_2(\fb, \gb, f, g) = \cD^*.
 \end{align*}

 It remains only to establish \eqref{eq:cD:opt}. From \propref{prop:opt:equiv}, it suffices to establish that $t_2 = t_2^*$ and (cf. \eqref{eq:cD:opt})
\begin{align}
  \cD_2 = \{(\fb, \gb, f, g) \in \cD \colon 
    \fb \otimes \gb = \gammab,
  f \otimes g = \gamma \},
  \label{eq:cD2:goal}
\end{align}
where we have defined   
 \begin{align*}
   t_2^* \defeq  \min_{(\fb, \gb, f, g) \in \cD} L_2(\fb, \gb, f, g),\quad t_2  \defeq \min_{(\fb, \gb, f, g)\in \cD_1^*} L_2(\fb, \gb, f, g),\quad   
   \cD_2 \defeq \argmin_{(\fb, \gb, f, g) \in \cD_1^*} L_2(\fb, \gb, f, g).
 \end{align*}
  To this end, let us define $\hb \in \sspcn{\cX}{k}$ and $\hb \in (\spc{\cX} \orthm \sspc{\cX})^k$ as $\hb \defeq \proj{f}{\sspc{\cX}}$, and $h \defeq f - \hb =  \proj{f}{\spc{\cX} \orthm \sspc{\cX}}$. Then, we have
  \begin{align}
    L_2(\fb, \gb, f, g) 
    &= \bnorm{\lpmi_{X; Y} - \fb \otimes \gb - f \otimes g}^2 \\
    &= \bnorm{\gammab
      +  \gamma
       - \fb \otimes \gb - (\hb + h) \otimes g}^2 \\
    &= \bnorm{\gammab
- \fb \otimes \gb - \hb \otimes g}^2 + 
\bnorm{\gamma   - h \otimes g}^2 \label{eq:ieq:L2:1}\\
    &\geq \bnorm{\gamma
      - h \otimes g}^2 \\
    &\geq \bnorm{r_k(\gamma)}^2 %
      \label{eq:ieq:L2}
  \end{align}
  where  \eqref{eq:ieq:L2:1} follows from the orthogonality principle, since
$\left(\gammab - \fb \otimes \gb - \hb \otimes g\right) \in \sspc{\cX} \otimes \spc{\cY}$ and $
\left(\gamma - h \otimes g\right)  \perp  \sspc{\cX} \otimes \spc{\cY}$.
Moreover, it is easily verified that the lower bound \eqref{eq:ieq:L2} is tight: all inequalities hold with equality for $(\fb, \gb, f, g)$ satisfying \eqref{eq:sspc:opt}.

Therefore, we have
\begin{align}
  t_2^* =  \min_{(\fb, \gb, f, g)\in \cD} L_2(\fb, \gb, f, g) =   \bnorm{r_k(\gamma)}^2
\end{align}

  On the other hand, since $\kb \geq \rank(\gammab)$, from \propref{prop:md:constraint}, we have
   $ \cD_1^* = \{(\fb, \gb, f, g)\colon \fb \otimes \gb = \gammab\}$. Hence, for all $(\fb, \gb, f, g) \in \cD_1^*$, we have
\begin{align*}
 L_2(\fb, \gb, f, g) =  
  \bnorm{\lpmi_{X; Y} - \fb \otimes \gb - f \otimes g}^2
=   \bnorm{\lpmi_{X; Y} - \gammab - f \otimes g}^2
=   \bnorm{\gamma - f \otimes g}^2
\geq   \bnorm{r_k(\gamma)}^2,
\end{align*}
where the inequality holds with equality if and only if $f \otimes g = \mdnl{k}(\gamma)$.

As a result,  $\cD_2  = \argmin_{(\fb, \gb, f, g)\in \cD_1^*} L_2(\fb, \gb, f, g) $ is given by \eqref{eq:cD2:goal}, and we have
\begin{gather*}
  t_2 = \min_{(\fb, \gb, f, g)\in \cD_1^*} L_2(\fb, \gb, f, g) = \bnorm{r_k(\gamma)}^2 = t_2^*.
\end{gather*}
\hfill\qed

\subsection{Proof of \thmref{thm:sspc:opt:n}}
\label{app:thm:sspc:opt:n}
To begin, we have
\begin{align*}
    \Hs\left(
    \begin{bmatrix}
      \fb\,\\
      f\,
    \end{bmatrix}
    ,
    \begin{bmatrix}
      \gb\\
      g
    \end{bmatrix}
    ; \cpir
  \right)
  &= \sum_{i = 1}^{\kb} \Hs(\fb_{[i]}, \gb_{[i]})
  + \sum_{i = 1}^{k}
    \Hs\left(
    \begin{bmatrix}
      \fb\\
      f_{[i]}
    \end{bmatrix}
    ,
    \begin{bmatrix}
      \gb\\
      g_{[i]}
    \end{bmatrix}
  \right)\\
  &= \sum_{i = 1}^{\kb-1} \Hs(\fb_{[i]}, \gb_{[i]})
    + \sum_{i = 1}^{k}
    \left(
    k^{-1} \cdot
    \Hs(\fb, \gb) + 
    \Hs\left(
    \begin{bmatrix}
      \fb\\
      f_{[i]}
    \end{bmatrix}
    ,
    \begin{bmatrix}
      \gb\\
      g_{[i]}
    \end{bmatrix}
  \right)\right).
\end{align*}
Note that for $\left( \begin{bmatrix}
      \fb\,\\
      f\,
    \end{bmatrix}
    ,
    \begin{bmatrix}
      \gb\\
      g
    \end{bmatrix}\right) \in \dom(\cpir)$ and each $i \in [\kb]$, $\Hs(\fb_{[i]}, \gb_{[i]})$ is maximized if and only if
  \begin{align}
    \fb_{[i]} \otimes \gb_{[i]} = \mdnl{i}(\lpmi_{X; Y}|\sspc{\cX} ,  \spc{\cY}).
    \label{eq:cpi:pf:1}
  \end{align}

In addition, from \thmref{thm:sspc:opt}, for each $i \in [k]$, the term
$\left( k^{-1} \cdot
    \Hs(\fb, \gb) + 
    \Hs\left(
    \begin{bmatrix}
      \fb\\
      f_{[i]}
    \end{bmatrix}
    ,
    \begin{bmatrix}
      \gb\\
      g_{[i]}
    \end{bmatrix}
  \right)\right)$ 
is maximized if and only if
\begin{align}
  \fb \otimes \gb = \proj{\lpmi_{X; Y}}{\sspc{\cX} \otimes \spc{\cY}},\qquad
  f_{[i]} \otimes g_{[i]} = \mdnl{i}(\proj{\lpmi_{X; Y}}{(\spc{\cX} \orthm \sspc{\cX}) \otimes \spc{\cY}}).
  \label{eq:cpi:pf:2}
\end{align}

Note that
\eqref{eq:sspc:opt:n} is the common solution of \eqref{eq:cpi:pf:1} and \eqref{eq:cpi:pf:2}. Hence, 
the proof is completed by applying \propref{prop:joint:opt}.
\hfill\qed

  \subsection{Proof of \propref{prop:pim:pic}}
  \label{app:prop:pim:pic}
  The relation $\lpmi_{X; S} = \llrt{\Pm_{X, S, Y}}$ can be directly verified from definition. To establish $\pim(\lpmi_{X; S, Y}) = \proj{\lpmi_{X; S, Y}}{\spc{\cX \times \cS}} = \lpmi_{X; S}$, 
  from \factref{fact:proj}, it suffices to show that %
  $(\lpmi_{X; S, Y} - \lpmi_{X; S}) \perp \spc{\cX \times \cS}$.

  To this end, note that 
  \begin{align*}
    \lpmi_{X; S, Y}(x, s, y) - \lpmi_{X; S}(x, s)
    &= \frac{P_{X, S, Y}(x, s, y) - \Pm_{X, S, Y}(x, s, y)}{R_{X, S, Y}(x, s, y)}\\
    &= \frac{P_{X, S, Y}(x, s, y) - P_{X|S}(x|s)P_{S}(s)P_{Y|S}(y|s)}{R_{X, S, Y}(x, s, y)}.
  \end{align*}  
  Therefore, for all $f \in \spc{\cX \times \cS}$, we have
  \begin{align*}
    \bip{\lpmi_{X; S, Y} - \lpmi_{X; S}}{f}%
      &= \sum_{x \in \cX, s \in \cS, y \in \cY} P_{X, S, Y}(x, s, y) f(x, s)%
      - \sum_{x \in \cX, s \in \cS, y \in \cY} P_{X|S}(x|s)P_{S}(s)P_{Y|S}(y|s)\cdot f(x, s)\\
    &= \Ed{P_{X, S}}{f(X, S)} - \Ed{P_{X, S}}{f(X, S)}\\
    &= 0.
  \end{align*}
  \hfill\qed

\subsection{Proof of \propref{prop:markov}}
\label{app:prop:markov}

Let $\gamma = \lpmi_{X; S, Y}^{(Q)}$, then the statement is equivalent to %
  \begin{align*}
     \gamma \in \spc{\cX \times \cS} \quad \iff \quad Q_{X, S, Y} = Q_{X|S}Q_{S}Q_{Y|S}.
  \end{align*}
  \paragraph{``$\implies$''} If $\gamma \in \spc{\cX \times \cS}$, then we have
  \begin{align}
    Q_{X, S, Y}(x, s, y)
    &= Q_X(x)Q_{S, Y}(s, y)\left(1 + \gamma(x, s)\right)\notag\\
    &= P_X(x)P_{S, Y}(s, y)\left(1 + \gamma(x, s)\right)\notag\\
    &= P_X(x)P_{S}(s)\left(1 + \gamma(x, s)\right) \cdot P_{Y|S}(y|s),\quad \text{for all $x, s, y$}.
      \label{eq:q:xsy}
  \end{align}
  Therefore,
  \begin{align}
    Q_{X, S}(x, s) = \sum_{y \in \cY}  Q_{X, S, Y}(x, s, y) = P_X(x)P_{S}(s)\left(1 + \gamma(x, s)\right).
      \label{eq:q:xs}
  \end{align}
  From \eqref{eq:q:xsy} and \eqref{eq:q:xs}, we obtain
  \begin{align}
    Q_{X, S, Y}(x, s, y) =  Q_{X, S}(x, s)P_{Y|S}(y|s) = Q_{X, S}(x, s)Q_{Y|S}(y|s) =  Q_{X|S}(x|s)Q_S(s)Q_{Y|S}(y|s).
    \label{eq:q:xsy:new}
  \end{align}

  \paragraph{``$\impliedby$''}
  It suffices to note that
  \begin{align*}
    \lpmi_{X; S, Y}^{(Q)}(x, s, y) 
    = \frac{Q_{X, S, Y}(x, s, y)}{Q_{X}(x)Q_{S, Y}(s, y)} - 1
    = \frac{Q_{X, S}(x, s)Q_{Y|S}(y|s)}{Q_{X}(x)Q_{S}(s) Q_{Y|S}(y|s)} - 1
    = \frac{Q_{X, S}(x, s)}{Q_{X}(x)Q_{S}(s)} - 1.
  \end{align*}
\hfill\qed

\subsection{Proof of \propref{prop:pred:est:side}}
\label{app:prop:pred:est:side}
The results on  $\norm{\lpmi_{X; S}}$ and $\norm{\lpmi_{X; Y|S}}$ directly follows from \proptyref{propty:exp}. To establish \eqref{eq:post:est}, note that from \propref{prop:pim:pic}, we have
\begin{align*}
    P_{X, S, Y}(x, s, y) -
   P_{X|S}(x|s)P_{S}(s)P_{Y|S}(y|s)%
    =  P_{X}(x)P_{S, Y}(s, y) \cdot  \lpmi_{X; Y|S}(x, s, y),
\end{align*}
which implies that
\begin{align}
    P_{Y|X, S}(y|x, s)
  &= P_{Y|S}(y|s)\cdot
    \left(1 +  \frac{P_X(x)P_S(s)}{P_{X, S}(x, s)}\cdot\lpmi_{X; Y|S}(x, s, y)\right)\notag\\
  &= P_{Y|S}(y|s)\cdot
    \left(1 +  \frac{1}{1 + \lpmi_{X;S}(x, s)}\cdot\lpmi_{X; Y|S}(x, s, y)\right)\label{eq:side:estimation:pf:1}\\
  &=P_{Y|S}(y|s) \cdot \left(1 + \frac{f^\T(x)  g(s, y)}{1 + \fb^{\,\T}(x)  \gb(s)} \right),                                              \label{eq:side:estimation:pf}
\end{align}
which further implies
 \eqref{eq:side:estimation}.\hfill\qed

\subsection{Proof of \propref{prop:posterior}}
\label{app:prop:posterior}
When $X$ and $(S, Y)$ are $\eps$-dependent, we have $\norm{\lpmi_{X; S, Y}} = O(\eps)$, and thus
\begin{align*}
  \lpmi_{X; S}(x, s) = O(\eps), \quad \lpmi_{X; Y|S}(x, s, y) = O(\eps).
\end{align*}
Therefore, we have
\begin{align*}
  \frac{1}{1 + \lpmi_{X; S}(x, s)} \cdot \lpmi_{X; Y|S}(x, s, y)
  = \left(1 - \lpmi_{X; S}(x, s)\right)\cdot \lpmi_{X; Y|S}(x, s, y) +  o(\eps)
  &= \lpmi_{X; Y|S}(x, s, y) +  o(\eps).
\end{align*}
Then, it follows from \eqref{eq:side:estimation:pf:1}  that
\begin{align*}
    P_{Y|X, S}(y|x, s)
= P_{Y|S}(y|s)\cdot
    \left(1 +\lpmi_{X; Y|S}(x, s, y)\right) + o(\eps).
\end{align*}

Finally, the proof is completed by using the decomposition \eqref{eq:dcmp:pic}.  \hfill\qed

\subsection{Proof of \thmref{thm:dnn:mulithead}}
\label{app:thm:dnn:mulithead}
For each $s, x, y$, let us define $R_{X, Y}^{(s)}(x, y) \defeq P_{X|S=s}(x) P_{Y|S=s}(y)$ and 
\begin{align}
  \lpmi^{(s)}_{X; Y}(x, y) 
  \defeq \frac{P_{X, Y|S = s}(x, y) - P_{X|S = s}(x) P_{Y|S = s}(y)}{P_{X|S = s}(x) P_{Y|S = s}(y)}.
\end{align}

In addition, we define $\vec{\mu}_s \in \mathbb{R}^k$ as $
\vec{\mu}_s \defeq \E{f(X)|S = s}$ for each $s \in \cS$. Also, let $\beta \in \spc{\cS \times \cY}$ be defined as $\beta(s, y) \defeq \vec{\mu}_s^\T g(s, y)$.

Then, for each $s \in \cS$, from
\eqref{eq:better:than:trivial:s} and \propref{prop:dnn} we have
\begin{align}
   |f^\T(x)g(s, y) - \beta(s, y)| = |(f(x) - \vec{\mu}_s)^\T g(s, y)| = O(\eps),\quad \text{for all $x, s, y$},
\label{eq:local:s:entry}
\end{align}
which implies that
\begin{align}
  \label{eq:fg:eps}
  \left|f^\T(x)g(s, y)\right| = O(\eps),\quad \text{for all $x, s, y$}.
\end{align}

In addition, from \proptyref{propty:dnn:center}, we have
\begin{align}
  &\llogsv{s}(f, g^{(s)})\notag\\
  &\qquad=
  \llogsv{s}(f - \vec{\mu}_s, g^{(s)})\notag\\
  &\qquad= \frac{1}{2} \cdot \left(\bnorm{\lpmi^{(s)}_{X; Y}}_{R_{X, Y}^{(s)}}^2 - \bnorm{\lpmi^{(s)}_{X; Y} - (f - \vec{\mu}_s) \otimes g^{(s)}}_{R_{X, Y}^{(s)}}^2
 \right)  - H(Y|S = s)  + o(\eps^2)\notag\\
 &\qquad=\bip{\lpmi^{(s)}_{X; Y}}{f \otimes g^{(s)}  - \vec{\mu}_s ^\T g^{(s)}}_{R_{X,Y}^{(s)}}
    -\frac{1}{2} \cdot \bnorm{f \otimes g^{(s)}   - \vec{\mu}_s^\T g^{(s)}}_{R_{X,Y}^{(s)}}^2 - H(Y|S = s) + o(\eps^2),
   \label{eq:log:s}
\end{align}
where $\ip{\,\cdot\,}{\,\cdot\,}_{R}$ and $\norm{\,\cdot\,}_{R}$ denote the inner product and corresponding induced norm on the function space, with respect to the metric distribution $R$.

For the first two terms in \eqref{eq:log:s}, we compute their expectations over $P_S$. For the first term, 
\begin{align}
  &\sum_{s \in \cS}P_S(s)\bip{\lpmi^{(s)}_{X; Y}}{f \otimes g^{(s)}   - \vec{\mu}_s ^\T g^{(s)}}_{R_{X,Y}^{(s)}}\notag\\  
  &\qquad=\sum_{x \in \cX, s \in \cS, y \in \cY}P_S(s) R_{X,Y}^{(s)}(x, y)\lpmi^{(s)}_{X; Y}(x, y) \left(f^\T(x) g^{(s)}(y) - \vec{\mu}_s ^\T g^{(s)}(y)\right)\notag\\
  &\qquad= \sum_{x \in \cX, s \in \cS, y \in \cY} P_X(x) P_{S, Y}(s, y)
    \lpmi_{X; Y|S}(x, s, y)  \cdot \left(f^\T(x) g(s, y) - \beta(s, y)\right)\notag\\
  &\qquad= \bip{\lpmi_{X; Y|S}}{f \otimes g} - 
\bip{\lpmi_{X; Y|S}}{\beta}\notag\\
  &\qquad= \bip{\lpmi_{X; Y|S}}{f \otimes g}
    \label{eq:e:s:1}
\end{align}
where to obtain the second equality we have used the facts that
\begin{align}
  \lpmi^{(s)}_{X; Y}(x, y) 
  &= \frac{P_{X, Y|S = s}(x, y) - P_{X|S = s}(x) P_{Y|S = s}(y)}{P_{X|S = s}(x) P_{Y|S = s}(y)}\notag\\
  &=\frac{P_{X, S, Y}(x, s, y) - \Pm_{X, S, Y}(x, s, y)}{P_X(x) P_{S, Y}(s, y)}\cdot \frac{P_X(x)}{P_{X|S = s}(x)}\notag\\
  &=\lpmi_{X; Y|S}(x, s, y) \cdot \frac{1}{1 + \lpmi_{X; S}(x, s)}
\end{align}
and
\begin{align}
  P_S(s) R_{X, Y}^{(s)}(x, y)
  = \Pm_{X, S, Y}(x, s, y)
  &= P_X(x)P_{S, Y}(s, y) \cdot (1 + \lpmi_{X;S}(x, s)),
    \label{eq:ps:r}
\end{align}
and where to obtain \eqref{eq:e:s:1} we have used the fact that $\lpmi_{X; Y|S} \perp \spc{\cS \times \cY} \ni \beta$.

For the second term of \eqref{eq:log:s}, we have
\begin{align}
&\sum_{s \in \cS}P_S(s)  \bnorm{f \otimes g^{(s)}    - \vec{\mu}_s^\T g^{(s)}}_{R_{X,Y}^{(s)}}^2\notag\\
&\qquad= \sum_{x \in \cX, s \in \cS, y \in \cY} \Pm_{X, S, Y}(x, s, y) \cdot
\left[(f(x) - \vec{\mu}_s)^\T g(s, y)\right]^2\\
&\qquad= \sum_{x \in \cX, s \in \cS, y \in \cY} P_{X}(x)P_{S, Y}(s, y) \cdot
\left[(f(x) - \vec{\mu}_s)^\T g(s, y)\right]^2 + o(\eps^2)\\
&\qquad= \sum_{x \in \cX, s \in \cS, y \in \cY} P_{X}(x)P_{S, Y}(s, y) \cdot
\left[f^\T(x)g(s, y) - \beta(s, y) \right]^2 + o(\eps^2)\\
&\qquad= \bbnorm{f \otimes g - \beta}^2 + o(\eps^2)\\
&\qquad= \bbnorm{f \otimes g}^2 + \bbnorm{\beta}^2 + o(\eps^2),
\label{eq:e:s:2}
\end{align}
where to obtain the second equality we have used
\eqref{eq:local:s:entry}, \eqref{eq:ps:r}, and the fact that $\norm{\lpmi_{X; S}} = O(\eps)$.
Furthermore, we can show that $\norm{\beta} = o(\eps)$. To see this, note that 
\begin{align*}
  \beta(s, y) = \vec{\mu}_s^\T g(s, y) &= \sum_{x \in \cX} P_{X|S}(x|s) f^\T(x)g(s, y)\\
              &= \sum_{x \in \cX} P_{X}(x) \cdot [ 1 + \lpmi_{X; S}(x, s) ] \cdot f^\T(x)g(s, y)\\
              &=\sum_{x \in \cX} P_{X}(x) \cdot \lpmi_{X; S}(x, s) \cdot f^\T(x)g(s, y).
\end{align*}
Then, from $\norm{\lpmi_{X; S}} = O(\eps)$ and \eqref{eq:fg:eps}, we obtain $|\beta(s, y)| = O(\eps^2)$ and $\norm{\beta} = O(\eps^2) = o(\eps)$.

Therefore, we can refine \eqref{eq:e:s:2} as
\begin{align}
  \sum_{s \in \cS}P_S(s)  \bnorm{f \otimes g^{(s)}    - \vec{\mu}_s^\T g^{(s)}}_{R_{X,Y}^{(s)}}^2
= \bbnorm{f \otimes g}^2 + o(\eps^2).
  \label{eq:e:s:2:n}
\end{align}

Combining
\eqref{eq:log:s},
\eqref{eq:e:s:1}, and \eqref{eq:e:s:2:n}, we have
\begin{align*}
  \llogs(f, g) = \sum_{s \in \cS} P_S(s)\cdot  \llogsv{s}(f, g^{(s)}) &=\fip{\lpmi_{X; Y|S}}{f\otimes g } - \frac12 \cdot\bnorm{f \otimes g}^2 - H(Y|S) + o(\eps^2)\\
&= \frac{1}{2} \cdot \left(\norm{\lpmi_{X; Y|S}}^2
 - \bnorm{\lpmi_{X; Y|S} - f \otimes g}^2
 \right)
  - H(Y|S) + o(\eps^2),
\end{align*}
which is maximized if and only if $f \otimes g = \mdnl{k}(\lpmi_{X; Y|S}) + o(\eps)$.
\hfill\qed

\subsection{Proof of \propref{prop:pib}}
\label{app:prop:pib}
Given any $h \in \spc{\cX_1 \times \cY} + \spc{\cX_2 \times \cY}$, we can represent $h = h^{(1)} + h^{(2)}$ with $h^{(i)} \in \spc{\cX_i \times \cY}, i = 1, 2$, i.e., $h(x_1, x_2, y) = h^{(1)}(x_1, y) + h^{(2)}(x_2, y)$.

Then, for any $Q_{X_1, X_2, Y} \in \cqb$, we have
\begin{align*}
  \ip{\lpmi^{(Q)}_{X_1, X_2; Y}}{h^{(i)}}
  &= \sum_{x_1 \in \cX_1, x_2 \in \cX_2, y \in \cY} P_{X_1, X_2}(x_1, x_2) P_Y(y) \lpmi^{(Q)}_{X_1, X_2; Y}(x_1, x_2, y) \cdot h^{(i)}(x_i, y) \\
  &= \sum_{x_1 \in \cX_1, x_2 \in \cX_2, y \in \cY} [Q_{X_1, X_2, Y}(x_1, x_2, y) - P_{X_1, X_2}(x_1, x_2) P_Y(y) ] \cdot h^{(i)}(x_i, y) \\
  &=\sum_{x_i \in \cX_i, y \in \cY} [P_{X_i, Y}(x_i, y) - P_{X_i}(x_i) P_Y(y) ] \cdot h^{(i)}(x_i, y)\\
  &=\ip{\lpmi_{X_i; Y}}{h^{(i)}} \qquad\qquad\text{for $i = 1, 2$.}
\end{align*}

As a result,
\begin{align*}
  \ip{\pib(\lpmi^{(Q)}_{X_1, X_2; Y})}{h}
  &=\ip{\lpmi^{(Q)}_{X_1, X_2; Y} - \pii(\lpmi^{(Q)}_{X_1, X_2; Y})}{h}
  \\
  &=\ip{\lpmi^{(Q)}_{X_1, X_2; Y}}{h}\\
  &=\ip{\lpmi^{(Q)}_{X_1, X_2; Y}}{h^{(1)}} + \ip{\lpmi^{(Q)}_{X_1, X_2; Y}}{h^{(2)}} \equiv \ip{\lpmi_{X_1; Y}}{h^{(1)}} + \ip{\lpmi_{X_2; Y}}{h^{(2)}},
\end{align*}
where the second equality follows from the fact that
$ \ip{\pii(\lpmi^{(Q)}_{X_1, X_2; Y})}{h} = 0$.

Hence, we obtain
$  \ip{\pib(\lpmi^{(Q)}_{X_1, X_2; Y}) - \pib(\llrt{P_{X_1, X_2, Y}})}{h} = 0$
for all $h \in  \spc{\cX_1 \times \cY} + \spc{\cX_2 \times \cY}$, which implies that $\pib(\lpmi^{(Q)}_{X_1, X_2; Y}) - \pib(\llrt{P_{X_1, X_2, Y}}) = 0$.\hfill\qed

\subsection{Proof of \propref{prop:pred:est:mm}}
\label{app:prop:pred:est:mm}

From the definition, we have
  \begin{align*}
  P_{X_1, X_2, Y}(x_1, x_2, y) 
    &= P_{X_1, X_2}(x_1, x_2)P_Y(y) \left(1 + 
      [\pib(\lpmi_{X_1, X_2; Y})](x_1, x_2, y) + [\pii(\lpmi_{X_1, X_2; Y})](x_1, x_2, y)\right)\\
    &= P_{X_1, X_2}(x_1, x_2)P_Y(y) \left[1 + 
\fb^{\,\T}(x_1, x_2)\gb(y) + f^\T(x_1, x_2)g(y)
\right],
\end{align*}
which implies \eqref{eq:p:mm:12}.

To obtain \eqref{eq:p:mm:1}, note that from \eqref{eq:pii:zero}, we have
\begin{align*}
  P_{X_1, Y}(x_1, y) 
  &= \sum_{x_2 \in \cX_2} \pbi_{X_1, X_2, Y}(x_1, x_2, y) \\
  &= P_{X_1}(x_1)P_Y(y) \left[1 + 
    \E{\fb^{\,\T}(x_1, X_2)\gb(y)|X_1 = x_1}
\right]\\
  &=P_{X_1}(x_1)P_Y(y) \left[1 + \left(\fb^{(1)}(x_1) + [\opxn{1}(\fb^{(2)})](x_1)\right)^\T \gb(y)\right],
\end{align*}
where to obtain the last equality we have used the fact that
$\opxn{1}(\fb) = \fb^{(1)} + \opxn{1}(\fb^{(2)})$.
Similarly, we can obtain \eqref{eq:p:mm:1}.

Finally, \eqref{eq:est:mm} can be readily obtained from \eqref{eq:p:mm}.\hfill\qed

\subsection{Proof of \propref{prop:pbi:pint}}
\label{app:prop:pbi:pint}  Note that since $\lpmi_{X_1, X_2; Y} \in \spct{\cX_1 \times \cX_2 \times \cY}$, we have
 $\pib(\lpmi_{X_1, X_2; Y}) \in \spct{\cX_1 \times \cX_2 \times \cY}$, which implies that
 \begin{align*}
   \sum_{(x_1, x_2, y)\in \cX_1 \times \cX_2 \times \cY}
P_{X_1, X_2}(x_1, x_2)P_Y(y)\cdot [\pib(\lpmi_{X_1, X_2; Y})](x_1, x_2, y) = \ip{\pib(\lpmi_{X_1, X_2; Y})}{1} = 0.
 \end{align*}
Therefore, it follows from the definition of $\pbi_{X_1, X_2, Y}$ that $\ds\sum_{(x_1, x_2, y)\in \cX_1 \times \cX_2 \times \cY}\pbi_{X_1, X_2, Y}(x_1, x_2, y) = 1$. Similarly, we have
$\ds\sum_{(x_1, x_2, y)\in \cX_1 \times \cX_2 \times \cY}\pint_{X_1, X_2, Y}(x_1, x_2, y) = 1$. 

Moreover, from \eqref{eq:gtr:n1}, we have
$\pbi_{X_1, X_2, Y}(x_1, x_2, y) \geq 0, \pint_{X_1, X_2, Y}(x_1, x_2, y) \geq 0$ for all $(x_1, x_2, y)$. Therefore, 
we obtain $\pbi_{X_1, X_2, Y}, \pint_{X_1, X_2, Y} \in \cP^{\cX_1 \times \cX_2 \times \cY}$.

Since $\lpmi_{X_1, X_2; Y} \perp \spc{\cX_1 \times \cX_2}$, we have $\pib(\lpmi_{X_1, X_2; Y}) \perp \spc{\cX_1 \times \cX_2}$. Therefore, for any $(\hat{x}_1, \hat{x}_2) \in \cX_1 \times \cX_2$, let $f(x_1, x_2) =  \delta_{x_1\hat{x}_1}
 \delta_{x_2\hat{x}_2}$, then we have
\begin{align*}
  \sum_{y \in \cY} [\pib(\lpmi_{X_1, X_2; Y})](\hat{x}_1, \xh_2, y)
  &=   \sum_{(x_1, x_2, y) \in \cX_1 \times \cX_2 \times \cY} [\pib(\lpmi_{X_1, X_2; Y})](x_1, x_2, y) \cdot \delta_{x_1\xh_1}
 \delta_{x_2\xh_2}\\
  &=\ip{\pib(\lpmi_{X_1, X_2; Y})}{f} = 0.
\end{align*}
This implies that $\pbi_{X_1, X_2} = P_{X_1, X_2}$. Similarly, we have $\pint_{X_1, X_2} = P_{X_1, X_2}$.

Finally, note that since
\begin{align*}
  P_{X_1, X_2, Y}(x_1, x_2, y) = P_{X_1, X_2}(x_1, x_2)P_Y(y)
\left[1 + [\pib(\lpmi_{X_1, X_2; Y})](x_1, x_2, y) + 
[\pii(\lpmi_{X_1, X_2; Y})](x_1, x_2, y)\right],
\end{align*}
we have
\begin{align*}
  P_{X_1, Y}(x_1, y) - \pbi_{X_1, Y}(x_1, y) 
  &= 
  \pint_{X_1, Y}(x_1, y) - P_{X_1}(x_1)P_{Y}(y)\\
  &= \sum_{x_2' \in \cX_2} P_{X_1, X_2}(x_1, x_2') P_Y(y) [\pii(\lpmi_{X_1, X_2; Y})](x_1, x_2', y)\\
  &=0.
\end{align*}
To obtain the last equality, note that for any $\xh_1, \yh \in \cX_1 \times \cY$, let us define $\gamma \in \spc{\cX_1 \times \cY}$ as $\gamma(x_1, y) = \delta_{x_1\xh_1}\delta_{y\yh}$. Then, from $\pii(\lpmi_{X_1, X_2; Y}) \perp \spc{\cX_1 \times \cY}$, we have
\begin{align}
  0 =\ip{\pii(\lpmi_{X_1, X_2; Y})}{\gamma} 
  = \sum_{x_2' \in \cX_2} P_{X_1, X_2}(\xh_1, x_2')P_Y(\yh)
  [\pii(\lpmi_{X_1, X_2; Y})](\xh_1, x_2', \yh).
  \label{eq:pii:zero}
\end{align}
Similarly, we can show that 
\begin{align*}
    P_{X_2, Y}(x_2, y) - \pbi_{X_2, Y}(x_2, y) &= 
  \pint_{X_2, Y}(x_2, y) - P_{X_2}(x_2)P_{Y}(y) = 0.
\end{align*}
\hfill\qed

\subsection{Proof of \propref{prop:max-ent}}
\label{app:prop:max-ent}

Note that since
    \begin{align*}
      H(Q_{X_1, X_2, Y})
      &= H(P_{X_1, X_2}) + H(P_{Y}) - I_Q(X_1, X_2; Y),
    \end{align*}
    we have
    \begin{align}
      \pment_{X_1, X_2, Y} = \argmin_{Q_{X_1, X_2, Y} \in \cqb} I_Q(X_1, X_2; Y),
      \label{eq:max:ent:min}
    \end{align}
    where $I_Q(X_1, X_2; Y)$ denotes the mutual information between $(X_1, X_2)$ and $Y$ with respect to $Q_{X_1, X_2, Y}$.

      Specifically, when we take
$P_{X_1, X_2, Y}$ as the $Q_{X_1, X_2, Y}$, we have
      $\ds I_P(X_1, X_2; Y) = \frac12 \cdot \norm{\lpmi_{X_1, X_2; Y}}^2 + o(\eps^2).$
    Therefore, to solve \eqref{eq:max:ent:min}, it suffices to consider $Q_{X_1, X_2, Y}$ with $I_Q(X_1, X_2; Y) < I_P(X_1, X_2; Y) = O(\eps^2)$. Since $Q_{X_1, X_2}Q_{Y} = P_{X_1, X_2}P_Y \in \relint(\cP^{\cX_1 \times \cX_2 \times \cY})$, we have $\norm{\llrt{Q_{X_1, X_2, Y}}}^2 = O(\eps^2)$ [cf. \cite[Eq. (338)]{Igal2016ieq}].  %

Therefore,
    \begin{align*}
      I_Q(X_1, X_2; Y)
      &=\frac12\cdot\norm{\llrt{Q_{X_1, X_2, Y}}}^2 + o(\eps^2)\\
      &=\frac12\cdot\norm{\pib(\llrt{Q_{X_1, X_2, Y}})}^2
        + \frac12\cdot\norm{\pii(\llrt{Q_{X_1, X_2, Y}})}^2 + o(\eps^2)\\
      &=\frac12\cdot\norm{\pib(\lpmi_{X_1, X_2; Y})}^2
        + \frac12\cdot\norm{\pii(\llrt{Q_{X_1, X_2, Y}})}^2 + o(\eps^2)\\
      &\geq \frac12\cdot\norm{\pib(\lpmi_{X_1, X_2; Y})}^2 + o(\eps^2)
    \end{align*}
    where the inequality holds with equality when $\bnorm{\pii(\llrt{Q_{X_1, X_2, Y}})} = o(\eps)$. Hence, we obtain
     $ \lpmi^{(\ent)}_{X_1, X_2; Y} = \llrt{\pment_{X_1, X_2, Y}} =\pib(\lpmi_{X_1, X_2; Y}) + o(\eps)$, 
    which gives \eqref{eq:max:ent:bi}.\hfill\qed

\subsection{Proof of \propref{prop:opt:tran}}
\label{app:prop:opt:tran}
For all $\phi \defeq \phi^{(1)} + \phi^{(2)} \in \spc{\cX_1}+\spc{\cX_2}$ and $\psi \in \spc{\cY}$ with $\norm{\psi} = 1$, we have $\Ed{P_{X_1, X_2}P_Y}{\psi(Y)\phi(X_1, X_2)} = 0
$. Hence,
\begin{align*}
  \E{\psi(Y)\phi(X_1, X_2)}  &=\Ed{P_{X_1, X_2}P_Y}{(1 + \lpmi_{X_1, X_2; Y}(X_1, X_2, Y)) \cdot \psi(Y)\phi(X_1, X_2)}\\
  &=  \ip{\lpmi_{X_1, X_2; Y}}{\phi \otimes \psi}\\
  &= \ip{\pib(\lpmi_{X_1, X_2; Y})}{\phi \otimes \psi} + \ip{\pii(\lpmi_{X_1, X_2; Y})}{\phi \otimes \psi}\\
  &= \ip{\pib(\lpmi_{X_1, X_2; Y})}{\phi \otimes \psi},
\end{align*}
where the first equality follows from the definition of $\lpmi_{X_1, X_2; Y}$, and where
 the last equality follow from the fact that $\pii(\lpmi_{X_1, X_2; Y}) \perp (\spc{\cX_1 \times \cY} + \spc{\cX_2 \times \cY}) \ni \phi \otimes \psi$.

Therefore, we can rewrite the objective \eqref{eq:opt:tran} as
\begin{align*}
  \E{\left(\psi(Y) - \phi^{(1)}(X_1) - \phi^{(2)}(X_2)\right)^2}  
  &=\E{\left(\psi(Y) - \phi(X_1, X_2)\right)^2}  \\
  &= \norm{\psi}^2 + \norm{\phi}^2 - 2 \cdot \E{\psi(Y)\phi(X_1, X_2)}\\
  &= 1 + \norm{\phi \otimes \psi}^2 - 2\ip{\pib(\lpmi_{X_1, X_2; Y})}{\phi \otimes \psi}\\
  &= 1 + \norm{\pib(\lpmi_{X_1, X_2; Y}) - \phi \otimes \psi}^2 - \norm{\pib(\lpmi_{X_1, X_2; Y})}^2.
\end{align*}
Finally, the proof is completed by noting that $\norm{\pib(\lpmi_{X_1, X_2; Y})}^2 = \sum_{i = 1}^{\Kb} \sigmab_i^2$, and we have
\begin{align*}
\norm{\pib(\lpmi_{X_1, X_2; Y}) - \phi \otimes \psi}^2 
\geq \norm{{r_1(\pib(\lpmi_{X_1, X_2; Y}))}}^2
= \sum_{i = 2}^{\Kb} \sigmab_i^2,
\end{align*}
where the inequality becomes equality if and only if $\phi \otimes \psi = \md(\pib(\lpmi_{X_1, X_2; Y})) = \sigmab_1 (\fb_1^* \otimes \gb_1^*)$. 

\hfill\qed

\subsection{Proof of \thmref{thm:hsbih}}
\label{app:thm:hsbih}
  We first introduce a useful lemma.  
  \begin{lemma}
    \label{lem:mm}
  For all $k \geq 1, f \in \spcn{\cX}{k}, g \in \spcn{\cY}{k}$, let $\hb \defeq \proj{f}{\spc{\cX_1} + \spc{\cX_2}}$, $h = f - \hb$.  Then, we have
  \begin{align}
     \Hsm(f, g) 
    &=\frac12 \cdot
     \left[ L(R_{X_1, X_2, Y}) -\eta_0
       \cdot  \bnorm{\pii(\llrt{\Ph^{(0)}_{X_1, X_2, Y}}) - h \otimes g}^2 - \Lb( \hb \otimes g)\right],
  \end{align}
   where we have defined %
   \begin{align}
     \Lb(\gamma) \defeq
     \eta_0 \cdot \bnorm{\pib(\llrt{\Ph^{(0)}_{X_1, X_2, Y}}) - \pib(\gamma)}^2
     + \eta_1 \cdot \bnorm{\llrt{\Ph^{(1)}_{X_1, Y}} - \pimn{1}(\gamma)}^2
     + \eta_2 \cdot \bnorm{\llrt{\Ph^{(2)}_{X_2, Y}} - \pimn{2}(\gamma)}^2.
     \label{eq:L:b}
   \end{align}

 \end{lemma}
 \begin{proof}%
   By definition, we have
  \begin{align}
    \Hs(f, g; \Ph^{(0)}_{X_1, X_2, Y})
    &= \frac12 \left(\bnorm{\llrt{\Ph^{(0)}_{X_1, X_2, Y}}}^2
    - \bnorm{\llrt{\Ph^{(0)}_{X_1, X_2, Y}} - f \otimes g}^2
    \right)\notag\\
    &= \frac12 \left(\bnorm{\llrt{\Ph^{(0)}_{X_1, X_2, Y}}}^2
    - \bnorm{(\pii(\llrt{\Ph^{(0)}_{X_1, X_2, Y}}) - h \otimes g) + (\pib(\llrt{\Ph^{(0)}_{X_1, X_2, Y}}) - \hb \otimes g)}^2
    \right)\notag\\
    &=
    \frac12 \left(\bnorm{\llrt{\Ph^{(0)}_{X_1, X_2, Y}}}^2
    - \bnorm{\pii(\llrt{\Ph^{(0)}_{X_1, X_2, Y}}) - h \otimes g}^2
    - \bnorm{\pib(\llrt{\Ph^{(0)}_{X_1, X_2, Y}}) - \hb \otimes g}^2
    \right),
      \label{eq:hs:ph:joint}
  \end{align}
  where to obtain the last equality we have used the orthogonality between
$\left(\pii(\llrt{\Ph^{(0)}_{X_1, X_2, Y}}) - h \otimes g\right) \perp (\spc{\cX_1 \times \cY} + \spc{\cX_2 \times \cY})$ and
$\left(\pib(\llrt{\Ph^{(0)}_{X_1, X_2, Y}}) - \hb \otimes g\right) \in (\spc{\cX_1 \times \cY} + \spc{\cX_2 \times \cY})$.

In addition, for each $i = 1, 2$, from
$\opxn{i}(f) = \opxn{i}(\hb)$ we have
  \begin{align}
    \Hs(\opxn{i}(f), g; \Ph^{(i)}_{X_i, Y}) =
    \Hs(\opxn{i}(\hb), g; \Ph^{(i)}_{X_i, Y}) &=
        \frac12 \left(\bnorm{\llrt{\Ph^{(i)}_{X_i, Y}}}^2
    - \bnorm{\llrt{\Ph^{(i)}_{X_i, Y}} - \opxn{i}(\hb) \otimes g}^2\right),\notag\\
&=      \frac12 \left(\bnorm{\llrt{\Ph^{(i)}_{X_i, Y}}}^2
  - \bnorm{\llrt{\Ph^{(i)}_{X_i, Y}} - \pimn{i}(\hb \otimes g)}^2\right),
  \label{eq:hs:opxn:i}
  \end{align}
where the last equality follows from that $\pimn{i}(\hb \otimes g) = \opxn{i}(\hb) \otimes g$.

From \eqref{eq:hs:ph:joint} and \eqref{eq:hs:opxn:i}, we can rewrite \eqref{eq:hsm:def} as
\begin{align*}
     \Hsm(f, g) 
  &= \eta_0 \cdot \Hs(f, g; \Ph^{(0)}_{X_1, X_2, Y})
    +   \eta_1 \cdot \Hs\left(\opxn{1}(f), g; \Ph^{(1)}_{X_1, Y}\right) +   \eta_2 \cdot  \Hs\left(\opxn{2}(f), g; \Ph^{(2)}_{X_2, Y}\right)\\
  &=     \frac{\eta_0}2 \cdot \left(\bnorm{\llrt{\Ph^{(0)}_{X_1, X_2, Y}}}^2
    - \bnorm{\pii(\llrt{\Ph^{(0)}_{X_1, X_2, Y}}) - h \otimes g}^2
    - \bnorm{\pib(\llrt{\Ph^{(0)}_{X_1, X_2, Y}}) - \hb \otimes g}^2
    \right)\\
  &\qquad+ \sum_{i = 1}^2 \frac{\eta_i}{2}\cdot \left(\bnorm{\llrt{\Ph^{(i)}_{X_i, Y}}}^2
    - \bnorm{\llrt{\Ph^{(i)}_{X_1, Y}} - \pimn{i}(\hb \otimes g)}^2\right)\\
  &= \frac12 \cdot \left[L(R_{X_1, X_2, Y}) - \eta_0 \cdot \bnorm{\pii(\llrt{\Ph^{(0)}_{X_1, X_2, Y}}) - h \otimes g}^2 \right]\\
  &\qquad - \frac12 \left[\eta_0\cdot \bnorm{\pib(\llrt{\Ph^{(0)}_{X_1, X_2, Y}}) - \hb \otimes g}^2
+ \sum_{i = 1}^2\eta_i \cdot \bnorm{\llrt{\Ph^{(i)}_{X_1, Y}} - \pimn{i}(\hb \otimes g)}^2
\right]\\
    &=\frac12 \cdot
     \left[ L(R_{X_1, X_2, Y}) -\eta_0
       \cdot  \bnorm{\pii(\llrt{\Ph^{(0)}_{X_1, X_2, Y}}) - h \otimes g}^2 - \Lb( \hb \otimes g)\right],
\end{align*}
where the last equality follows from the fact that 
$\hb \otimes g = \pib(\hb \otimes g)$.

 \end{proof}

  Let use define $\hb \defeq \proj{f}{\spc{\cX_1} + \spc{\cX_2}}$ and $h \defeq f - \hb$.  Then, from \lemref{lem:mm}, 
we have
   \begin{align}
     & \Hsm(\fb, \gb) 
     = \frac12\cdot\left[
 L(R_{X_1, X_2, Y}) -\eta_0
       \cdot  \bnorm{\pii(\llrt{\Ph^{(0)}_{X_1, X_2, Y}})}^2 - \Lb( \fb \otimes \gb)
       \right],
       \label{eq:hsm:fb:gb}
   \end{align}
   and,  similarly, 
   \begin{align}
     \Hsm\left(
     \begin{bmatrix}
       \fb\,\\
       f\, 
     \end{bmatrix},
     \begin{bmatrix}
       \gb\\
       g
     \end{bmatrix}
     \right)
     =
     \frac12 \cdot
     \left[
     L(R_{X_1, X_2, Y}) -\eta_0
       \cdot  \bnorm{\pii(\llrt{\Ph^{(0)}_{X_1, X_2, Y}}) - h \otimes g}^2 - \Lb(\fb \otimes \gb + \hb \otimes g)
     \right].
       \label{eq:hsm:f:g}
   \end{align}

   Therefore, \eqref{eq:hsm:fb:gb} is minimized if and only if 
   \begin{align}
     \fb \otimes \gb = \pib\Bigl(\llrt{\Pest_{X_1, X_2, Y}}\Bigr),
     \label{eq:hsm:fb:gb:opt}
   \end{align}
and \eqref{eq:hsm:f:g} is minimized if and only if 
   \begin{align}
     \fb \otimes \gb + \hb \otimes g = \pib\Bigl(\llrt{\Pest_{X_1, X_2, Y}}\Bigr), \qquad 
     h \otimes g = \mdnl{k}\left(\pii\Bigl(\llrt{\Ph^{(0)}_{X_1, X_2, Y}}\Bigr)\right).
     \label{eq:hsm:f:g:opt}
   \end{align}

   Hence, the common solution of \eqref{eq:hsm:fb:gb:opt} 
 and  \eqref{eq:hsm:f:g:opt} is
   \begin{align}
     \fb \otimes \gb = \pib\Bigl(\llrt{\Pest_{X_1, X_2, Y}}\Bigr), \qquad \hb \otimes g = 0, \qquad
     h \otimes g = \mdnl{k}\left(\pii\Bigl(\llrt{\Ph^{(0)}_{X_1, X_2, Y}}\Bigr)\right),
     \label{eq:hsm:opt}
   \end{align}
   which is equivalent to \eqref{eq:hsbih:opt}. Finally, from \propref{prop:joint:opt}, this is also the solution that maximizes the nested H-score \eqref{eq:hsbi:miss}.\hfill\qed

\subsection{Proof of \thmref{thm:pest:pml}}
\label{app:thm:pest:pml}
  We first introduce  two useful facts.

  \begin{fact}[{\citealt[Theorem 11.1.2]{cover2006elements}}]
    \label{fact:type}
    Given $m$ samples $Z_1, \dots, Z_m$ i.i.d. generated from $Q_Z \in \cP^\cZ$, th probability of observing $\{Z_i\}_{i = 1}^m= \{z_i\}_{i = 1}^m$, denoted by $\prob{\{z_i\}_{i = 1}^m; Q_{Z}}$, satisfies 
$\ds      -\log \prob{\{z_i\}_{i = 1}^m; Q_{Z}} = m \left[H(\Ph_Z)  + \kld{\Ph_Z}{Q_Z}\right]$,
    where $\Ph_Z$ is the empirical distribution of $\{z_i\}_{i = 1}^m$.
\end{fact}

    \begin{fact}[{\citealt[Lemma 4.5]{HuangMWZ2024}}]
      \label{fact:local:z}
      Given a reference distribution $R \in \relint(\cP^\cZ)$. Then, for $P, Q \in \cP^\cZ$ with $\norm{\llrrt{P}{R}} = O(\eps)$, $\norm{\llrrt{Q}{R}} = O(\eps)$, we have
      \begin{align*}
        D(P\|Q) = \frac12 \cdot \norm{\llrrt{P}{R} - \llrrt{Q}{R}}^2 + o(\eps^2).
      \end{align*}
    \end{fact}

Note that since the data from three different datasets are generated independently, we have 
  \begin{align*}
    \prob{\Ds_0, \Ds_1, \Ds_2; Q_{X_1, X_2, Y}}
    &=  \prob{\Ds_0; Q_{X_1, X_2, Y}} \cdot
      \prob{\Ds_1; Q_{X_1, X_2, Y}}
      \cdot       \prob{\Ds_2; Q_{X_1, X_2, Y}}\\
    &=  \prob{\Ds_0; Q_{X_1, X_2, Y}} \cdot
      \prob{\Ds_1; Q_{X_1, Y}}
      \cdot  \prob{\Ds_2; Q_{X_2, Y}}.
  \end{align*}
  Therefore, from \factref{fact:type},
  \begin{align*}
    &- \log \prob{\Ds_0, \Ds_1, \Ds_2; Q_{X_1, X_2, Y}}\\
    &\qquad= -\log \prob{\Ds_0; Q_{X_1, X_2, Y}} - \log      \prob{\Ds_1; Q_{X_1, Y}} - \log \prob{\Ds_2; Q_{X_2, Y}}\\
    &\qquad=n_0\cdot \bkld{\Ph^{(0)}_{X_1, X_2, Y}}{Q_{X_1, X_2, Y}} + n_1 \cdot \bkld{\Ph_{X_1, Y}^{(1)}}{Q_{X_1, Y}} + n_2 \cdot \bkld{\Ph_{X_2, Y}^{(2)}}{Q_{X_2, Y}}\\
    &\qquad\qquad + n_0\cdot H(\Ph^{(0)}_{X_1, X_2, Y}) + n_1 \cdot H(\Ph_{X_1, Y}^{(1)}) + n_2 \cdot H(\Ph_{X_2, Y}^{(2)})\\
    &\qquad= n \cdot \lml(Q_{X_1, X_2, Y}) + n_0\cdot H(\Ph^{(0)}_{X_1, X_2, Y}) + n_1 \cdot H(\Ph_{X_1, Y}^{(1)}) + n_2 \cdot H(\Ph_{X_2, Y}^{(2)})
  \end{align*}
  where the second equality follows from \factref{fact:type}, and where we have defined
  \begin{align*}
    \lml(Q_{X_1, X_2, Y}) \defeq \eta_0\cdot \bkld{\Ph^{(0)}_{X_1, X_2, Y}}{Q_{X_1, X_2, Y}} + \eta_1 \cdot \bkld{\Ph_{X_1, Y}^{(1)}}{Q_{X_1, Y}} + \eta_2 \cdot \bkld{\Ph_{X_2, Y}^{(2)}}{Q_{X_2, Y}}.
  \end{align*}

  Hence, we can rewrite the maximum likelihood solution $\Pml_{X_1, X_2, Y}$ as
  \begin{align*}
    \Pml_{X_1, X_2, Y}
    &= \argmax_{Q_{X_1, X_2, Y}}~  \prob{\Ds_0, \Ds_1, \Ds_2; Q_{X_1, X_2, Y}} = \argmin_{Q_{X_1, X_2, Y}} \lml(Q_{X_1, X_2, Y}).
  \end{align*}

From $L(R_{X_1, X_2, Y}) = O(\eps^2)$, we obtain $\bnorm{\llrt{\Ph^{(0)}_{X_1, X_2, Y}}} = O(\eps), \bnorm{\llrt{\Ph^{(1)}_{X_1, Y}}} = O(\eps),  \bnorm{\llrt{\Ph^{(2)}_{X_2, Y}}} = O(\eps)$. By definition, we have $L({\Pest_{X_1, X_2, Y}}) < L(R_{X_1, X_2, Y}) = O(\eps^2)$, which implies that $\bnorm{\llrt{\Pest_{X_1, X_2, Y}}} = O(\eps)$.

We first consider $Q_{X_1, X_2, Y}$ with
\begin{align}
 \bnorm{\llrt{Q_{X_1, X_2, Y}} - \llrt{\Pest_{X_1, X_2, Y}}} \leq \eps.\label{eq:q:ball}  
\end{align}
Then, we have
\begin{align}
 \bnorm{  \llrt{Q_{X_1, X_2, Y}}} \leq \bnorm{\llrt{Q_{X_1, X_2, Y}} -  \llrt{\Pest_{X_1, X_2, Y}}} + \bnorm{\llrt{\Pest_{X_1, X_2, Y}}} = O(\eps),
\end{align}
and it follows from \factref{fact:local:z} that %
\begin{align}
  \lml(Q_{X_1, X_2, Y}) &= \eta_0\cdot \bkld{\Ph^{(0)}_{X_1, X_2, Y}}{Q_{X_1, X_2, Y}} + \eta_1 \cdot \bkld{\Ph_{X_1, Y}^{(1)}}{Q_{X_1, Y}} + \eta_2 \cdot \bkld{\Ph_{X_2, Y}^{(2)}}{Q_{X_2, Y}}\notag\\
  &= \frac12 \Bigl(
    \eta_0 \cdot \bnorm{\llrt{\Ph^{(0)}_{X_1, X_2, Y}} - \llrt{Q_{X_1, X_2, Y}}}^2  +   \eta_1 \cdot \bnorm{ \llrt{\Ph^{(1)}_{X_1, Y}} - \llrt{Q_{X_1, Y}}}^2\notag\\
  &\qquad\quad+   \eta_2 \cdot \bnorm{\llrt{\Ph^{(2)}_{X_2, Y}} - \llrt{Q_{X_2, Y}}}^2
    \Bigr) + o(\eps^2)\notag\\
                        &=  \frac12\cdot L({Q_{X_1, X_2, Y}}) + o(\eps^2).
   \label{eq:lml:local:q}
\end{align}
Therefore, for $Q_{X_1, X_2, Y}$ that satisfies \eqref{eq:q:ball}, the minimum value of $\lml(Q_{X_1, X_2, Y})$ is achieved by $Q_{X_1, X_2, Y} = \Pest_{X_1, X_2, Y} + o(\eps)$.

Now we consider the case of $Q_{X_1, X_2, Y}$ with 
 $\bnorm{\llrt{Q_{X_1, X_2, Y}} - \llrt{\Pest_{X_1, X_2, Y}}} > \eps$. Let $\eps' = \eps/\bnorm{\llrt{Q_{X_1, X_2, Y}} - \llrt{\Pest_{X_1, X_2, Y}}} < 1$ and define  
 \begin{align}
   \Qb_{X_1, X_2, Y} \defeq \eps'\cdot Q_{X_1, X_2, Y} + (1-\eps') \cdot \Pest_{X_1, X_2, Y}.
   \label{eq:qb:def}
 \end{align}
Then, we have   $\llrt{\Qb_{X_1, X_2, Y}} = \eps'\cdot \llrt{Q_{X_1, X_2, Y}} + (1-\eps') \cdot \llrt{\Pest_{X_1, X_2, Y}}$, which implies
\begin{align}
  \bnorm{\llrt{\Qb_{X_1, X_2, Y}} -  \llrt{\Pest_{X_1, X_2, Y}}} = \eps'\cdot \bnorm{\llrt{Q_{X_1, X_2, Y}}  -  \llrt{\Pest_{X_1, X_2, Y}}} = \eps.
  \label{eq:dist:qb:pest}
\end{align}
   As a result, we can apply the same analysis on $\Qb_{X_1, X_2, Y}$ and obtain [cf. \eqref{eq:lml:local:q}]
\begin{align}
  \lml(\Qb_{X_1, X_2, Y}) =  \frac12\cdot L({\Qb_{X_1, X_2, Y}}) + o(\eps^2).
\end{align}

Hence, we obtain from \eqref{eq:dist:qb:pest} that
\begin{align}
  \lml(\Qb_{X_1, X_2, Y}) >   \lml(\Pest_{X_1, X_2, Y}) = 
  \frac12\cdot L({\Pest_{X_1, X_2, Y}}) + o(\eps^2)
  \label{eq:ieq:qb:pest}
\end{align}
for $\eps$ sufficiently small.

In addition, since $\lml$ is convex, it follows from Jensen's inequality that (cf. \eqref{eq:qb:def})
\begin{align}
  \lml(\Qb_{X_1, X_2, Y}) \leq \eps' \cdot \lml(Q_{X_1, X_2, Y}) + (1-\eps') \cdot \lml (\Pest_{X_1, X_2, Y}).
  \label{eq:ieq:cvx}
\end{align}
As a result, from \eqref{eq:ieq:qb:pest} and \eqref{eq:ieq:cvx} we have  $\lml(Q_{X_1, X_2, Y}) > \lml (\Pest_{X_1, X_2, Y})$.

Combining both cases of $Q_{X_1, X_2, Y}$, we obtain \eqref{eq:pml:pest}. \hfill\qed

\subsection{Proof of \propref{prop:seq}}
\label{app:prop:seq}
It suffices to establish that there exist $\alpha \colon \cU \to \mathbb{R}$, $\beta \colon \cV \to \mathbb{R}$, such that
\begin{subequations}  
  \label{eq:lpmi:xu:yv}
  \begin{align}
    \lpmi_{\vec{X}; U}(\vec{x}, u) &= \alpha(u) \cdot 
                                     \left[ \tanh\left(2 w  \cdot \varphi(\vec{x}) + b_U\right) - \tanh(b_U)\right],\label{eq:lpmi:xu}\\
    \lpmi_{\vec{Y}; V}(\vec{y}, v) &= \beta(v) \cdot \left[\tanh\left(2 w  \cdot \varphi(\vec{y}) + b_V\right) - \tanh(b_V)\right].
  \end{align}
\end{subequations}

To see this, note that from the Markov relation $\vec{X}-U-V-\vec{Y}$, we have
\begin{align*}
  P_{\vec{X}, \vec{Y}}(\vec{x}, \vec{y})
  &= \sum_{u \in \cU, v \in \cV} P_{\vec{X}|U}(\vec{x}|u) \cdot
  P_{\vec{Y}|V}(\vec{y}|v) \cdot P_{U, V}(u, v)\\
  &= P_{\vec{X}}(\vec{x}) P_{\vec{Y}}(\vec{y}) \sum_{u \in \cU, v \in \cV}  P_{U, V}(u, v) \cdot ( 1 + \lpmi_{\vec{X}; U}(\vec{x}, u)) \cdot( 1 + \lpmi_{\vec{Y}; V}(\vec{y}, v))\\
  &= P_{\vec{X}}(\vec{x}) P_{\vec{Y}}(\vec{y}) \left(1 + \sum_{u \in \cU, v \in \cV}  P_{U, V}(u, v) \cdot \lpmi_{\vec{X}; U}(\vec{x}, u) \cdot  \lpmi_{\vec{Y}; V}(\vec{y}, v)\right),
\end{align*}
where to obtain the last equality we have used the fact that
\begin{align*}
  \sum_{u \in \cU} P_U(u) \cdot \lpmi_{\vec{X}; U}(\vec{x}, u)
  = \sum_{v \in \cV} P_V(v) \cdot \lpmi_{\vec{Y}; V}(\vec{y}, v) = 0.  
\end{align*}

Therefore, we obtain
\begin{align*}
  \lpmi_{\vec{X}; \vec{Y}}(\vec{x}, \vec{y})
  &= \frac{P_{\vec{X}, \vec{Y}}(\vec{x}, \vec{y})}{P_{\vec{X}}(\vec{x})P_{\vec{Y}}(\vec{y})} - 1\\
  &=\sum_{u \in \cU, v \in \cV}  P_{U, V}(u, v) \cdot \lpmi_{\vec{X}; U}(\vec{x}, u) \cdot  \lpmi_{\vec{Y}; V}(\vec{y}, v)\\
  &= \E{\alpha(U)\beta(V)} \cdot \left[ \tanh\left(2 w  \cdot \varphi(\vec{x}) + b_U\right) - \tanh(b_U)\right]
    \cdot \left[\tanh\left(2 w  \cdot \varphi(\vec{y}) + b_V\right) - \tanh(b_V)\right],
\end{align*}
which gives \eqref{eq:seq:fg}.

It remains only to establish \eqref{eq:lpmi:xu:yv}. For symmetry, we consider only \eqref{eq:lpmi:xu}. To begin, 
by definition, we have
\begin{align*}
  P_{\vec{X}|U = u}(\vec{x}) &=  \frac12\cdot \prod_{i = 1}^{l - 1} \left[q_u^{(1 - \delta_{x_i x_{i + 1}})} (1 - q_u)^{\delta_{x_i x_{i + 1}}}\right],
\end{align*}
which implies
\begin{align*}
  \log P_{\vec{X}|U = u}(\vec{x}) = \log \frac12 + \log q_u\cdot \sum_{i = 1}^{l - 1} (1 - \delta_{x_i x_{i + 1}}) + \log (1 - q_u) \cdot \sum_{i = 1}^{l - 1}\delta_{x_i x_{i + 1}}.
\end{align*}

Therefore, we obtain
\begin{align*}
 \frac12 \log\frac{P_{\vec{X}|U = 1}(\vec{x})}{P_{\vec{X}|U = 0}(\vec{x})}
  &= \frac12 \left[\log{P_{\vec{X}|U = 1}(\vec{x})} - \log{P_{\vec{X}|U = 0}(\vec{x})}\right]\\
  &= \frac12 \left[\log \frac{q_1}{q_0}\cdot \sum_{i = 1}^{l - 1} (1 - \delta_{x_i x_{i + 1}}) + \log \frac{1 - q_1}{1 - q_0} \cdot \sum_{i = 1}^{l - 1}\delta_{x_i x_{i + 1}}\right]\\
  &=  \log \frac{q_1}{q_0} \cdot \left(\frac{l - 1}{2} -  \sum_{i = 1}^{l - 1}\delta_{x_i x_{i + 1}}\right)\\
  &= 2 w \cdot \varphi(\vec{x}).
\end{align*}

As a consequence, 
\begin{align*}
\frac{P_{\vec{X}|U = 1}(\vec{x}) - P_{\vec{X}|U = 0}(\vec{x})}{P_{\vec{X}}(\vec{x})}
  &=  \frac{P_{\vec{X}|U = 1}(\vec{x}) - P_{\vec{X}|U = 0}(\vec{x})}{P_U(1)\cdot P_{\vec{X}|U = 1}(\vec{x}) + P_U(0)\cdot P_{\vec{X}|U = 0}(\vec{x})}\\
  &= \frac1{P_U(0)} \cdot \frac{\frac{P_{\vec{X}|U = 1}(\vec{x})}{P_{\vec{X}|U = 0}(\vec{x})} - 1}{\frac{P_U(1)}{P_U(0)} \cdot\frac{P_{\vec{X}|U = 1}(\vec{x})}{P_{\vec{X}|U = 0}(\vec{x})} + 1}\\
  &= \frac{1}{2 P_U(0)P_U(1)} \cdot \left[ \frac{\frac{P_U(1)}{P_U(0)} \cdot\frac{P_{\vec{X}|U = 1}(\vec{x})}{P_{\vec{X}|U = 0}(\vec{x})} - 1}{\frac{P_U(1)}{P_U(0)} \cdot\frac{P_{\vec{X}|U = 1}(\vec{x})}{P_{\vec{X}|U = 0}(\vec{x})} + 1} - \frac{\frac{P_U(1)}{P_U(0)} - 1}{\frac{P_U(1)}{P_U(0)} + 1}
\right]\\
  &=  \frac{1}{2 P_U(0)P_U(1)}\cdot\left[
\tanh\left(2 w  \cdot \varphi(\vec{x}) + b_U\right)
-
\tanh(b_U)
    \right], 
\end{align*}
where to obtain the last equality we have used the fact that $  \frac{t - 1}{t + 1} = \tanh\left(\frac12 \log t\right)$.

Hence, with $u' = 1-u$, we have
\begin{align*}
\lpmi_{\vec{X}; U}(\vec{x}, u)
  &=\frac{P_{\vec{X}, U}(\vec{x}, u) - P_{\vec{X}}(\vec{x})P_U(u)}{P_{\vec{X}}(\vec{x}) P_U(u)}\\
  &=  \frac{P_{\vec{X}|U=u}(\vec{x}) - P_{\vec{X}}(\vec{x})}{P_{\vec{X}}(\vec{x})}\\
  &=  \frac{P_{\vec{X}|U=u}(\vec{x}) - P_U(u)P_{\vec{X}|U = u}(\vec{x}) - P_U(u')P_{\vec{X}|U = u'}(\vec{x})}{P_{\vec{X}}(\vec{x}) }\\ %
  &=  P_U(u') \cdot\frac{P_{\vec{X}|U = u}(\vec{x}) - P_{\vec{X}|U = u'}(\vec{x})}{P_{\vec{X}}(\vec{x})}\\%
  &=  P_U(u') \cdot (-1)^{u+1} \cdot\frac{P_{\vec{X}|U = 1}(\vec{x}) - P_{\vec{X}|U = 0}(\vec{x})}{P_{\vec{X}}(\vec{x})}\\
  &= \frac{ (-1)^{u+1}}{2 P_U(u)}\cdot\left[
\tanh\left(2 w  \cdot \varphi(\vec{x}) + b_U\right)
-
\tanh(b_U)
    \right], 
\end{align*}
which gives \eqref{eq:lpmi:xu} as desired, with $\alpha(u) =\frac{ (-1)^{u+1}}{2 P_U(u)}$.\hfill\qed

\section{Implementation Details of Experiments}
\label{app:exp}
We implement our experiments in Python 3 \citep{python3}, where we use 
the \texttt{PyTorch} \citep{paszke2019pytorch} library for neural network training and use the \texttt{Matplotlib} \citep{hunter2007matplotlib} library for plotting. We also make use of  \texttt{NumPy}  \citep{harris2020array} and \texttt{SciPy} \citep{virtanen2020scipy} for the computation. 

In the experiments, we apply Adam \citep{KingmaB14} as the optimizer with the default parameters: a learning rate of $10^{-3}$, $\beta_1 = 0.9, \beta_2 = 0.999$, and $\eps = 10^{-8}$. For each MLP (multilayer perceptron) used in the experiments, we set the activation function to be the softplus function $x \mapsto \log(1 + e^x)$, which are applied to all layers except the output layer.

It is worth mentioning that our choices of network architectures, optimizers and hyperparameters are not optimized with respect to the used data distributions. It is possible to further optimize such choices to improve the performance or convergence. 

\subsection{Learning Maximal Correlation Functions}
\label{app:exp:mc}
We first introduce the implementation details for \secref{sec:exp:mc}, where the goal is to learn maximal correlation functions for different data. The corresponding learning objective is the nested H-score \eqref{eq:H:cnest}, which are maximized during the training.

\subsubsection{Implementation of \secref{sec:exp:mc:discrete}}
\label{app:exp:mc:discrete}
We set $|\cX| = 8, |\cY| = 6$, and feature dimension $k = 3$. To generate the discrete distributions $P_{X, Y}$, we draw $(|\cX| \cdot |\cY|)$ i.i.d. numbers from $\Unif[0, 1]$ and divide each number by their sum. We then use the resulting $|\cX| \times |\cY|$ table as the values for the probability mass function $P_{X, Y}$. To ensure reproducible results, we set the random seed of \texttt{NumPy} to $20\,230\,606$ in the generating process.

Then, we generate $N = 30\,000$ training sample pairs of $(X, Y)$ from $P_{X, Y}$, then apply one-hot encoding such that the inputs are represented as $|\cX|$ and $|\cY|$ dimensional vectors. Then, we use two one-layer linear networks as the feature extractors $f$ and $g$. 

We train the networks with a minibatch size of 128 for 100 epochs. Then, we obtain the estimated $f_i^*$, $g_i^*$, and $\sigma_i$ by applying \eqref{eq:f:g:simga:i} and compare them with corresponding theoretical values, which we compute from the SVD of corresponding CDM matrix [cf. \eqref{eq:bt:def}], with the results shown in \figref{fig:exp:discrete}. Note that since $f_i^* \otimes g_i^* = (-f_i^*) \otimes (-g_i^*)$, both $(f_i^*, g_i^*)$ and $(-f_i^*, -g_i^*)$ are the optimal feature pairs. For the sake of presentation, we applied a sign modification before the plotting.

\subsubsection{Implementation of \secref{sec:exp:mc:cosine}}
\label{app:exp:mc:cosine}

In this experiment, we first generate $N = 50\,000$ samples of $(X, Y) \in \mathbb{R}^2$ for training, to learn $k = 2$ dimensional features $f$ and $g$. We use two MLPs of the same architecture as the feature extractors for $f$ and $g$. Specifically, each MLP is with three layers, where the dimensions for all intermediate features, from input to output, are: $\text{input} = 1$ -- 32 -- 32 -- $2 = \text{output}$. We then train the networks with a minibatch size of 256 for 100 epochs and use the learned features for estimation tasks, as demonstrated in \secref{sec:exp:mc:cosine}.

\subsubsection{Implementation of \secref{sec:exp:mc:seq}}
\label{app:exp:mc:seq}

In this experiment, we set $k = 1$. To extract $f$ and $g$ from input sequences $\vec{X}$ and $\vec{Y}$, we use one-dimensional convolutional neural networks as the feature extractors, which are used in sentence classifications \citep{kim-2014-convolutional, zhang2017sensitivity}. In particular, $f$ and $g$ are of the same architecture, composed of an embedding (linear) layer, a 1 dimensional convolutional layer, an average pooling layer, and a fully connected (linear) layer. 

We use feature extractor $f$ as an example to illustrate the processing of sequential data. First, we represent $\vec{x}$ sequence as a one-hot encoded list, i.e., each  $x_i \in \{(1, 0)^\T, (0, 1)^\T\}$. Then, the embedding layer maps each $x_i$ to a 4-dimensional vector. The one-dimensional convolutional layer then processes the length-$l$ list of embedded 4-dimensional vectors, by 32 convolutional kernels of size 4. We then activate the convolution results by the ReLU function $x \mapsto \max\{x, 0\}$. The output from each convolutional kernel is further averaged by the average pooling layer, leading to a  32 dimensional feature, with each dimension corresponding to a convolutional kernel.
 Finally, we feed the 32 dimensional feature to the fully connected layer and generate $k = 1$ dimensional output.

Then, we train the feature extractors $f$ and $g$ with a minibatch size of 128 for 100 epochs. The learned features are shown in \secref{sec:exp:mc:seq}.

\subsection{Learning With Orthogonality Constraints}
\label{app:exp:fbar}
In this experiment, we use the same dataset generated in \appref{app:exp:mc:cosine}. We set $\kb = k = 1$, i.e., we learn one-dimensional feature $f$ from $X$ orthogonal to given one-dimensional $\fb$. To this end, we use three MLPs of the same architecture as the feature extractors for $\gb, f, g$, with dimensions
$\text{input} = 1$ -- 32 -- 32 -- $1 = \text{output}$. We then train the networks with a minibatch size of 256 for 100 epochs to maximize the nested H-score restricted to $\fb = \phi$ [cf. \eqref{eq:Hsf:def}], for $\phi(x) = x$ and $\phi(x) = x^2$, respectively.

\subsection{Learning With Side Information}
\label{app:exp:side}
We set $|\cX| = 8, |\cS| = |\cY| = 3$, and generate $P_{X, S, Y}$ in a manner similar to \appref{app:exp:mc:discrete}, with the same random seed. 
Then, we generate $N = 50\,000$ training samples of $(X, S, Y)$ triples. In our implementation, we set $\kb = |\cS| - 1 = 2$ and $k = 1$, and set  feature extractors  $\fb \in \spcn{\cX}{2}, \gb \in \spcn{\cS}{2}$, $f \in \spc{\cX}$, $g \in \spc{\cS \times \cY}$ as corresponding one-layer linear networks with one-hot encoded inputs. In particular, we convert each $(s, y)$ to one unique one-hot vector of  in $\mathbb{R}^{|\cS|\cdot|\cY|}$ as the input to the network $g$. Then, we train these feature extractors on the training set with a minibatch size of 256 for 100 epochs, to maximize the nested H-score configured by $\cside$ [cf. \eqref{eq:cside:def}].

For comparison, we train a multihead network shown in \figref{fig:dnn:mh} on the same dataset, with the same minibatch size and epochs. The feature $f$ is again implemented by a one-layer linear network. In particular, we maximize the log-likelihood function \eqref{eq:likelihood:side:def} to learn the corresponding feature and weights. Then, we convert the weights to  $g \in \spc{\cS \times \cY}$ via the correspondence [cf.   \eqref{eq:def:G:b:side}] $g(s, y) = G_s(1, y)$.
The comparison between features learned from two approaches are shown in  \figref{fig:exp:multihead}. For the sake of presentation, we have normalized the features before plotting, such that $f$ and each $g(s, \cdot)$ are zero-mean, and unit variance with respect to $\Unif(\cX)$ and $\Unif(\cY)$, respectively. 

\subsection{Multimodal Learning With Missing Modalities}
\label{app:exp:mm}

\subsubsection{Implementation of \secref{sec:exp:mm:comp}}
\label{app:exp:mm:comp}
We first generate $N = 50\,000$ triples of $(X_1, X_2, Y)$ for training. In implementing the algorithm, we $\kb = k = 1$. To represent $\fb \in \spc{\cX}$, we set each $\fb^{(i)}$, $i \in \{1, 2\}$ as an MLP with 
dimensions $\text{input} = 1$ -- 32 -- 32 -- $1 = \text{output}$. To represent $f$, we use an MLP with dimensions $\text{input} = 2$ -- 32 -- 32 -- $1 = \text{output}$ with the input set to $X_1 \pplus X_2 = (X_1, X_2)^\T$. Since $Y$ is discrete, we use one-layer linear network as $\gb$ and $g$, where the inputs are one-hot encoded $Y$. Therefore, both $\gb$ and $g$ are with one linear layer, taking $|\cY| = 2$ dimensional input and outputting one-dimensional feature.

We then train these feature extractors on the training set with a minibatch size of 256 for 100 epochs, to maximize the nested H-score configured by $\cbi$ [cf. \eqref{eq:cbi:def}].
The learned features are then shown in \figref{fig:exp:mm:cosine}.

For the prediction problem with missing modality, we use two MLPs to represent conditional expectation operators $\phi_1 \defeq \opxn{1}(\fb^{(2)})$ and $\phi_2 \defeq \opxn{2}(\fb^{(1)})$, each with dimensions $\text{input} = 1$ -- 32 -- 32 -- $1 = \text{output}$. These two networks are optimized by minimizing the corresponding mean square error. For example, we train $\phi_1$ by minimizing
$\E{\left(\fb^{(2)}(X_2) - \phi_1(X_1)\right)^2}$ over all parameters in $\phi_1$ network. The training of $\phi_1, \phi_2$ is with  minibatch size of 256, and 100 epochs.

\subsubsection{Implementation of \secref{sec:exp:mm:pair}}
\label{app:exp:mm:pair}

We generate $N = 50\,000$ triples of $(X_1, X_2, Y)$ from \eqref{eq:cos:mm:dist} and \eqref{eq:cos:mm:cond:pair}, and adopt the  decomposition \eqref{eq:dcmp:triple} on each triple. This gives three pairwise datasets with samples of $(X_1, X_2)$, $(X_1, Y)$, $(X_2, Y)$, where each dataset has $N$ sample pairs. Then, we adopt the same setting of networks to represent one-dimensional $\fb^{(1)}, \fb^{(2)}$ and $g$.

We then train these feature extractors on the three datasets for 
 100 epochs with a minibatch size of 256, to maximize  $\Hs(\fb, \gb)$.
Here, we compute $\Hs(\fb, \gb)$ based on the minibatches from the three pairwise datasets according to \eqref{eq:Hs:fb:gb}. 

To solve the prediction problem with missing modality, we use the same network architectures and training settings to learn $\opxn{1}(\fb^{(2)})$ and $\opxn{2}(\fb^{(1)})$, as introduced in \appref{app:exp:mm:comp}.

\bibliography{ref}

\end{document}